%% file: main.tex
\PassOptionsToPackage{table,dvipsnames}{xcolor}
\documentclass{article}
\usepackage{iclr2027_conference,times}

\usepackage{amsmath,amssymb,amsthm}
\usepackage{dsfont}
\usepackage{mathtools}
\usepackage{graphicx}
\usepackage{tikz}
\usetikzlibrary{decorations.pathreplacing,calligraphy,fit,arrows.meta,calc,fadings}
\usepackage{pgfplots}
\pgfplotsset{compat=1.18}
\usepackage{booktabs}
\usepackage{enumitem}
\usepackage{multirow}
\usepackage{tabularx}
\usepackage{float}
\usepackage{algorithm}
\usepackage[noend]{algpseudocode}
\usepackage{caption}
\newcommand{\best}[1]{{\boldmath\textbf{#1}}}

\newcommand{\second}[1]{\underline{#1}}
\newcolumntype{Y}{>{\centering\arraybackslash}X}
\usepackage{xcolor}
\usepackage{xspace}
\usepackage{microtype}
\usepackage{hyperref}
\definecolor{citeblue}{HTML}{1D4ED8}
\definecolor{refred}{HTML}{C62828}
\definecolor{royalblue}{HTML}{4169E1}
\hypersetup{colorlinks=true, linkcolor=MidnightBlue,
            citecolor=ForestGreen, urlcolor=royalblue}
\makeatletter
\renewcommand\hyper@natlinkbreak[2]{#1}
\let\memo@natlinkstart\hyper@natlinkstart\let\memo@natlinkend\hyper@natlinkend
\renewcommand\hyper@natlinkstart[1]{\leavevmode\hbox\bgroup\memo@natlinkstart{#1}}
\renewcommand\hyper@natlinkend{\memo@natlinkend\egroup}
\makeatother
\usepackage{url}
\usepackage[normalem]{ulem}
\input{math_commands.tex}

\newtheorem{theorem}{Theorem}
\newtheorem{proposition}[theorem]{Proposition}

\theoremstyle{remark}

\theoremstyle{plain}

\newcommand{\ours}{\textsc{Memorizon}\xspace}

\newcommand{\spanm}{m}
\newcommand{\band}{\mathcal{B}}
\newcommand{\tgt}{\mathcal{T}}
\newcommand{\rec}{\mathcal{R}}
\newcommand{\anchor}{A}
\newcommand{\SE}{\mathrm{SE}}
\newcommand{\Sco}{\mathcal{S}}
\newcommand{\Ovl}{\mathcal{O}}
\newcommand{\Pdist}{\mathcal{D}}

\definecolor{okgreen}{HTML}{1B7F4B}
\definecolor{amberc}{HTML}{B26B00}
\definecolor{graym}{HTML}{A0A0A0}
\definecolor{hdrgray}{HTML}{555555}
\definecolor{grouprow}{HTML}{F3F4F6}
\definecolor{ourrow}{HTML}{D0F2F3}
\definecolor{ourline}{HTML}{79DBDC}
\definecolor{bandgray}{HTML}{EFEFEF}
\definecolor{cAnF}{HTML}{F6D9B0}\definecolor{cAnD}{HTML}{B9822F}
\definecolor{cBaF}{HTML}{BFE0C8}\definecolor{cBaD}{HTML}{3E8A5B}
\definecolor{cReF}{HTML}{BFD6EE}\definecolor{cReD}{HTML}{2F6DA8}
\definecolor{cTgF}{HTML}{D9CBEC}\definecolor{cTgD}{HTML}{6C4B9E}
\definecolor{cHiF}{HTML}{EDEDED}\definecolor{cHiD}{HTML}{9A9A9A}
\definecolor{cMkD}{HTML}{C0392B}
\definecolor{cK2}{HTML}{E69F00}\definecolor{cK4}{HTML}{56B4E9}
\definecolor{cK6}{HTML}{009E73}\definecolor{cK8}{HTML}{CC79A7}
\definecolor{cK10}{HTML}{0072B2}

\title{Memorizon: Training World Models Beyond Their Context Window}
\author{Tingting Liao\textsuperscript{1}, \ Xuezhi Liang\textsuperscript{1}, \ Hao Li\textsuperscript{2}, \ Guangyi Liu\textsuperscript{1}}

\iclrfinalcopy

\hypersetup{pdftitle={Memorizon: Training World Models Beyond Their Context Window},
            pdfauthor={Tingting Liao, Xuezhi Liang, Hao Li, Guangyi Liu}}

\usepackage{fontawesome5}
\usepackage[most]{tcolorbox}
\definecolor{absbg}{HTML}{FFFDEC}
\makeatletter
\def\@maketitle{\vbox{\hsize\textwidth\centering
  \lhead{Preprint}
  \vskip -18pt
  {\fontsize{13.4}{16}\selectfont\bfseries Memorizon: Training World Models Beyond Their Context Window\par}
  \vskip 0.75em
  {\fontsize{10}{12}\selectfont\bfseries \@author\par}
  \vskip 0.3em
  {\fontsize{8}{10}\selectfont \textsuperscript{1}IFM, MBZUAI \quad \textsuperscript{2}MBZUAI\par}
  \vskip 0.45em
  {\fontsize{8}{10}\selectfont \href{http://tingtingliao.github.io/memorizon}{{\color{ourline}\faGlobe}\ Webpage}\qquad
   \href{https://github.com/TingtingLiao/memorizon}{{\color{black}\faGithub}\ Code}\qquad
   \href{https://huggingface.co/Luffuly/memorizon}{\raisebox{-0.3ex}{\includegraphics[height=1.9ex]{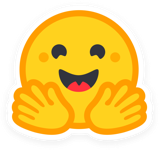}}\ Model}\par}
  \vskip 0.7em
  \hrule height 0.5pt
  \vskip 1.1em}}
\renewenvironment{abstract}{\begin{tcolorbox}[enhanced,colback=absbg,colframe=absbg,boxrule=0pt,arc=2pt,
    left=7pt,right=7pt,top=5pt,bottom=5pt,before skip=0pt,after skip=1.2em]\noindent\textbf{Abstract}\enspace\ignorespaces}
  {\end{tcolorbox}}
\makeatother

\begin{document}

\maketitle

\input{sections/00_abstract}
\input{sections/01_intro}
\input{sections/02_related}
\input{sections/03_method}

\input{sections/04_experiments}

\input{sections/05_discussion}

\bibliography{refs}
\bibliographystyle{iclr2027_conference}

\clearpage
\appendix
\raggedbottom
\input{appendix/00_overview}
\input{appendix/01_method}

\input{appendix/02_data}
\input{appendix/03_experiments}
\input{appendix/04_visual}

\end{document}

%% file: math_commands.tex
\usepackage{amsmath,amsfonts,bm}

\def\eqref#1{equation~\ref{#1}}

\def\1{\bm{1}}

\DeclareMathAlphabet{\mathsfit}{\encodingdefault}{\sfdefault}{m}{sl}
\SetMathAlphabet{\mathsfit}{bold}{\encodingdefault}{\sfdefault}{bx}{n}

%% file: sections/00_abstract.tex
\begin{abstract}
Streaming world models should render a place consistently across repeated visits. Directly supervising such revisits requires training samples that capture both visits, often spanning minutes.
Yet dense attention over the full span incurs quadratic costs, making long-span supervision expensive.
Memorizon breaks this coupling: long spans are needed for supervision, but not for attention
since the two visits can share a forward pass without including every intervening frame.
A training sample covers a span of any length but is scored only on its last $k$ chunks. Instead of tokenizing the history before them, each scored chunk retrieves its own top-$K$ latents by camera co-visibility, and the union of these requests forms a shared bank. The bank is bounded by $kK$, so the sequence stays bounded however long the span; at the shortest span the recipe is exactly conventional training. Adding the bank raises the cost of a step once; beyond that, a longer span costs little, and going from $100$ to $400$\,s adds $12\%$ to the step time. Against a sliding-window baseline, retrieval raises revisit consistency on every split, and a span long enough to reach the first visit of each return adds a further $24\%$ to $30\%$, at some cost in image quality; beyond that span, more length no longer helps. Filling the bank from another episode lowers revisit correlation by $83\%$, so the model uses what it retrieves.
\par\smallskip\noindent\textbf{Project page:} \url{https://tingtingliao.github.io/memorizon}
\end{abstract}

%% file: sections/01_intro.tex
\section{Introduction}
\label{sec:intro}

General video world models should stay consistent over a long rollout: a place the camera leaves and later returns to should look as it did before. The prevailing remedy is architectural --- recurrent states, retrieval banks and learned summaries of history that outlive the attention window \citep{worldmem, contextasmemory, framepack, infiniteworld}.

Whether a model learns to use a memory depends on what its training samples contain, not only on its architecture.
When training is limited to a short video clip, that clip \emph{is} the sequence the transformer attends to, so the span a sample covers equals its context. Attention cost often limits that span to about tens of seconds, while the model is deployed for minutes. If the camera takes longer than a training clip to leave a place and return, no sample contains both visits and no loss term relates them. On our corpus the shortest genuine return spans $13.25$\,s, and stricter definitions push it past $28$\,s (\hyperref[app:windowcount]{App.~\ref*{app:windowcount}}).
Memory-conditioned models
\citep{contextasmemory,worldmem} already relax this by placing frames retrieved from outside the clip into the training sequence. Two questions remain open: how retrieval should be organized when the scored chunks of one sample look at different places, and how long a span must be for what they retrieve to contain the first visit at all.

\input{figures/fig_teaser}

The direct remedy, a longer training window, scales poorly. Attention cost is quadratic and activation memory linear in the context, so a $30$\,s window needs roughly three times the activation memory of a $10$\,s one, yet only $3.8\%$ of such windows on our corpus contain a genuine return. Sparse attention reduces the cost of processing long sequences \citep{vsa,moc}, while history compression reduces the number of tokens representing the past~\citep{framepack}. Neither changes which pairs of visits a sample can relate. Retrieving once for the whole clip, as in Context-as-Memory~\citep{contextasmemory}, gives every chunk the same frames although each draws a different view; we show this loses most of what retrieval buys. We instead let each scored chunk retrieve its own top-$K$ and share their union as one bank, and treat the span as a sampled variable rather than a fixed clip length, so the sequence stays bounded however far back the first visit lies.

\paragraph{\ours.} Each training sample covers a span of $\spanm$ chunks, drawn per sample, and scores only the last $k$; the unretrieved history before them is never tokenized. Instead, each scored chunk retrieves its own top-$K$ latents by camera co-visibility, ranking candidates by frustum overlap penalized by relative pose,
and their union forms a shared memory bank. Within the sequence, attention is causal across chunks with a sliding window of $r{+}1$ chunks, the first frame serves as an attention sink, and a per-chunk mask lets each chunk read only its own retrieved entries, so every chunk directly attends to a bounded number of positions. A rollout reads the same layout one chunk at a time. \ours is thus neither a long-context method, which cheapens attention over more tokens, nor a method for longer rollouts, which are already routine; it is a recipe for training on long video at bounded cost.

\paragraph{Results.} Retrieval raises revisit consistency on every split, already for a model trained on $10$\,s clips, and training on spans of up to $200$\,s raises it further (\hyperref[tab:headline]{Tables~\ref*{tab:headline}} and \hyperref[tab:headline_seen]{\ref*{tab:headline_seen}}). Widening the window without retrieval is no substitute. Longer spans cost some image quality. \textbf{The bank levels off}: for $K = 6$ it holds $28$ entries at $100$\,s and $31.5$ at $400$\,s, while the candidate pool grows tenfold from $50$ to $400$\,s (\hyperref[fig:teaser]{Figure~\ref*{fig:teaser}}b). Against open world models, \ours and CaR are the only two with a clear memory, with \ours ahead on most memory metrics and on returns met mid-path, and replacing the bank's contents at inference shows that it reads what it retrieves (\hyperref[sec:exp]{Sec.~\ref*{sec:exp}}).

\paragraph{Contributions.} (i) We propose \textsc{Memorizon}, a recipe that samples spans of any length and supplies their history through a shared bank formed from per-chunk retrieval, so the training sequence stays bounded however long the span and reduces exactly to ordinary training at the shortest one (\hyperref[sec:method]{Sec.~\ref*{sec:method}}). (ii) We show that its two parts do different jobs. Per-chunk retrieval under a shared positional index lets a model read frames far older than any it was trained on; a span that reaches the first visit turns returns into supervision and adds the gain on the returns that test memory hardest, those met mid-path, which the first frame cannot serve, and those more than $10$\,s apart, with no further gain once the span covers the first visits (\hyperref[sec:span]{Sec.~\ref*{sec:span}}, \hyperref[app:tenswhy]{App.~\ref*{app:tenswhy}}--\ref{app:gapsplit}). (iii) We show that the bank saturates and the step cost barely grows with the span, that the model reads what it retrieves rather than treating it as filler, and we report the cost in image quality this currently carries (\hyperref[sec:exp]{Sec.~\ref*{sec:exp}}).

%% file: figures/fig_teaser.tex
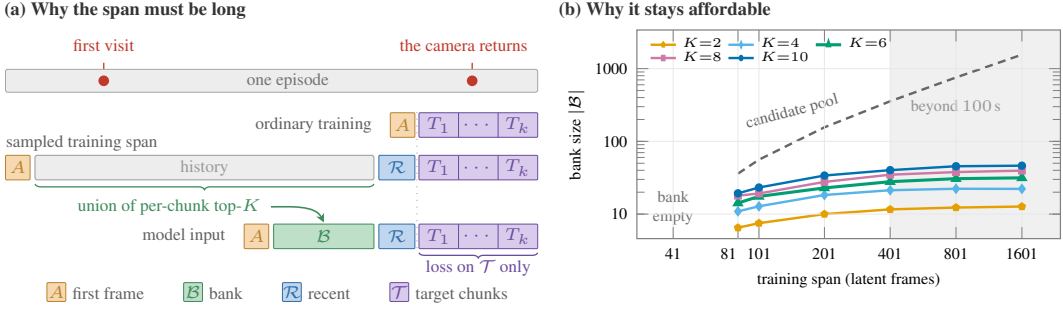
\begin{figure}[t]
\centering
\vspace{-20pt}
\begin{minipage}[t]{0.505\textwidth}
{\scriptsize\bfseries\color{black!80}(a) Why the span must be long}\par\vspace{3pt}
\centering
\resizebox{\linewidth}{!}{%
\begin{tikzpicture}[
  font=\scriptsize,
  tok/.style={draw,line width=0.35pt,rounded corners=0.4pt,
    minimum height=3mm,inner sep=0pt,font=\tiny},
  lbl/.style={font=\tiny,text=black!62},
  mk/.style={circle,fill=cMkD,inner sep=0pt,minimum size=1.3mm},
  first/.style={tok,fill=cAnF,draw=cAnD,text=cAnD,minimum width=3mm},
  recent/.style={tok,fill=cReF,draw=cReD,text=cReD,minimum width=4.4mm},
  target/.style={tok,fill=cTgF,draw=cTgD,text=cTgD,minimum width=14.5mm}
]
\node[lbl,text=cMkD,anchor=south] at (1.30,0.78) {first visit};
\node[lbl,text=cMkD,anchor=south east] at (6.60,0.78) {the camera returns};
\draw[cMkD,line width=0.45pt,dashed] (1.30,0.76) -- (1.30,0.68);
\draw[cMkD,line width=0.45pt,dashed] (5.80,0.76) -- (5.80,0.68);
\draw[fill=black!7,draw=black!35,line width=0.35pt,rounded corners=0.8pt]
  (0.10,0.40) rectangle (6.60,0.68);
\node[lbl,text=black!55] at (3.55,0.54) {one episode};
\node[mk] at (1.30,0.54) {};
\node[mk] at (5.80,0.54) {};

\node[lbl,anchor=east] at (4.70,0.00) {ordinary training};
\node[first] at (4.95,0.00) {$\anchor$};
\node[target] at (5.875,0.00) {};
\draw[cTgD,line width=0.35pt]
  (5.633333,-0.15) -- (5.633333,0.15)
  (6.116667,-0.15) -- (6.116667,0.15);
\node[font=\tiny,text=cTgD,inner sep=0pt] at (5.391667,0.00) {$T_1$};
\node[font=\tiny,text=cTgD,inner sep=0pt] at (5.875,0.00) {$\cdots$};
\node[font=\tiny,text=cTgD,inner sep=0pt] at (6.358333,0.00) {$T_k$};
\draw[black!40,line width=0.35pt,dotted] (5.125,0.16) -- (5.125,-1.54);

\node[lbl,anchor=south west,inner sep=0pt] at (0.10,-0.32)
  {sampled training span};
\node[first] at (0.25,-0.52) {$\anchor$};
\draw[fill=cHiF,draw=cHiD,line width=0.35pt,rounded corners=0.8pt]
  (0.46,-0.67) rectangle (4.60,-0.37);
\node[lbl,text=cHiD] at (2.53,-0.52) {history};
\node[recent] at (4.88,-0.52) {$\rec$};
\node[target] at (5.875,-0.52) {};
\draw[cTgD,line width=0.35pt]
  (5.633333,-0.67) -- (5.633333,-0.37)
  (6.116667,-0.67) -- (6.116667,-0.37);
\node[font=\tiny,text=cTgD,inner sep=0pt] at (5.391667,-0.52) {$T_1$};
\node[font=\tiny,text=cTgD,inner sep=0pt] at (5.875,-0.52) {$\cdots$};
\node[font=\tiny,text=cTgD,inner sep=0pt] at (6.358333,-0.52) {$T_k$};

\draw[decorate,decoration={brace,amplitude=2pt,mirror},line width=0.45pt,cBaD]
  (0.46,-0.72) -- (4.60,-0.72);
\node[lbl,text=cBaD,anchor=north] (sel) at (2.10,-0.78)
  {union of per-chunk top-$K$};
\draw[-{Stealth[length=3pt,width=3pt]},line width=0.55pt,cBaD]
  (sel.east) to[out=0,in=115,looseness=0.9] (3.9874,-1.19);

\node[lbl,anchor=east] at (2.91,-1.36) {model input};
\node[first] at (3.1648,-1.36) {$\anchor$};
\node[tok,fill=cBaF,draw=cBaD,text=cBaD,minimum width=12.252mm]
  at (3.9874,-1.36) {$\band$};
\node[recent] at (4.88,-1.36) {$\rec$};
\node[target] at (5.875,-1.36) {};
\draw[cTgD,line width=0.35pt]
  (5.633333,-1.51) -- (5.633333,-1.21)
  (6.116667,-1.51) -- (6.116667,-1.21);
\node[font=\tiny,text=cTgD,inner sep=0pt] at (5.391667,-1.36) {$T_1$};
\node[font=\tiny,text=cTgD,inner sep=0pt] at (5.875,-1.36) {$\cdots$};
\node[font=\tiny,text=cTgD,inner sep=0pt] at (6.358333,-1.36) {$T_k$};

\draw[decorate,decoration={brace,amplitude=1.8pt,mirror},line width=0.45pt,cTgD]
  (5.15,-1.56) -- (6.60,-1.56);
\node[lbl,text=cTgD,anchor=north,inner sep=1pt]
  at (5.875,-1.61) {loss on $\tgt$ only};

\coordinate (kz) at (0.60,-2.04);
\foreach \f/\d/\m/\t in {
  cAnF/cAnD/\anchor/first frame,
  cBaF/cBaD/\band/bank,
  cReF/cReD/\rec/recent,
  cTgF/cTgD/\tgt/target chunks}{
  \node[draw=\d,fill=\f,text=\d,line width=0.35pt,rounded corners=0.4pt,
    minimum width=2.3mm,minimum height=2.3mm,inner sep=0pt,font=\tiny,
    anchor=west] (ks) at (kz) {$\m$};
  \node[lbl,anchor=west,inner sep=2pt,text height=1.4ex,text depth=0.3ex]
    (kt) at (ks.east) {\t};
  \coordinate (kz) at ([xshift=4mm]kt.east);
}
\end{tikzpicture}%
}
\end{minipage}\hfill
\begin{minipage}[t]{0.475\textwidth}
{\scriptsize\bfseries\color{black!80}(b) Why it stays affordable}\par\vspace{3pt}
\centering
\resizebox{\linewidth}{!}{%
\begin{tikzpicture}[font=\scriptsize]
\begin{axis}[
  width=7.10cm,height=4.30cm,
  xmode=log,ymode=log,xmin=28,xmax=2400,ymin=4.6,ymax=3400,
  xlabel={training span (latent frames)},
  ylabel={bank size $|\band|$},
  xlabel style={font=\tiny,yshift=2pt},
  ylabel style={font=\tiny,yshift=-3pt},
  tick label style={font=\tiny},
  axis line style={black!50,line width=0.45pt},
  grid=major,grid style={black!9,line width=0.3pt},
  xtick={41,101,201,401,801,1601},xticklabels={41,101,201,401,801,1601},
  extra x ticks={81},extra x tick labels={81},
  extra x tick style={tick label style={xshift=-3.5pt}},
  ytick={10,100,1000},yticklabels={10,100,1000},
  log ticks with fixed point,axis on top,
  legend columns=3,
  legend style={font=\tiny,draw=none,fill=white,
    at={(0.02,0.98)},anchor=north west,
    inner sep=2pt,column sep=3pt},
  legend image post style={xscale=0.55,mark size=0.75pt},
  legend style={
    font=\fontsize{4.5}{5.2}\selectfont,
    draw=none,fill=none,
    at={(0.025,0.97)},anchor=north west,
    inner sep=0.5pt,column sep=1pt,row sep=0pt,
    nodes={inner xsep=1pt,inner ysep=0pt}
  },
  legend cell align=left
]
\fill[black!6] (axis cs:401,4.6) rectangle (axis cs:2400,3400);
\node[font=\tiny,text=black!42,anchor=north,align=center]
  at (axis cs:800,500) {beyond $100$\,s};

\addplot[black!60,line width=0.9pt,densely dashed,mark=none,forget plot]
  coordinates {(81,36)(101,56)(201,156)(401,356)(801,756)(1601,1556)};
\node[font=\tiny,text=black!60,anchor=south,rotate=15]
  at (axis cs:150,152) {candidate pool};

\addplot[cK2,line width=0.9pt,mark=pentagon*,mark size=1.1pt,mark options={fill=cK2}]
  coordinates {(81,6.5)(101,7.5)(201,10.0)(401,11.6)(801,12.3)(1601,12.7)};
\addplot[cK4,line width=0.9pt,mark=diamond*,mark size=1.3pt,mark options={fill=cK4}]
  coordinates {(81,10.9)(101,12.8)(201,18.3)(401,21.3)(801,22.3)(1601,22.2)};
\addplot[cK6,line width=1.15pt,mark=triangle*,mark size=1.3pt,mark options={fill=cK6}]
  coordinates {(81,14.2)(101,17.5)(201,22.9)(401,27.9)(801,30.7)(1601,31.5)};
\addplot[cK8,line width=0.9pt,mark=square*,mark size=1pt,mark options={fill=cK8}]
  coordinates {(81,17.9)(101,19.2)(201,27.7)(401,34.8)(801,37.8)(1601,39.6)};
\addplot[cK10,line width=0.9pt,mark=*,mark size=1.1pt,mark options={fill=cK10}]
  coordinates {(81,19.3)(101,23.2)(201,33.9)(401,40.2)(801,45.5)(1601,46.3)};

\node[font=\tiny,text=black!55,anchor=south,align=center] at (axis cs:41,5.4) {bank\\empty};
\legend{$K{=}2$,$K{=}4$,$K{=}6$,$K{=}8$,$K{=}10$}
\end{axis}
\end{tikzpicture}%
}
\end{minipage}

\caption{\textbf{Long-horizon training at bounded cost.} \textbf{(a)} Ordinary training
covers only a short tail of the episode, so the first visit falls outside it; \ours
samples a span that holds both visits, scores only its tail and supplies the history
through a bank $\band$ of retrieved latents. \textbf{(b)} Bank size against training span:
the bank levels off at every $K$ while the candidate pool grows; at $41$ latents there is
no history. Shading marks spans beyond $100$\,s.}
\label{fig:teaser}
\end{figure}

%% file: sections/02_related.tex
\section{Related Work}
\label{sec:related}

\paragraph{Long-Horizon Video World Models.} Recent interactive world models generate for a minute or more yet train on clips of a few seconds, a length set by attention cost, and rely on the model to generalize across the gap \citep{xiang2024pandorageneralworldmodel,genie, genie3, matrixgame2, hyworld15, lingbotworld, lingbotworldv2, yume, dreamxworld, wonder,panteam2026panworldmodelgeneral}. Infinite-World \citep{infiniteworld} observes that memory collapses beyond the temporal window seen in training and answers it with a pose-free hierarchical memory that compresses the history into a compact state, trained on revisit-dense data; we instead remove the need for a return to fit inside the window at all.

\paragraph{Memory in Video World Models.} Memory mechanisms either compress history or select from it. Compression folds history into a compact state \citep{framepack, infiniteworld, yume, relic}, in the limit into a fixed summary of the opening chunk \citep{streamingt2v}. CaR \citep{car} sits between the two: it compresses the history with a lightweight network and retrieves from it implicitly, through attention over viewpoints injected by positional encoding, so what it reads grows with the history it attends to; we instead select retrieved latents explicitly by camera co-visibility and keep the sequence bounded however long the span. Selection keeps a few past frames, retrieved by frustum overlap \citep{worldmem, contextasmemory, worldpack}, by 3D structure such as point maps or image patches lifted to 3D \citep{vmem, memoryforcing, spmem, mosaicmem}, by camera-aware scores or gating \citep{hyworld15, matrixgame3, memorizewhenneeded}, as retrieval-augmented context \citep{vrag}, or through a learned query \citep{memlearner}; training-free variants select inside the KV cache \citep{worldkv, tethercache, closingloop, worldtrace}, and so reach far back only at rollout.
Context-as-Memory \citep{contextasmemory} retrieves clean frames for each predicted segment based on field-of-view overlap, enabling a bidirectional model to generate longer videos. WorldMem \citep{worldmem}, in contrast, trains a causal window with memory frames sampled from anywhere in the same video based on pose proximity. However, it is trained solely on Minecraft, leaving its memory mechanism specialized to a single environment without demonstrating the ability to generalize across environments.

\paragraph{Efficient Attention and Long-Sequence Training.} Attention sinks \citep{streamingllm}, sparse attention \citep{vsa, moc, sparseforcing} and history routing \citep{resamplingforcing} make each token cheaper, and sequence parallelism shards one sequence across devices, as in LWM \citep{lwm}. Both still pay for every token in the sequence, whereas we keep almost none of a long span there, so the sequence does not grow with it. The closer precedent is retrieval-augmented language modelling, which trains on short subsequences while retrieving from a long document \citep{memtransformers, landmark, focusedtransformer}; we retrieve by camera co-visibility instead of learned similarity. Diffusion forcing \citep{diffusionforcing} and self-forcing or distribution-matching distillation \citep{selfforcing, causvid, dmd2}, as used by RELIC \citep{relic}, narrow the gap at rollout, but all leave the span of a sample equal to its length.

%% file: sections/03_method.tex
\section{Method}
\label{sec:method}

\subsection{Overview}
\label{sec:overview}

Let $c$ be the chunk size in latents. A training sample covers the first frame $\anchor$ and $\spanm$ chunks behind it, with $\spanm$ drawn per sample from $[\spanm_{\min}, \spanm_{\max}]$; the span is the only quantity that varies between samples. Two constants partition it: the last $k$ chunks are \emph{scored}, the $r$ chunks before them form the \emph{recent} block, and the remaining $\eta = \spanm - k - r$ are \emph{history}, \[ \underbrace{\anchor}_{\text{first frame}}\;\Big|\; \underbrace{H_1 \dots H_{\eta}}_{\text{history}}\;\Big|\; \underbrace{R_{1} \dots R_{r}}_{\rec,\ r\ \text{chunks}}\;\Big|\; \underbrace{T_{1} \dots T_{k}}_{\tgt,\ k\ \text{chunks}}. \] A span shorter than $k + r$ chunks has fewer recent chunks and no history. Only the history grows with $\spanm$, and it is never tokenized. It is a pool from which the scored chunks retrieve, so the transformer reads
\begin{equation}
[\;\underbrace{\anchor}_{1}\;\mid\;\underbrace{\band}_{|\band|}\;\mid\;
   \underbrace{\rec}_{r\,c}\;\mid\;\underbrace{\tgt}_{k\,c}\;],
\qquad
L \;=\; 1 + |\band| + rc + kc \ \text{slots},
\label{eq:layout}
\end{equation}
where the bank $\band$ holds the retrieved history latents (\hyperref[sec:bank]{Sec.~\ref*{sec:bank}}) and never exceeds $k \cdot K$ entries. \textbf{$L$ is bounded independently of $\spanm$}: the only term that answers to the span at all is $|\band|$, which is bounded by how many chunks ask rather than by how much history exists. The recent block is sized to the attention window of \hyperref[sec:masks]{Sec.~\ref*{sec:masks}}: a window of $r{+}1$ chunks reaches $r$ chunks back from $T_1$, which is exactly $\rec$. We use $r = 1$ throughout.

\paragraph{Ordinary Training as a Special Case.} At $\spanm = k$ the recent block and the history are both empty, the span is $\tgt$ itself, and $L = 1 + kc$ is the standard training sample. With $k = 10$ and $r = 1$, a draw of $\spanm = 10$ therefore carries no recent chunk and $T_1$ opens the window, while $\spanm = 11$ is the shortest draw at which $\rec$ is full and $\spanm = 12$ the shortest with any history at all. Setting $\spanm_{\min} = k$ makes ordinary training the shortest draw of our sampler, so every comparison against standard practice changes a single integer.

\subsection{Attention and Memory Bank}
\label{sec:masks}

\input{figures/fig_overview}

\paragraph{Attention Sink.} Every chunk and latent attend to the first frame $\anchor$, which serves as an attention sink \citep{streamingllm}. The first frame only attends to itself.

\paragraph{Causal Sliding Window.} Attention is bidirectional within a chunk and causal across chunks. Each scored chunk $T_j$ sees itself and the $r$ chunks before it, a window of $r{+}1$ chunks. The window of $T_1$ covers $\rec$, and it then slides right by one chunk for each subsequent $T_j$. The first frame, bank, and recent chunks are pure context and never attend to $\tgt$.

\paragraph{Per-Chunk Retrieval.} Beyond its window, $T_j$ reads its own top-$K$ latents, ranked by \hyperref[eq:score]{Eq.~\ref*{eq:score}} over everything completed before its window: the bank, which carries the history; the recent chunks outside the window; and the scored chunks $T_1 \dots T_{j-r-1}$. Candidates already in the sequence are read in place through the mask, so retrieving them costs no slots: those in $\rec$ are read as the clean context they are, those in $T_1 \dots T_{j-r-1}$ as the noisy latents diffusion forcing has made them. No chunk ever sees a \emph{scored} latent of its own window in clean form --- the clean latents $T_j$ can reach are all pure context, which is never a target of the loss. $T_j$ therefore attends to at most $1 + K + (r{+}1)c$ positions, and to exactly that many once $K$ candidates have completed before its window, independent of the span and of $|\band|$ (\hyperref[fig:overview]{Figure~\ref*{fig:overview}}; \hyperref[app:bounds]{App.~\ref*{app:bounds}}).

\paragraph{Memory Bank.}
\label{sec:bank}
The bank collects what the scored chunks retrieve from history: each $T_j$ contributes its own top-$K$ over the history, and the bank is their union. All $k$ cameras are known when a sample is drawn, so the union is formed before the forward pass and shared by every scored chunk, while the per-chunk mask leaves each chunk with only its own entries. Sharing costs $1.82$ slots per scored latent against $3.75$ for a private copy per chunk (\hyperref[app:layout]{App.~\ref*{app:layout}}), and because the bank draws only from history, no latent the model is scored on ever enters the sequence in clean form. The union also beats curating a memory in advance: it contains each scored chunk's own top-$K$, so a chunk re-selecting inside it recovers exactly the set it would have chosen from the whole history (\hyperref[prop:lossless]{Proposition~\ref*{prop:lossless}}), which no rule that admits entries before knowing who will ask can guarantee.

\paragraph{RoPE Index.} The first frame takes temporal index $0$, every bank entry shares index $1$, and the recent and scored chunks follow consecutively. A shared index reflects that a retrieved set has no order, and it keeps every position in the scored block fixed as $|\band|$ varies across samples and again at inference; methods that retrieve a fixed number of frames can enumerate them instead \citep{contextasmemory}. It also withholds a frame's age: a bank entry carries its camera but not how long ago it was drawn. Positions therefore never depend on the length of the history, so they neither grow without bound nor leave the range seen in training, however long the rollout. What the model learns on entries of one age carries over to entries of any other: a model trained only on $10$\,s clips reads, at inference, bank entries up to $59$\,s old and gains on returns $20$ to $60$\,s apart (App.~\ref{app:tenswhy}). Numbering the bank in temporal order instead is not a clear win (Sec.~\ref{sec:span}).

\subsection{Retrieval Criterion}
\label{sec:criterion}

\input{figures/fig_selector}

A scored chunk should retrieve the frames that saw the place it is about to draw. We estimate this co-visibility from camera poses alone, with two cues. \emph{Frustum overlap} $\Ovl(q,c) \in [0,1]$ is the fraction of points sampled in the query view $q$ that fall inside the frustum of candidate $c$ (\hyperref[eq:overlap]{Eq.~\ref*{eq:overlap}}); it measures how much of what is about to be rendered the candidate already holds, and requires intrinsics and a depth range (\hyperref[app:criterion]{App.~\ref*{app:criterion}} gives the grid and the range we use). \emph{Pose distance} $\Pdist(q,c)$ needs neither, though its translation term, like the depth range, is in the units of the corpus and so depends on the scene scale (\hyperref[app:scale]{App.~\ref*{app:scale}}). Candidates are ranked by
\begin{align}
\Sco(q,c) &\;=\; \Ovl(q,c) \;-\; \lambda\, \Pdist(q,c), \label{eq:score}\\
\Pdist(q,c) &\;=\; \lVert t_q - t_c \rVert \;+\; w\, \angle\!\left(f_q, f_c\right), \label{eq:posedist}
\end{align}
with camera centres $t$ in the units of the corpus, optical axes $f$, the angle $\angle(\cdot,\cdot)$ between them in radians, $\lambda = 0.2$ and $w = 4$, so that turning by $15^\circ$ weighs about as much as moving one unit ($4$\,m). Each chunk keeps its top $K$, ties in $\Sco$ broken in favour of the candidate nearer the query in time, so a candidate further back never displaces an equal-scoring incumbent. The form is not new: WorldMem \citep{worldmem} and WorldPack \citep{worldpack} subtract a time penalty from frustum overlap, and HY-WorldPlay~1.5 \citep{hyworld15} combines overlap with camera distance. We penalize relative pose rather than elapsed time, since a camera can return to a place long after it left.

\paragraph{Ground-Truth Evaluation.} The ground-truth latents of the view being drawn and of every candidate are available at training time, so a rule can be graded directly: by the cosine similarity between the frames it selects and that view, expressed as the fraction of achievable gain between random selection and the best candidates that exist (\hyperref[fig:selector]{Figure~\ref*{fig:selector}}). Three results follow. Overlap alone reaches $47.4\%$, because exact ties are common: any candidate that contains the whole query view scores $\Ovl = 1$. The mixed score reaches $57.2\%$ at $\lambda = 0.2$, level with the best value ($57.3\%$ at $\lambda = 0.1$). Pose distance reaches $55.9\%$ on its own, mostly through its orientation term: ranking by camera centres alone gives $19.8\%$, since these cameras turn far more than they move. Neither cue accounts for occlusion or for motion along the viewing axis (\hyperref[app:selector]{App.~\ref*{app:selector}}).

\subsection{Distillation}
\label{sec:distill}

Training conditions every scored chunk on clean context, while inference conditions it on the model's own output; this exposure bias is what costs image quality (\hyperref[app:selfcond]{App.~\ref*{app:selfcond}}). To reduce it, we distill the causal model into a four-step generator with Self Forcing \citep{selfforcing} and distribution matching distillation \citep{dmd2}, initializing the generator, the critic and the teacher all from the trained model. The generator rolls out the $k$ scored chunks of a sample one after another from its own outputs, as at inference, so the recent chunk and, further into the rollout, the retrieved entries hold generated rather than ground-truth latents. The generated chunks are then placed back into the packed layout of \hyperref[eq:layout]{Eq.~\ref*{eq:layout}}, each noised at its own timestep as in training, and denoised by the teacher, with guidance at scale $4$, and by the critic. The difference between the two estimates, normalized per chunk, is the gradient applied to the generator's output, and the critic is trained on the same rollouts with the flow-matching loss, four critic updates for every generator update. The final denoising step of all $k$ chunks is recomputed with gradient in a single packed forward under a block-diagonal mask. We train for $1{,}000$ steps, $200$ of them generator updates, with learning rates of $5 \times 10^{-7}$ for the generator and $4 \times 10^{-7}$ for the critic.

%% file: figures/fig_overview.tex
\begin{figure}[t]
\centering
\vspace{-20pt}
\colorlet{mskF}{black!5}
\newcommand{\legendrowa}{\tiny\color{black!75}%
  \tikz[baseline=-0.6ex]\node[fill=cAnD!80,draw=none,minimum width=2.3mm,minimum height=2.3mm,inner sep=0pt]{};\,first frame\hspace{1.4em}%
  \tikz[baseline=-0.6ex]\node[fill=cTgD!85,draw=none,minimum width=2.3mm,minimum height=2.3mm,inner sep=0pt]{};\,itself (scored)\hspace{1.4em}%
  \tikz[baseline=-0.6ex]\node[fill=cTgD!40,draw=none,minimum width=2.3mm,minimum height=2.3mm,inner sep=0pt]{};\,predecessor\hspace{1.4em}%
  \tikz[baseline=-0.6ex]\node[fill=cReD!70,draw=none,minimum width=2.3mm,minimum height=2.3mm,inner sep=0pt]{};\,$\rec$, predecessor of $T_1$\hspace{1.4em}%
}
\newcommand{\legendrowb}{\tiny\color{black!75}%
  \tikz[baseline=-0.6ex]\node[fill=cBaD!85,draw=none,minimum width=2.3mm,minimum height=2.3mm,inner sep=0pt]{};\,top-$K$ from the bank (clean)\hspace{1.4em}%
  \tikz[baseline=-0.6ex]\node[fill=cBaF,draw=cBaD!90,line width=0.7pt,minimum width=2.3mm,minimum height=2.3mm,inner sep=0pt]{};\,top-$K$ in the window (noisy)\hspace{1.4em}%
  \tikz[baseline=-0.6ex]\node[fill=cBaF!42,draw=none,minimum width=2.3mm,minimum height=2.3mm,inner sep=0pt]{};\,candidate\hspace{1.4em}%
  \tikz[baseline=-0.6ex]\node[fill=mskF,draw=none,minimum width=2.3mm,minimum height=2.3mm,inner sep=0pt]{};\,masked\hspace{1.4em}%
}
\begin{tikzpicture}[x=1cm,y=1cm,
  rl/.style={font=\tiny,anchor=east,text=cTgD,inner sep=1pt},
  cnt/.style={font=\tiny,text=black!60,inner sep=1pt},
  bh/.style={font=\scriptsize,anchor=south,inner sep=1.5pt,text height=1.6ex,text depth=0.4ex},
  ch/.style={font=\tiny,inner sep=1pt,text=cTgD!90},
  dt/.style={font=\tiny,inner sep=0pt,text=black!45},
]
\path (0,0) (13.700,0);
\draw[cAnD,line width=0.6pt] (0.630,0.505) -- (0.630,0.575) -- (0.882,0.575) -- (0.882,0.505);
\node[bh,text=cAnD] at (0.756,0.595) {$\anchor$};
\draw[cBaD,line width=0.6pt] (0.952,0.505) -- (0.952,0.575) -- (4.745,0.575) -- (4.745,0.505);
\node[bh,text=cBaD] at (2.849,0.595) {bank $\band$};
\draw[cReD,line width=0.6pt] (4.815,0.505) -- (4.815,0.575) -- (5.884,0.575) -- (5.884,0.505);
\node[bh,text=cReD] at (5.350,0.595) {$\rec$};
\draw[cTgD,line width=0.6pt] (5.954,0.505) -- (5.954,0.575) -- (13.070,0.575) -- (13.070,0.505);
\node[bh,text=cTgD] at (9.512,0.595) {scored block $\tgt$};
\node[ch] at (6.489,0.295) {$T_{1}$};
\node[ch] at (7.608,0.295) {$T_{2}$};
\node[ch] at (8.727,0.295) {$T_{3}$};
\node[ch] at (9.847,0.295) {$T_{4}$};
\node[ch] at (10.631,0.295) {$\cdots$};
\node[ch] at (11.416,0.295) {$T_{k-1}$};
\node[ch] at (12.535,0.295) {$T_{k}$};
\node[ch,text=black!55] at (13.390,0.295) {slots};
\node[ch,text=black!55,anchor=east] at (0.600,0.295) {query};
\draw[black!15,line width=0.4pt] (0.620,-2.215) -- (13.080,-2.215);
\node[ch,text=black!60] at (0.756,-2.385) {$0$};
\node[ch,text=black!60] at (2.849,-2.385) {$1$, shared};
\node[ch,text=black!60] at (5.350,-2.385) {$2$--$5$};
\node[ch,text=black!60] at (9.512,-2.385) {$6, 7, \dots$};
\node[ch,text=black!60,anchor=east] at (0.600,-2.385) {RoPE};
\fill[cAnD!80] (0.633,-0.142) rectangle (0.879,0.142);
\fill[cBaF!42] (0.955,-0.142) rectangle (1.202,0.142);
\fill[cBaD!85] (1.228,-0.142) rectangle (1.474,0.142);
\fill[cBaD!85] (1.500,-0.142) rectangle (1.746,0.142);
\fill[cBaD!85] (1.772,-0.142) rectangle (2.019,0.142);
\fill[cBaF!42] (2.045,-0.142) rectangle (2.291,0.142);
\fill[cBaF!42] (2.317,-0.142) rectangle (2.563,0.142);
\fill[cBaF!42] (2.589,-0.142) rectangle (2.836,0.142);
\fill[cBaD!85] (2.862,-0.142) rectangle (3.108,0.142);
\fill[cBaF!42] (3.134,-0.142) rectangle (3.380,0.142);
\fill[cBaF!42] (3.406,-0.142) rectangle (3.653,0.142);
\fill[cBaF!42] (3.679,-0.142) rectangle (3.925,0.142);
\fill[cBaD!85] (3.951,-0.142) rectangle (4.197,0.142);
\fill[cBaF!42] (4.223,-0.142) rectangle (4.470,0.142);
\fill[cBaD!85] (4.496,-0.142) rectangle (4.742,0.142);
\fill[cReD!70] (4.818,-0.142) rectangle (5.064,0.142);
\fill[cReD!70] (5.090,-0.142) rectangle (5.337,0.142);
\fill[cReD!70] (5.363,-0.142) rectangle (5.609,0.142);
\fill[cReD!70] (5.635,-0.142) rectangle (5.881,0.142);
\fill[cTgD!85] (5.957,-0.142) rectangle (6.204,0.142);
\fill[cTgD!85] (6.230,-0.142) rectangle (6.476,0.142);
\fill[cTgD!85] (6.502,-0.142) rectangle (6.748,0.142);
\fill[cTgD!85] (6.774,-0.142) rectangle (7.020,0.142);
\fill[mskF] (7.076,-0.142) rectangle (7.323,0.142);
\fill[mskF] (7.349,-0.142) rectangle (7.595,0.142);
\fill[mskF] (7.621,-0.142) rectangle (7.867,0.142);
\fill[mskF] (7.893,-0.142) rectangle (8.140,0.142);
\fill[mskF] (8.196,-0.142) rectangle (8.442,0.142);
\fill[mskF] (8.468,-0.142) rectangle (8.714,0.142);
\fill[mskF] (8.740,-0.142) rectangle (8.987,0.142);
\fill[mskF] (9.013,-0.142) rectangle (9.259,0.142);
\fill[mskF] (9.315,-0.142) rectangle (9.561,0.142);
\fill[mskF] (9.587,-0.142) rectangle (9.834,0.142);
\fill[mskF] (9.860,-0.142) rectangle (10.106,0.142);
\fill[mskF] (10.132,-0.142) rectangle (10.378,0.142);
\fill[mskF] (10.434,-0.142) rectangle (10.828,0.142);
\node[dt,text=black!45] at (10.631,0.000) {$\cdots$};
\fill[mskF] (10.884,-0.142) rectangle (11.131,0.142);
\fill[mskF] (11.157,-0.142) rectangle (11.403,0.142);
\fill[mskF] (11.429,-0.142) rectangle (11.675,0.142);
\fill[mskF] (11.701,-0.142) rectangle (11.948,0.142);
\fill[mskF] (12.004,-0.142) rectangle (12.250,0.142);
\fill[mskF] (12.276,-0.142) rectangle (12.522,0.142);
\fill[mskF] (12.548,-0.142) rectangle (12.795,0.142);
\fill[mskF] (12.821,-0.142) rectangle (13.067,0.142);
\node[rl] at (0.540,0.000) {$T_{1}$};
\node[cnt] at (13.390,0.000) {$15$};
\fill[cAnD!80] (0.633,-0.452) rectangle (0.879,-0.168);
\fill[cBaF!42] (0.955,-0.452) rectangle (1.202,-0.168);
\fill[cBaD!85] (1.228,-0.452) rectangle (1.474,-0.168);
\fill[cBaF!42] (1.500,-0.452) rectangle (1.746,-0.168);
\fill[cBaF!42] (1.772,-0.452) rectangle (2.019,-0.168);
\fill[cBaF!42] (2.045,-0.452) rectangle (2.291,-0.168);
\fill[cBaF!42] (2.317,-0.452) rectangle (2.563,-0.168);
\fill[cBaF!42] (2.589,-0.452) rectangle (2.836,-0.168);
\fill[cBaF!42] (2.862,-0.452) rectangle (3.108,-0.168);
\fill[cBaD!85] (3.134,-0.452) rectangle (3.380,-0.168);
\fill[cBaD!85] (3.406,-0.452) rectangle (3.653,-0.168);
\fill[cBaD!85] (3.679,-0.452) rectangle (3.925,-0.168);
\fill[cBaD!85] (3.951,-0.452) rectangle (4.197,-0.168);
\fill[cBaF!42] (4.223,-0.452) rectangle (4.470,-0.168);
\fill[cBaF!42] (4.496,-0.452) rectangle (4.742,-0.168);
\fill[cBaF!42] (4.818,-0.452) rectangle (5.064,-0.168);
\fill[cBaF!42] (5.090,-0.452) rectangle (5.337,-0.168);
\draw[fill=cBaF,draw=cBaD!90,line width=0.75pt] (5.102,-0.440) rectangle (5.325,-0.180);
\fill[cBaF!42] (5.363,-0.452) rectangle (5.609,-0.168);
\fill[cBaF!42] (5.635,-0.452) rectangle (5.881,-0.168);
\fill[cTgD!40] (5.957,-0.452) rectangle (6.204,-0.168);
\fill[cTgD!40] (6.230,-0.452) rectangle (6.476,-0.168);
\fill[cTgD!40] (6.502,-0.452) rectangle (6.748,-0.168);
\fill[cTgD!40] (6.774,-0.452) rectangle (7.020,-0.168);
\fill[cTgD!85] (7.076,-0.452) rectangle (7.323,-0.168);
\fill[cTgD!85] (7.349,-0.452) rectangle (7.595,-0.168);
\fill[cTgD!85] (7.621,-0.452) rectangle (7.867,-0.168);
\fill[cTgD!85] (7.893,-0.452) rectangle (8.140,-0.168);
\fill[mskF] (8.196,-0.452) rectangle (8.442,-0.168);
\fill[mskF] (8.468,-0.452) rectangle (8.714,-0.168);
\fill[mskF] (8.740,-0.452) rectangle (8.987,-0.168);
\fill[mskF] (9.013,-0.452) rectangle (9.259,-0.168);
\fill[mskF] (9.315,-0.452) rectangle (9.561,-0.168);
\fill[mskF] (9.587,-0.452) rectangle (9.834,-0.168);
\fill[mskF] (9.860,-0.452) rectangle (10.106,-0.168);
\fill[mskF] (10.132,-0.452) rectangle (10.378,-0.168);
\fill[mskF] (10.434,-0.452) rectangle (10.828,-0.168);
\node[dt,text=black!45] at (10.631,-0.310) {$\cdots$};
\fill[mskF] (10.884,-0.452) rectangle (11.131,-0.168);
\fill[mskF] (11.157,-0.452) rectangle (11.403,-0.168);
\fill[mskF] (11.429,-0.452) rectangle (11.675,-0.168);
\fill[mskF] (11.701,-0.452) rectangle (11.948,-0.168);
\fill[mskF] (12.004,-0.452) rectangle (12.250,-0.168);
\fill[mskF] (12.276,-0.452) rectangle (12.522,-0.168);
\fill[mskF] (12.548,-0.452) rectangle (12.795,-0.168);
\fill[mskF] (12.821,-0.452) rectangle (13.067,-0.168);
\node[rl] at (0.540,-0.310) {$T_{2}$};
\node[cnt] at (13.390,-0.310) {$15$};
\fill[cAnD!80] (0.633,-0.762) rectangle (0.879,-0.478);
\fill[cBaD!85] (0.955,-0.762) rectangle (1.202,-0.478);
\fill[cBaF!42] (1.228,-0.762) rectangle (1.474,-0.478);
\fill[cBaF!42] (1.500,-0.762) rectangle (1.746,-0.478);
\fill[cBaD!85] (1.772,-0.762) rectangle (2.019,-0.478);
\fill[cBaD!85] (2.045,-0.762) rectangle (2.291,-0.478);
\fill[cBaF!42] (2.317,-0.762) rectangle (2.563,-0.478);
\fill[cBaF!42] (2.589,-0.762) rectangle (2.836,-0.478);
\fill[cBaF!42] (2.862,-0.762) rectangle (3.108,-0.478);
\fill[cBaF!42] (3.134,-0.762) rectangle (3.380,-0.478);
\fill[cBaF!42] (3.406,-0.762) rectangle (3.653,-0.478);
\fill[cBaF!42] (3.679,-0.762) rectangle (3.925,-0.478);
\fill[cBaF!42] (3.951,-0.762) rectangle (4.197,-0.478);
\fill[cBaD!85] (4.223,-0.762) rectangle (4.470,-0.478);
\fill[cBaF!42] (4.496,-0.762) rectangle (4.742,-0.478);
\fill[cBaF!42] (4.818,-0.762) rectangle (5.064,-0.478);
\draw[fill=cBaF,draw=cBaD!90,line width=0.75pt] (4.830,-0.750) rectangle (5.052,-0.490);
\fill[cBaF!42] (5.090,-0.762) rectangle (5.337,-0.478);
\fill[cBaF!42] (5.363,-0.762) rectangle (5.609,-0.478);
\fill[cBaF!42] (5.635,-0.762) rectangle (5.881,-0.478);
\fill[cBaF!42] (5.957,-0.762) rectangle (6.204,-0.478);
\fill[cBaF!42] (6.230,-0.762) rectangle (6.476,-0.478);
\draw[fill=cBaF,draw=cBaD!90,line width=0.75pt] (6.242,-0.750) rectangle (6.464,-0.490);
\fill[cBaF!42] (6.502,-0.762) rectangle (6.748,-0.478);
\fill[cBaF!42] (6.774,-0.762) rectangle (7.020,-0.478);
\fill[cTgD!40] (7.076,-0.762) rectangle (7.323,-0.478);
\fill[cTgD!40] (7.349,-0.762) rectangle (7.595,-0.478);
\fill[cTgD!40] (7.621,-0.762) rectangle (7.867,-0.478);
\fill[cTgD!40] (7.893,-0.762) rectangle (8.140,-0.478);
\fill[cTgD!85] (8.196,-0.762) rectangle (8.442,-0.478);
\fill[cTgD!85] (8.468,-0.762) rectangle (8.714,-0.478);
\fill[cTgD!85] (8.740,-0.762) rectangle (8.987,-0.478);
\fill[cTgD!85] (9.013,-0.762) rectangle (9.259,-0.478);
\fill[mskF] (9.315,-0.762) rectangle (9.561,-0.478);
\fill[mskF] (9.587,-0.762) rectangle (9.834,-0.478);
\fill[mskF] (9.860,-0.762) rectangle (10.106,-0.478);
\fill[mskF] (10.132,-0.762) rectangle (10.378,-0.478);
\fill[mskF] (10.434,-0.762) rectangle (10.828,-0.478);
\node[dt,text=black!45] at (10.631,-0.620) {$\cdots$};
\fill[mskF] (10.884,-0.762) rectangle (11.131,-0.478);
\fill[mskF] (11.157,-0.762) rectangle (11.403,-0.478);
\fill[mskF] (11.429,-0.762) rectangle (11.675,-0.478);
\fill[mskF] (11.701,-0.762) rectangle (11.948,-0.478);
\fill[mskF] (12.004,-0.762) rectangle (12.250,-0.478);
\fill[mskF] (12.276,-0.762) rectangle (12.522,-0.478);
\fill[mskF] (12.548,-0.762) rectangle (12.795,-0.478);
\fill[mskF] (12.821,-0.762) rectangle (13.067,-0.478);
\node[rl] at (0.540,-0.620) {$T_{3}$};
\node[cnt] at (13.390,-0.620) {$15$};
\fill[cAnD!80] (0.633,-1.072) rectangle (0.879,-0.788);
\fill[cBaF!42] (0.955,-1.072) rectangle (1.202,-0.788);
\fill[cBaD!85] (1.228,-1.072) rectangle (1.474,-0.788);
\fill[cBaF!42] (1.500,-1.072) rectangle (1.746,-0.788);
\fill[cBaF!42] (1.772,-1.072) rectangle (2.019,-0.788);
\fill[cBaD!85] (2.045,-1.072) rectangle (2.291,-0.788);
\fill[cBaF!42] (2.317,-1.072) rectangle (2.563,-0.788);
\fill[cBaD!85] (2.589,-1.072) rectangle (2.836,-0.788);
\fill[cBaD!85] (2.862,-1.072) rectangle (3.108,-0.788);
\fill[cBaF!42] (3.134,-1.072) rectangle (3.380,-0.788);
\fill[cBaF!42] (3.406,-1.072) rectangle (3.653,-0.788);
\fill[cBaF!42] (3.679,-1.072) rectangle (3.925,-0.788);
\fill[cBaF!42] (3.951,-1.072) rectangle (4.197,-0.788);
\fill[cBaF!42] (4.223,-1.072) rectangle (4.470,-0.788);
\fill[cBaF!42] (4.496,-1.072) rectangle (4.742,-0.788);
\fill[cBaF!42] (4.818,-1.072) rectangle (5.064,-0.788);
\fill[cBaF!42] (5.090,-1.072) rectangle (5.337,-0.788);
\fill[cBaF!42] (5.363,-1.072) rectangle (5.609,-0.788);
\fill[cBaF!42] (5.635,-1.072) rectangle (5.881,-0.788);
\fill[cBaF!42] (5.957,-1.072) rectangle (6.204,-0.788);
\fill[cBaF!42] (6.230,-1.072) rectangle (6.476,-0.788);
\fill[cBaF!42] (6.502,-1.072) rectangle (6.748,-0.788);
\draw[fill=cBaF,draw=cBaD!90,line width=0.75pt] (6.514,-1.060) rectangle (6.736,-0.800);
\fill[cBaF!42] (6.774,-1.072) rectangle (7.020,-0.788);
\fill[cBaF!42] (7.076,-1.072) rectangle (7.323,-0.788);
\fill[cBaF!42] (7.349,-1.072) rectangle (7.595,-0.788);
\fill[cBaF!42] (7.621,-1.072) rectangle (7.867,-0.788);
\fill[cBaF!42] (7.893,-1.072) rectangle (8.140,-0.788);
\draw[fill=cBaF,draw=cBaD!90,line width=0.75pt] (7.905,-1.060) rectangle (8.128,-0.800);
\fill[cTgD!40] (8.196,-1.072) rectangle (8.442,-0.788);
\fill[cTgD!40] (8.468,-1.072) rectangle (8.714,-0.788);
\fill[cTgD!40] (8.740,-1.072) rectangle (8.987,-0.788);
\fill[cTgD!40] (9.013,-1.072) rectangle (9.259,-0.788);
\fill[cTgD!85] (9.315,-1.072) rectangle (9.561,-0.788);
\fill[cTgD!85] (9.587,-1.072) rectangle (9.834,-0.788);
\fill[cTgD!85] (9.860,-1.072) rectangle (10.106,-0.788);
\fill[cTgD!85] (10.132,-1.072) rectangle (10.378,-0.788);
\fill[mskF] (10.434,-1.072) rectangle (10.828,-0.788);
\node[dt,text=black!45] at (10.631,-0.930) {$\cdots$};
\fill[mskF] (10.884,-1.072) rectangle (11.131,-0.788);
\fill[mskF] (11.157,-1.072) rectangle (11.403,-0.788);
\fill[mskF] (11.429,-1.072) rectangle (11.675,-0.788);
\fill[mskF] (11.701,-1.072) rectangle (11.948,-0.788);
\fill[mskF] (12.004,-1.072) rectangle (12.250,-0.788);
\fill[mskF] (12.276,-1.072) rectangle (12.522,-0.788);
\fill[mskF] (12.548,-1.072) rectangle (12.795,-0.788);
\fill[mskF] (12.821,-1.072) rectangle (13.067,-0.788);
\node[rl] at (0.540,-0.930) {$T_{4}$};
\node[cnt] at (13.390,-0.930) {$15$};
\node[dt] at (0.756,-1.270) {$\vdots$};
\node[dt] at (2.849,-1.270) {$\vdots$};
\node[dt] at (5.350,-1.270) {$\vdots$};
\node[dt] at (6.489,-1.270) {$\vdots$};
\node[dt] at (7.608,-1.270) {$\vdots$};
\node[dt] at (8.727,-1.270) {$\vdots$};
\node[dt] at (9.847,-1.270) {$\vdots$};
\node[dt] at (10.631,-1.270) {$\vdots$};
\node[dt] at (11.416,-1.270) {$\vdots$};
\node[dt] at (12.535,-1.270) {$\vdots$};
\node[rl] at (0.540,-1.270) {$\vdots$};
\fill[cAnD!80] (0.633,-1.822) rectangle (0.879,-1.538);
\fill[cBaD!85] (0.955,-1.822) rectangle (1.202,-1.538);
\fill[cBaF!42] (1.228,-1.822) rectangle (1.474,-1.538);
\fill[cBaF!42] (1.500,-1.822) rectangle (1.746,-1.538);
\fill[cBaD!85] (1.772,-1.822) rectangle (2.019,-1.538);
\fill[cBaF!42] (2.045,-1.822) rectangle (2.291,-1.538);
\fill[cBaF!42] (2.317,-1.822) rectangle (2.563,-1.538);
\fill[cBaF!42] (2.589,-1.822) rectangle (2.836,-1.538);
\fill[cBaF!42] (2.862,-1.822) rectangle (3.108,-1.538);
\fill[cBaF!42] (3.134,-1.822) rectangle (3.380,-1.538);
\fill[cBaF!42] (3.406,-1.822) rectangle (3.653,-1.538);
\fill[cBaF!42] (3.679,-1.822) rectangle (3.925,-1.538);
\fill[cBaF!42] (3.951,-1.822) rectangle (4.197,-1.538);
\fill[cBaD!85] (4.223,-1.822) rectangle (4.470,-1.538);
\fill[cBaF!42] (4.496,-1.822) rectangle (4.742,-1.538);
\fill[cBaF!42] (4.818,-1.822) rectangle (5.064,-1.538);
\fill[cBaF!42] (5.090,-1.822) rectangle (5.337,-1.538);
\fill[cBaF!42] (5.363,-1.822) rectangle (5.609,-1.538);
\draw[fill=cBaF,draw=cBaD!90,line width=0.75pt] (5.375,-1.810) rectangle (5.597,-1.550);
\fill[cBaF!42] (5.635,-1.822) rectangle (5.881,-1.538);
\fill[cBaF!42] (5.957,-1.822) rectangle (6.204,-1.538);
\fill[cBaF!42] (6.230,-1.822) rectangle (6.476,-1.538);
\fill[cBaF!42] (6.502,-1.822) rectangle (6.748,-1.538);
\fill[cBaF!42] (6.774,-1.822) rectangle (7.020,-1.538);
\fill[cBaF!42] (7.076,-1.822) rectangle (7.323,-1.538);
\draw[fill=cBaF,draw=cBaD!90,line width=0.75pt] (7.088,-1.810) rectangle (7.311,-1.550);
\fill[cBaF!42] (7.349,-1.822) rectangle (7.595,-1.538);
\fill[cBaF!42] (7.621,-1.822) rectangle (7.867,-1.538);
\fill[cBaF!42] (7.893,-1.822) rectangle (8.140,-1.538);
\fill[cBaF!42] (8.196,-1.822) rectangle (8.442,-1.538);
\fill[cBaF!42] (8.468,-1.822) rectangle (8.714,-1.538);
\draw[fill=cBaF,draw=cBaD!90,line width=0.75pt] (8.480,-1.810) rectangle (8.702,-1.550);
\fill[cBaF!42] (8.740,-1.822) rectangle (8.987,-1.538);
\fill[cBaF!42] (9.013,-1.822) rectangle (9.259,-1.538);
\fill[cBaF!42] (9.315,-1.822) rectangle (9.561,-1.538);
\fill[cBaF!42] (9.587,-1.822) rectangle (9.834,-1.538);
\fill[cBaF!42] (9.860,-1.822) rectangle (10.106,-1.538);
\fill[cBaF!42] (10.132,-1.822) rectangle (10.378,-1.538);
\fill[cTgD!40] (10.434,-1.822) rectangle (10.828,-1.538);
\node[dt,text=black!45] at (10.631,-1.680) {$\cdots$};
\fill[cTgD!85] (10.884,-1.822) rectangle (11.131,-1.538);
\fill[cTgD!85] (11.157,-1.822) rectangle (11.403,-1.538);
\fill[cTgD!85] (11.429,-1.822) rectangle (11.675,-1.538);
\fill[cTgD!85] (11.701,-1.822) rectangle (11.948,-1.538);
\fill[mskF] (12.004,-1.822) rectangle (12.250,-1.538);
\fill[mskF] (12.276,-1.822) rectangle (12.522,-1.538);
\fill[mskF] (12.548,-1.822) rectangle (12.795,-1.538);
\fill[mskF] (12.821,-1.822) rectangle (13.067,-1.538);
\node[rl] at (0.540,-1.680) {$T_{k-1}$};
\node[cnt] at (13.390,-1.680) {$15$};
\fill[cAnD!80] (0.633,-2.132) rectangle (0.879,-1.848);
\fill[cBaF!42] (0.955,-2.132) rectangle (1.202,-1.848);
\fill[cBaF!42] (1.228,-2.132) rectangle (1.474,-1.848);
\fill[cBaD!85] (1.500,-2.132) rectangle (1.746,-1.848);
\fill[cBaF!42] (1.772,-2.132) rectangle (2.019,-1.848);
\fill[cBaD!85] (2.045,-2.132) rectangle (2.291,-1.848);
\fill[cBaD!85] (2.317,-2.132) rectangle (2.563,-1.848);
\fill[cBaF!42] (2.589,-2.132) rectangle (2.836,-1.848);
\fill[cBaF!42] (2.862,-2.132) rectangle (3.108,-1.848);
\fill[cBaF!42] (3.134,-2.132) rectangle (3.380,-1.848);
\fill[cBaF!42] (3.406,-2.132) rectangle (3.653,-1.848);
\fill[cBaF!42] (3.679,-2.132) rectangle (3.925,-1.848);
\fill[cBaF!42] (3.951,-2.132) rectangle (4.197,-1.848);
\fill[cBaF!42] (4.223,-2.132) rectangle (4.470,-1.848);
\fill[cBaF!42] (4.496,-2.132) rectangle (4.742,-1.848);
\fill[cBaF!42] (4.818,-2.132) rectangle (5.064,-1.848);
\fill[cBaF!42] (5.090,-2.132) rectangle (5.337,-1.848);
\fill[cBaF!42] (5.363,-2.132) rectangle (5.609,-1.848);
\fill[cBaF!42] (5.635,-2.132) rectangle (5.881,-1.848);
\fill[cBaF!42] (5.957,-2.132) rectangle (6.204,-1.848);
\draw[fill=cBaF,draw=cBaD!90,line width=0.75pt] (5.969,-2.120) rectangle (6.192,-1.860);
\fill[cBaF!42] (6.230,-2.132) rectangle (6.476,-1.848);
\draw[fill=cBaF,draw=cBaD!90,line width=0.75pt] (6.242,-2.120) rectangle (6.464,-1.860);
\fill[cBaF!42] (6.502,-2.132) rectangle (6.748,-1.848);
\fill[cBaF!42] (6.774,-2.132) rectangle (7.020,-1.848);
\fill[cBaF!42] (7.076,-2.132) rectangle (7.323,-1.848);
\fill[cBaF!42] (7.349,-2.132) rectangle (7.595,-1.848);
\fill[cBaF!42] (7.621,-2.132) rectangle (7.867,-1.848);
\fill[cBaF!42] (7.893,-2.132) rectangle (8.140,-1.848);
\fill[cBaF!42] (8.196,-2.132) rectangle (8.442,-1.848);
\fill[cBaF!42] (8.468,-2.132) rectangle (8.714,-1.848);
\fill[cBaF!42] (8.740,-2.132) rectangle (8.987,-1.848);
\fill[cBaF!42] (9.013,-2.132) rectangle (9.259,-1.848);
\fill[cBaF!42] (9.315,-2.132) rectangle (9.561,-1.848);
\fill[cBaF!42] (9.587,-2.132) rectangle (9.834,-1.848);
\fill[cBaF!42] (9.860,-2.132) rectangle (10.106,-1.848);
\draw[fill=cBaF,draw=cBaD!90,line width=0.75pt] (9.872,-2.120) rectangle (10.094,-1.860);
\fill[cBaF!42] (10.132,-2.132) rectangle (10.378,-1.848);
\fill[cBaF!42] (10.434,-2.132) rectangle (10.828,-1.848);
\node[dt,text=black!45] at (10.631,-1.990) {$\cdots$};
\fill[cTgD!40] (10.884,-2.132) rectangle (11.131,-1.848);
\fill[cTgD!40] (11.157,-2.132) rectangle (11.403,-1.848);
\fill[cTgD!40] (11.429,-2.132) rectangle (11.675,-1.848);
\fill[cTgD!40] (11.701,-2.132) rectangle (11.948,-1.848);
\fill[cTgD!85] (12.004,-2.132) rectangle (12.250,-1.848);
\fill[cTgD!85] (12.276,-2.132) rectangle (12.522,-1.848);
\fill[cTgD!85] (12.548,-2.132) rectangle (12.795,-1.848);
\fill[cTgD!85] (12.821,-2.132) rectangle (13.067,-1.848);
\node[rl] at (0.540,-1.990) {$T_{k}$};
\node[cnt] at (13.390,-1.990) {$15$};
\node[inner sep=0pt] at (6.850,-2.785) {\legendrowa};
\node[inner sep=0pt] at (6.850,-3.085) {\legendrowb};
\end{tikzpicture}
\caption{\textbf{The training sequence and what each scored chunk reads.} \emph{Top:} the
packed sequence and its RoPE index; bank entries share one index. \emph{Grid:} the
attention mask, one row per scored chunk. Each chunk attends to the first frame, its
window of $r{+}1$ chunks and its own top-$K$: $1{+}K{+}(r{+}1)c = 15$ latents, whatever
the span. Drawn for $K{=}6$, $r{=}1$, $c{=}4$.}
\label{fig:overview}
\end{figure}
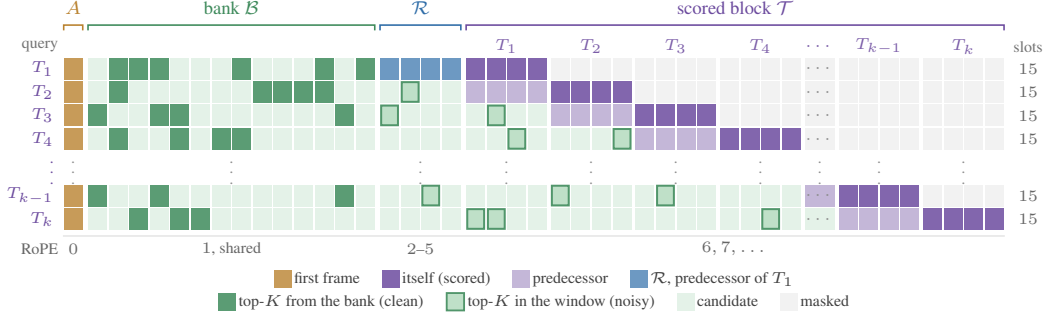

%% file: figures/fig_selector.tex
\definecolor{selOv}{HTML}{EB6834}
\definecolor{selMix}{HTML}{2A78D6}
\definecolor{selPose}{HTML}{1BAF7A}
\begin{figure}[t]
\centering
\begin{minipage}[c]{0.43\textwidth}
\centering
\begin{tikzpicture}[font=\scriptsize]
\begin{axis}[
  xbar, bar width=3.1pt, bar shift=0pt,
  width=2.75cm, height=1.38cm, scale only axis=true,
  xmin=0, xmax=100, enlarge y limits=0.09,
  symbolic y coords={trans,rot,pose,ovl,l01,l02},
  ytick=data,
  yticklabels={},
  y dir=normal,
  xtick={0,25,50,75,100}, xticklabel={\pgfmathprintnumber{\tick}\%},
  tick label style={font=\tiny,text=black!60},
  xlabel={fraction of achievable gain}, xlabel style={font=\tiny,text=black!70,yshift=3pt},
  axis x line*=bottom, axis y line*=left,
  axis line style={black!35,line width=0.4pt},
  xmajorgrids, grid style={black!8,line width=0.3pt},
  major tick length=2pt, ytick style={draw=none},
  clip=false,
  nodes near coords, nodes near coords style={font=\tiny,text=black!75,anchor=west,xshift=5pt},
  point meta=explicit symbolic,
]
  \addplot[fill=selMix,draw=none,error bars/.cd,x dir=both,x explicit,error bar style={black!55,line width=0.4pt},error mark options={black!55,mark size=1pt,line width=0.4pt,rotate=90}]        coordinates {(57.2,l02) +- (2.07,0) [57.2]};
  \addplot[fill=selMix!55,draw=none,error bars/.cd,x dir=both,x explicit,error bar style={black!55,line width=0.4pt},error mark options={black!55,mark size=1pt,line width=0.4pt,rotate=90}]     coordinates {(57.3,l01) +- (2.07,0) [57.3]};
  \addplot[fill=selOv,draw=none,error bars/.cd,x dir=both,x explicit,error bar style={black!55,line width=0.4pt},error mark options={black!55,mark size=1pt,line width=0.4pt,rotate=90}]         coordinates {(47.4,ovl) +- (2.04,0) [47.4]};
  \addplot[fill=selPose,draw=none,error bars/.cd,x dir=both,x explicit,error bar style={black!55,line width=0.4pt},error mark options={black!55,mark size=1pt,line width=0.4pt,rotate=90}]       coordinates {(55.9,pose) +- (2.1,0) [55.9]};
  \addplot[fill=selPose!55,draw=none,error bars/.cd,x dir=both,x explicit,error bar style={black!55,line width=0.4pt},error mark options={black!55,mark size=1pt,line width=0.4pt,rotate=90}]    coordinates {(53.7,rot) +- (2.16,0) [53.7]};
  \addplot[fill=selPose!30,draw=none,error bars/.cd,x dir=both,x explicit,error bar style={black!55,line width=0.4pt},error mark options={black!55,mark size=1pt,line width=0.4pt,rotate=90}]    coordinates {(19.8,trans) +- (2.27,0) [19.8]};
  \draw[black!45,line width=0.5pt,densely dashed] ([yshift=-6pt]axis cs:100,trans) -- ([yshift=6pt]axis cs:100,l02);
  \node[font=\tiny,text=black!55,anchor=south] at ([yshift=6pt]axis cs:100,l02) {ideal};
  \node[font=\tiny,text=black!80,anchor=east] at ([xshift=-3pt]axis cs:0,l02) {$\Ovl-0.2\,\Pdist$ \textbf{(ours)}};
  \node[font=\tiny,text=black!80,anchor=east] at ([xshift=-3pt]axis cs:0,l01) {$\Ovl-0.1\,\Pdist$};
  \node[font=\tiny,text=black!80,anchor=east] at ([xshift=-3pt]axis cs:0,ovl) {$\Ovl$ alone};
  \node[font=\tiny,text=black!80,anchor=east] at ([xshift=-3pt]axis cs:0,pose) {$\Pdist$ alone};
  \node[font=\tiny,text=black!80,anchor=east] at ([xshift=-3pt]axis cs:0,rot) {$\Pdist$, rotation only};
  \node[font=\tiny,text=black!80,anchor=east] at ([xshift=-3pt]axis cs:0,trans) {$\Pdist$, translation only};
  \draw[black!18,line width=0.4pt] ([xshift=-2.75cm,yshift=-3.3pt]axis cs:0,ovl) -- ([yshift=-3.3pt]axis cs:100,ovl);
\end{axis}
\end{tikzpicture}
\end{minipage}\hfill
\begin{minipage}[c]{0.56\textwidth}
\centering
\begin{tikzpicture}[
  img/.style={inner sep=0pt,outer sep=0pt},
  ttl/.style={font=\tiny,inner sep=1pt,anchor=south,text=black!80},
  val/.style={font=\tiny,inner sep=1pt,anchor=north,text=black!55},
]
\def\W{1.79}\def\G{0.12}
\node[img,anchor=west] (q) at (0,0) {\includegraphics[width=\W cm]{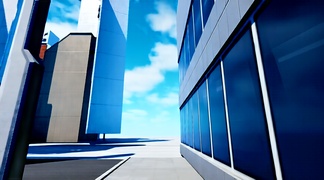}};
\node[ttl] at (q.north) {view being drawn};
\node[val] at (q.south) {$152$ candidates};
\foreach \n/\k/\col/\lab/\ov/\dd in {%
  lam0/1/selOv/{$\Ovl$ alone}/0.95/4.67,
  lam02/2/selMix/{$\Ovl-0.2\,\Pdist$ (ours)}/0.84/1.49,
  dist/3/selPose/{$\Pdist$ alone}/0.75/1.12}{
  \node[img,anchor=west] (f\n) at ({\k*(\W+\G)+0.06},0) {\includegraphics[width=\W cm]{figures/f3_\n.jpg}};
  \draw[\col,line width=1.1pt] (f\n.north west) rectangle (f\n.south east);
  \node[ttl] at (f\n.north) {\lab};
  \node[val] at (f\n.south) {$\Ovl{=}\ov$, $\Pdist{=}\dd$};}
\draw[black!20,line width=0.4pt] ({\W+\G/2+0.03},0.62) -- ({\W+\G/2+0.03},-0.62);
\end{tikzpicture}
\end{minipage}
\caption{\textbf{Retrieval criterion graded against ground truth.} \emph{Left:} fraction of
achievable gain over $960$ scored chunks, from random ($0\%$) to the best candidates
($100\%$); $\lambda = 0.1$ and $0.2$ are level within the standard error; bars show $\pm 1$ standard error over chunks. \emph{Right:} the
frame each rule ranks first for one chunk.}
\label{fig:selector}
\end{figure}

%% file: sections/04_experiments.tex
\section{Experiments}
\label{sec:exp}

We first describe training and evaluation (\hyperref[sec:setup]{Sec.~\ref*{sec:setup}}), then compare \ours with open world models (\hyperref[sec:sota]{Sec.~\ref*{sec:sota}}). An ablation adds retrieval, the bank and a longer span one at a time (\hyperref[sec:span]{Sec.~\ref*{sec:span}}).

\subsection{Experimental Setup}
\label{sec:setup}

\input{tables/sota}

\paragraph{Training Details.} We initialize from Wan2.2-TI2V-5B \citep{wan}, add a camera branch with PRoPE \citep{prope}, and train the model as a chunked causal diffusion transformer using diffusion forcing \citep{diffusionforcing}. Each chunk contains $c = 4$ latents. Each sample scores $k = 10$ chunks behind $r = 1$ recent chunk, with the span sampled as $\spanm \sim U[10, \spanm_{\max}]$. Each scored chunk independently selects its top-$K$ entries, with $K = 6$, using \hyperref[eq:score]{Eq.~\ref*{eq:score}}. Entries preceding the window are accessed through the union bank, while entries within it are accessed through the mask. All runs use the same initialization, data, and training schedule, and train for $6{,}000$ steps with a batch size of $32$, using one sample per GPU across $32$ H200 GPUs. \hyperref[app:trainingdetail]{App.~\ref*{app:trainingdetail}} details the optimizer, conditioning, and sampler; \hyperref[app:layout]{App.~\ref*{app:layout}} reports the resulting sequence sizes.

\paragraph{Evaluation.} For each model, we average five rollouts with different seeds, using $20$ denoising steps and classifier-free guidance at scale $4$. The evaluation set comprises $20$ $60$\,s clips from ten training scenes, $16$ clips from four held-out scenes, and $20$ web photographs. Rollouts on the latter two splits last $400$\,s; all tables report the first $60$\,s.

\paragraph{Metrics.} Two latents form a \emph{return} when their cameras lie within $0.5$ units and $15^\circ$ of each other, at least $8$\,s apart, with the camera away in between. \emph{Revisit} is the Pearson correlation between the frames a model generates at the two ends: it asks whether a model draws a place as it drew it before. Frames of one video correlate even without a return, so \emph{Gain} subtracts the correlation of control pairs at the same time gaps whose cameras differ. \emph{DINO} scores the same returns by the cosine similarity of DINO ViT-B/16 features \citep{dino} instead, which tolerates the small misalignments a pixel correlation penalizes. A return is \emph{to the starting pose} when either end meets the first frame's camera and \emph{mid-path} otherwise; the first frame is in every sequence, so only the second kind needs the bank. We also report PSNR and LPIPS \citep{lpips} against the rendered ground truth where it exists, and VBench \citep{vbench} for consistency and quality. \hyperref[app:eval]{App.~\ref*{app:eval}} gives the rest.

\subsection{Comparison with SOTA}
\label{sec:sota}

We compare against open world models including: LingBot-World~2.0 \citep{lingbotworldv2}, DreamX-World~1.0 \citep{dreamxworld}, HY-WorldPlay~1.5 \citep{hyworld15}, Matrix-Game~3.0 \citep{matrixgame3}, Infinite-World \citep{infiniteworld} and CaR \citep{car}.
In \hyperref[tab:sota]{Table~\ref*{tab:sota}}, every method rolls out $60$\,s from the same photographs along the same path, using its released weights, default sampler, native resolution and frame rate, and its own intrinsics; a method driven by discrete actions receives the path converted to its action vocabulary. For Revisit, frames are matched to the path by timestamp and resized to $864\times480$. The action vocabulary of Matrix-Game~3.0 turns more slowly than the fastest paths, which it therefore follows only approximately.

\ours and CaR are the only systems with a clear memory. \ours is highest in five of the eight
memory columns and CaR in the other three: $0.365$ Revisit at the starting pose and $0.447$
mid-path for \ours, $0.350$ and $0.450$ for CaR, against $0.254$ and $0.283$ for Matrix-Game~3.0,
the strongest of the other baselines. In DINO features the two split by kind of return: CaR is
ahead at the starting pose, which the first frame can serve, and \ours on mid-path returns,
which need what was generated along the way.
The Gain columns are the sharper statement, because Gain removes what any two frames of one
video share at the same time gap. For the other five baselines it is at most $0.089$, and for
LingBot-World~2.0 at the starting pose it is $-0.006$: what these models draw at a return is
about what an ordinary pair of frames already agrees on, so little of it can be credited to
having been at that place before. Ours is $0.230$ and $0.329$, and CaR's $0.220$ and $0.312$. \hyperref[fig:revisit]{Figure~\ref*{fig:revisit}} shows one return, and \hyperref[app:visual]{App.~\ref*{app:visual}} (\hyperref[fig:sotamore]{Figure~\ref*{fig:sotamore}}) three more for every system.

A natural objection is that this margin measures camera control rather than memory: the five other baselines follow the prescribed path with correlations of $0.70$ to $0.88$ against our
$0.95$ to $0.97$ (\hyperref[app:camfollow]{App.~\ref*{app:camfollow}}). It cannot explain
the gap. Of the $173$ returns in this window, $125$ ($72\%$) are fold-backs, where the path
out is retraced on the way back, and on a fold-back any consistent error in the scale of a
model's turns and steps largely cancels, so the camera returns close to the pose it held on the first
visit (\hyperref[app:eval]{App.~\ref*{app:eval}}). What
differs at that pose is what the model draws there.

We do not lead the VBench columns and state it plainly. HY-WorldPlay~1.5 scores higher on
subject and background consistency, $0.839$ and $0.908$ against our $0.778$ and $0.857$,
and four of the six baselines score higher on imaging quality, up to $0.761$ against our
$0.705$. Part of this is what the columns measure: both consistency
scores compare neighbouring frames, so a rollout that drifts smoothly away from what it
drew a minute earlier scores well on them and poorly on Revisit. The rest is a real cost of
training on a longer span, which \hyperref[sec:span]{Sec.~\ref*{sec:span}} reports and
\hyperref[app:selfcond]{App.~\ref*{app:selfcond}} traces to conditioning on generated
frames rather than to a weaker generator.

\input{figures/fig_revisit}

\subsection{Ablation Study}
\label{sec:span}

\hyperref[tab:headline]{Table~\ref*{tab:headline}} adds one ingredient per row (seen scenes in \hyperref[tab:headline_seen]{Table~\ref*{tab:headline_seen}}). The first three rows train on $10$\,s clips: a sliding window of the first frame, the predecessor and the chunk itself; a window widened to the whole clip, which measures what more attention buys without retrieval; and the sliding window plus per-chunk top-$K$, whose candidates lie inside the clip in training but span the whole generated history at inference, as for every row. The next three raise the span to $100$\,s and differ only in how the bank is formed: one retrieval for the whole scored block, as in Context-as-Memory; per-chunk retrieval with the bank numbered in temporal order; and ours, with one shared index. The last two raise $\spanm_{\max}$ alone. All rows share the initialization, data, optimizer, step count and five rollout seeds of \hyperref[sec:setup]{Sec.~\ref*{sec:setup}}.

\input{tables/headline}

\paragraph{Effect of Retrieval.} Letting each chunk retrieve raises Revisit on every split and multiplies Gain several times over, already for a model trained on $10$\,s clips (\hyperref[fig:ablation]{Figure~\ref*{fig:ablation}}; more returns from \ours in \hyperref[fig:teasergrid0]{Figure~\ref*{fig:teasergrid0}}). At inference, that model retrieves from the whole generated history like every other row, and since retrieval is by pose and the bank shares one index, nothing tells it how old a retrieved frame is: what it learns on frames less than $10$\,s old carries over to frames retrieved from far further back (\hyperref[app:tenswhy]{App.~\ref*{app:tenswhy}}). The rise is smallest on seen scenes, whose places the model has already learned in training and can partly redraw without retrieval, and largest on web photographs, where only the retrieved frames tell the model what the place looked like. More attention is no substitute. Attending to every earlier chunk of the clip without retrieval helps on unseen scenes, is level on seen ones and lowers Revisit on web photographs, on $16$ of the $20$ photographs and in every one of the five rollout seeds. Apart from the first frame, such a window holds only frames the model generated itself. In the rendered domain those frames stay close to real; from a photograph outside it, each chunk drifts slightly, and a longer window keeps conditioning on that drift.

Who retrieves matters as much as whether. Sharing one retrieval across the whole scored block, the arrangement of Context-as-Memory, lowers Revisit and Gain on every split at the same span, slightly on rendered scenes and by about a third on web photographs: one retrieval cannot match the view of every chunk in the block, so most chunks read frames chosen for another pose. The model also reads what the bank holds, not only how many frames it has: varying only the contents of the bank's slots at inference, with checkpoint, trajectory and seed fixed, mid-path Revisit halves with random frames from the same history and falls further with an empty bank or frames from another episode (\hyperref[app:bankuse]{App.~\ref*{app:bankuse}}).

\paragraph{Effect of Training Span.} Training on longer spans raises Revisit and Gain further on every split, most at $200$\,s, and the gain over $10$\,s top-$K$ lies mainly in returns more than $10$\,s apart and in mid-path returns (\hyperref[app:gapsplit]{App.~\ref*{app:gapsplit}}). At $400$\,s both fall back on every split, to about the level of $100$\,s and below it on web photographs. Once the span covers the first visits of returns, more length only adds candidates no chunk chooses: a span of $100$, $200$ and $400$\,s reaches the first visit of $64\%$, $84\%$ and $93\%$ of returning latents in our corpus (\hyperref[app:data]{App.~\ref*{app:data}}), and the bank barely grows past $200$\,s, since it stops changing once new history offers no candidate that outscores a chunk's current $K$-th (\hyperref[fig:teaser]{Figure~\ref*{fig:teaser}}b). The span is a coverage setting, not a quantity to maximize. Image quality moves the other way. Imaging quality falls as the span grows to $200$\,s, most on unseen scenes and web photographs, and recovers at $400$\,s, where Revisit falls back. The cause is not a weaker generator: given a ground-truth history every model we tested, from $10$ to $200$\,s, reaches the quality of the rendered videos, and the loss comes from conditioning on its own output, most of it through the bank (\hyperref[app:selfcond]{App.~\ref*{app:selfcond}}). The $400$\,s run reverses both trends at once, with Revisit falling back while imaging quality recovers, which is consistent with a model that relies less on its bank, the path through which most of the quality loss enters.

\input{figures/fig_ablation}

\paragraph{Effect of the Bank Index.} Numbering the bank's entries in temporal order rather than giving them one shared index is not a clear win either way. At $100$\,s it raises Revisit on seen and unseen scenes, close to the $200$\,s run, and is level on web photographs. DINO Revisit, however, is higher only on seen scenes, and imaging quality is lower on every split: the ordered index reproduces the retrieved frames more closely in pixels without drawing the place better in features, and at a cost in image quality.

\paragraph{Returns Beyond the Training Window.} Our splits score the first $60$\,s, where returns are at most a minute apart, so they cannot separate spans longer than that. We therefore roll the same web photographs out to $400$\,s and score every return by the interval between its two visits, over the same five rollout seeds (\hyperref[fig:interval]{Figure~\ref*{fig:interval}}). The $10$\,s top-$K$ model, whose retrieval also reaches the whole history, decays as the interval grows, from $0.390 \pm 0.023$ under half a minute to $0.303 \pm 0.025$ at two to four minutes and $0.145 \pm 0.044$ beyond four. On the same returns the $200$\,s model is ahead of it in every bin, by $0.106 \pm 0.022$ under half a minute, $0.148 \pm 0.040$ at one to two minutes and $0.103 \pm 0.061$ beyond four, and the $100$\,s model is ahead in four of the five bins and level at two to four minutes ($-0.005 \pm 0.026$). The $400$\,s model does not extend this. It is level with $10$\,s top-$K$ under a minute ($-0.003 \pm 0.026$ and $+0.017 \pm 0.016$) and at two to four minutes ($-0.015 \pm 0.017$), and ahead only at one to two minutes ($+0.055 \pm 0.042$) and beyond four ($+0.087 \pm 0.052$), as in \hyperref[tab:headline]{Table~\ref*{tab:headline}}: a span that already covers the first visits gains nothing from more length. Returns more than four minutes apart are few, $30$ per seed, so the last bin is the least certain.

\input{figures/fig_interval}

%% file: tables/sota.tex
\begin{table}[t]
\caption{\textbf{Comparison with open world models} on the web photographs, with returns to the
starting pose and to mid-path poses scored apart. Mean over five seeds; \hyperref[tab:sota_std]{Table~\ref*{tab:sota_std}} gives the spreads. \ours: the $100$\,s run of \hyperref[tab:headline]{Table~\ref*{tab:headline}}. DINO and D.\,Gain
use DINO ViT-B/16 cosine in place of pixel correlation. $^\dagger$Tracks our walks only to within
$1.1$\,m and $34^\circ$.}
\label{tab:sota}
\centering\scriptsize
\setlength{\tabcolsep}{4pt}
\renewcommand{\arraystretch}{1.12}
\resizebox{\textwidth}{!}{%
\begin{tabular}{@{}l cccc cccc ccc@{}}
\toprule
& \multicolumn{4}{c}{Starting pose} & \multicolumn{4}{c}{Mid-path} & \multicolumn{3}{c}{VBench} \\
\cmidrule(lr){2-5}\cmidrule(lr){6-9}\cmidrule(l){10-12}
Method & Revisit$\uparrow$ & Gain$\uparrow$ & DINO$\uparrow$ & D.\,Gain$\uparrow$ & Revisit$\uparrow$ & Gain$\uparrow$ & DINO$\uparrow$ & D.\,Gain$\uparrow$ & Subj.$\uparrow$ & Backg.$\uparrow$ & Imag.$\uparrow$ \\
\midrule
LingBot-World~2.0 & 0.143 & $-$0.006 & 0.526 & $-$0.001 & 0.147 & 0.011 & 0.546 & 0.017 & 0.765 & 0.852 & 0.741 \\
DreamX-World~1.0 & 0.146 & 0.033 & 0.547 & $-$0.045 & 0.109 & 0.009 & 0.591 & 0.007 & 0.781 & 0.868 & 0.684 \\
HY-WorldPlay~1.5$^\dagger$ & 0.243 & 0.040 & 0.679 & 0.025 & 0.194 & 0.037 & 0.674 & 0.027 & \best{0.839} & \best{0.908} & 0.702 \\
Matrix-Game~3.0 & 0.254 & 0.075 & 0.717 & 0.042 & 0.283 & 0.089 & \second{0.760} & 0.100 & \second{0.813} & \second{0.895} & \best{0.761} \\
Infinite-World & 0.109 & 0.019 & 0.610 & 0.013 & 0.168 & 0.070 & 0.646 & 0.052 & 0.789 & 0.856 & \second{0.746} \\
CaR & \second{0.350} & \second{0.220} & \best{0.843} & \best{0.305} & \best{0.450} & \second{0.312} & 0.759 & \second{0.211} & 0.787 & 0.846 & 0.715 \\
\rowcolor{ourrow}
\multicolumn{1}{@{}>{\columncolor{ourrow}[0pt][\tabcolsep]}l}{\textbf{\ours (ours)}} & \best{0.365} & \best{0.230} & \second{0.787} & \second{0.219} & \second{0.447} & \best{0.329} & \best{0.816} & \best{0.258} & 0.778 & 0.857 & \multicolumn{1}{>{\columncolor{ourrow}[\tabcolsep][0pt]}c@{}}{0.705} \\
\bottomrule
\end{tabular}}
\end{table}

%% file: figures/fig_revisit.tex
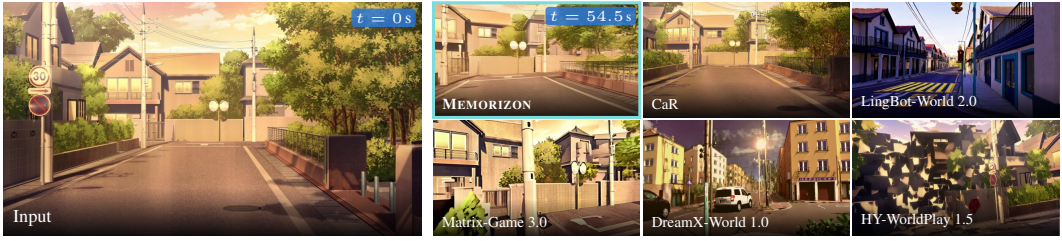
\begin{figure}[t]
\centering
\resizebox{\textwidth}{!}{%
\begin{tikzpicture}[
  img/.style={inner sep=0pt,outer sep=0pt},
  tag/.style={font=\fontsize{4.6}{5.6}\selectfont,text=white,fill opacity=0.85,text opacity=1,rounded corners=1.2pt,inner sep=1.4pt},
  mname/.style={font=\fontsize{5.0}{6}\selectfont,text=white,inner sep=1.5pt,anchor=south west},
]
\node[img,anchor=north west] (in) at (0,0) {\includegraphics[width=4.868cm]{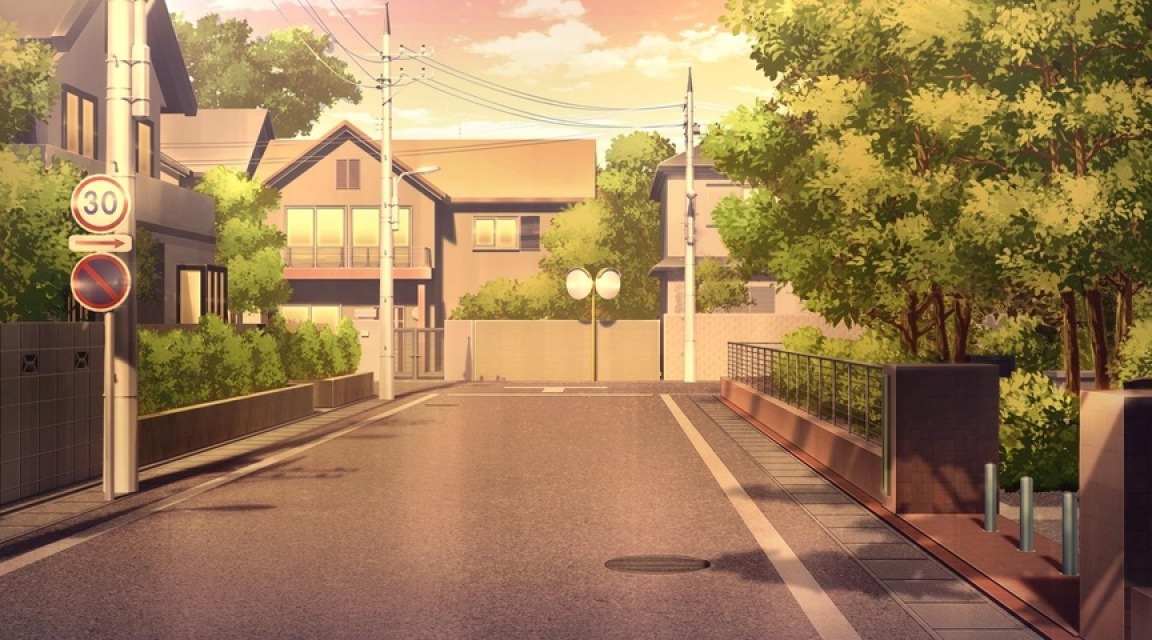}};
\fill[black,path fading=north,fill opacity=0.85] (0,-2.705) rectangle (4.868,-2.155);
\node[mname,font=\fontsize{6.0}{7}\selectfont] at ([shift={(0.06,0.05)}]in.south west) {Input};
\node[tag,fill=selMix,anchor=north east,font=\fontsize{5.4}{6.4}\selectfont] at ([shift={(-0.06,-0.06)}]in.north east) {$t = 0$\,s};
\node[img,anchor=north west] (nours) at (5.008,0.000) {\includegraphics[width=2.4cm]{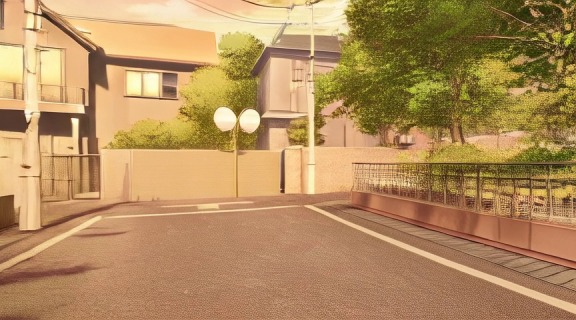}};
\fill[black,path fading=north,fill opacity=0.85] (5.008,-1.333) rectangle (7.408,-0.913);
\draw[ourline,line width=1.0pt] (nours.north west) rectangle (nours.south east);
\node[mname,font=\fontsize{5.0}{6}\selectfont\bfseries] at ([shift={(0.04,0.03)}]nours.south west) {\ours};
\node[tag,fill=selMix,anchor=north east] at ([shift={(-0.04,-0.04)}]nours.north east) {$t = 54.5$\,s};
\node[img,anchor=north west] (ncar) at (7.446,0.000) {\includegraphics[width=2.4cm]{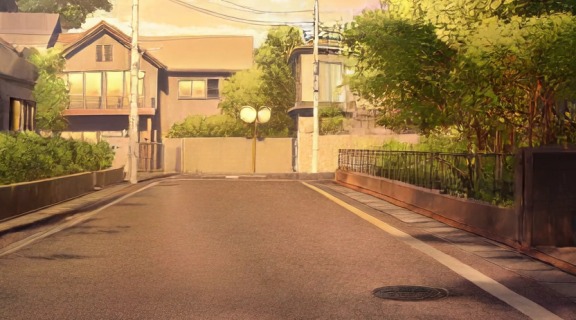}};
\fill[black,path fading=north,fill opacity=0.85] (7.446,-1.333) rectangle (9.846,-0.913);
\node[mname] at ([shift={(0.04,0.03)}]ncar.south west) {CaR};
\node[img,anchor=north west] (nlingbot) at (9.884,0.000) {\includegraphics[width=2.4cm]{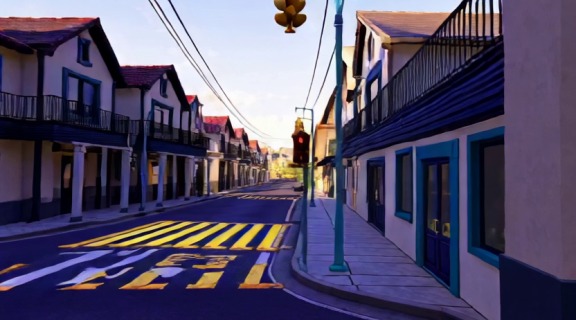}};
\fill[black,path fading=north,fill opacity=0.85] (9.884,-1.333) rectangle (12.284,-0.913);
\node[mname] at ([shift={(0.04,0.03)}]nlingbot.south west) {LingBot-World~2.0};
\node[img,anchor=north west] (nmatrix) at (5.008,-1.371) {\includegraphics[width=2.4cm]{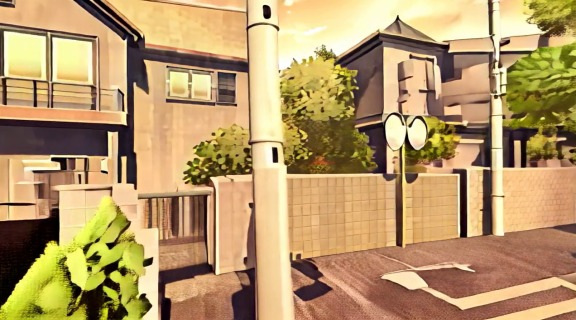}};
\fill[black,path fading=north,fill opacity=0.85] (5.008,-2.705) rectangle (7.408,-2.285);
\node[mname] at ([shift={(0.04,0.03)}]nmatrix.south west) {Matrix-Game~3.0};
\node[img,anchor=north west] (ndreamx) at (7.446,-1.371) {\includegraphics[width=2.4cm]{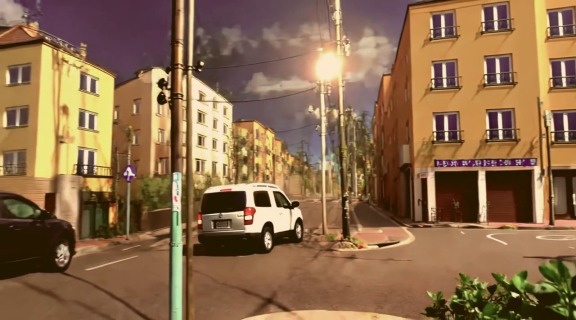}};
\fill[black,path fading=north,fill opacity=0.85] (7.446,-2.705) rectangle (9.846,-2.285);
\node[mname] at ([shift={(0.04,0.03)}]ndreamx.south west) {DreamX-World~1.0};
\node[img,anchor=north west] (nhyworld) at (9.884,-1.371) {\includegraphics[width=2.4cm]{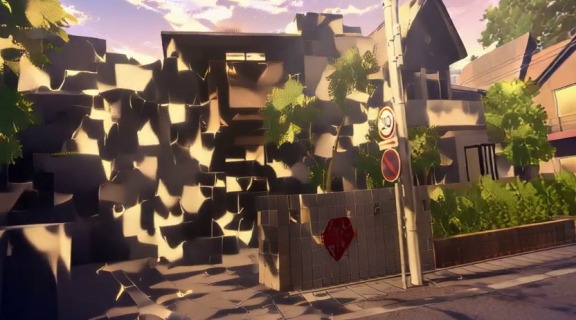}};
\fill[black,path fading=north,fill opacity=0.85] (9.884,-2.705) rectangle (12.284,-2.285);
\node[mname] at ([shift={(0.04,0.03)}]nhyworld.south west) {HY-WorldPlay~1.5};
\end{tikzpicture}}
\caption{\textbf{Comparison with SOTA} on a web photograph. Each system is shown late in the
walk, when the camera is back near the input view; the input is the reference.}
\label{fig:revisit}
\end{figure}

%% file: tables/headline.tex
\begin{table}[t]
\caption{\textbf{Ablation on training span, bank and retrieval.} Rows add one ingredient at
a time (\hyperref[sec:span]{Sec.~\ref*{sec:span}}); spans are drawn uniformly from $10$\,s
to the length given. PSNR and LPIPS need rendered ground truth, which web photographs
lack. All runs at step $6{,}000$; mean over five rollout seeds, standard deviation of the
per-seed means in small type. DINO is Revisit with the cosine of DINO ViT-B/16 features in place of pixel correlation. Best per split in bold, second underlined. Seen scenes are in \hyperref[tab:headline_seen]{Table~\ref*{tab:headline_seen}} (\hyperref[app:seen]{App.~\ref*{app:seen}}).}
\label{tab:headline}
\centering\scriptsize
\setlength{\tabcolsep}{4pt}
\renewcommand{\arraystretch}{1.1}
\resizebox{\textwidth}{!}{%
\begin{tabular}{@{}l cc ccc ccc@{}}
\toprule
& \multicolumn{2}{c}{Fidelity to GT} & \multicolumn{3}{c}{Memory} & \multicolumn{3}{c}{VBench} \\
\cmidrule(lr){2-3}\cmidrule(lr){4-6}\cmidrule(l){7-9}
Training & PSNR$\uparrow$ & LPIPS$\downarrow$ & Revisit$\uparrow$ & Gain$\uparrow$ & DINO$\uparrow$ & Subj.$\uparrow$ & Backg.$\uparrow$ & Imag.$\uparrow$ \\
\midrule
\rowcolor{grouprow}
\multicolumn{9}{@{}l}{\textit{Unseen scenes}} \\
$10$\,s, no retrieval & 10.60{\fontsize{5}{5}\selectfont$\,\pm$.06} & 0.646{\fontsize{5}{5}\selectfont$\,\pm$.005} & 0.311{\fontsize{5}{5}\selectfont$\,\pm$.018} & 0.049{\fontsize{5}{5}\selectfont$\,\pm$.020} & 0.557{\fontsize{5}{5}\selectfont$\,\pm$.020} & 0.751{\fontsize{5}{5}\selectfont$\,\pm$.006} & 0.873{\fontsize{5}{5}\selectfont$\,\pm$.001} & 0.666{\fontsize{5}{5}\selectfont$\,\pm$.014} \\
$10$\,s, $10$-chunk window & 10.47{\fontsize{5}{5}\selectfont$\,\pm$.32} & 0.662{\fontsize{5}{5}\selectfont$\,\pm$.006} & 0.363{\fontsize{5}{5}\selectfont$\,\pm$.033} & 0.114{\fontsize{5}{5}\selectfont$\,\pm$.033} & 0.572{\fontsize{5}{5}\selectfont$\,\pm$.020} & 0.726{\fontsize{5}{5}\selectfont$\,\pm$.008} & 0.863{\fontsize{5}{5}\selectfont$\,\pm$.003} & 0.643{\fontsize{5}{5}\selectfont$\,\pm$.007} \\
$10$\,s, top-$K$, no bank & 10.71{\fontsize{5}{5}\selectfont$\,\pm$.20} & 0.630{\fontsize{5}{5}\selectfont$\,\pm$.005} & 0.503{\fontsize{5}{5}\selectfont$\,\pm$.022} & 0.255{\fontsize{5}{5}\selectfont$\,\pm$.026} & 0.735{\fontsize{5}{5}\selectfont$\,\pm$.011} & 0.757{\fontsize{5}{5}\selectfont$\,\pm$.003} & 0.873{\fontsize{5}{5}\selectfont$\,\pm$.003} & \second{0.675}{\fontsize{5}{5}\selectfont$\,\pm$.016} \\
$100$\,s, per-segment retrieval, bank & 11.27{\fontsize{5}{5}\selectfont$\,\pm$.22} & 0.636{\fontsize{5}{5}\selectfont$\,\pm$.013} & 0.511{\fontsize{5}{5}\selectfont$\,\pm$.040} & 0.259{\fontsize{5}{5}\selectfont$\,\pm$.023} & 0.743{\fontsize{5}{5}\selectfont$\,\pm$.032} & 0.763{\fontsize{5}{5}\selectfont$\,\pm$.003} & 0.873{\fontsize{5}{5}\selectfont$\,\pm$.002} & 0.671{\fontsize{5}{5}\selectfont$\,\pm$.022} \\
$100$\,s, top-$K$, ordered bank & \best{11.78}{\fontsize{5}{5}\selectfont$\,\pm$.15} & \second{0.618}{\fontsize{5}{5}\selectfont$\,\pm$.010} & \second{0.600}{\fontsize{5}{5}\selectfont$\,\pm$.022} & 0.299{\fontsize{5}{5}\selectfont$\,\pm$.011} & 0.809{\fontsize{5}{5}\selectfont$\,\pm$.016} & \second{0.775}{\fontsize{5}{5}\selectfont$\,\pm$.008} & \second{0.879}{\fontsize{5}{5}\selectfont$\,\pm$.004} & 0.630{\fontsize{5}{5}\selectfont$\,\pm$.010} \\
$100$\,s, top-$K$, bank & \second{11.62}{\fontsize{5}{5}\selectfont$\,\pm$.21} & 0.643{\fontsize{5}{5}\selectfont$\,\pm$.005} & 0.559{\fontsize{5}{5}\selectfont$\,\pm$.037} & \second{0.300}{\fontsize{5}{5}\selectfont$\,\pm$.020} & 0.822{\fontsize{5}{5}\selectfont$\,\pm$.022} & 0.772{\fontsize{5}{5}\selectfont$\,\pm$.006} & 0.878{\fontsize{5}{5}\selectfont$\,\pm$.002} & 0.668{\fontsize{5}{5}\selectfont$\,\pm$.013} \\
$200$\,s, top-$K$, bank & 11.55{\fontsize{5}{5}\selectfont$\,\pm$.30} & \best{0.616}{\fontsize{5}{5}\selectfont$\,\pm$.009} & \best{0.623}{\fontsize{5}{5}\selectfont$\,\pm$.026} & \best{0.343}{\fontsize{5}{5}\selectfont$\,\pm$.014} & \best{0.854}{\fontsize{5}{5}\selectfont$\,\pm$.003} & \best{0.776}{\fontsize{5}{5}\selectfont$\,\pm$.008} & \best{0.881}{\fontsize{5}{5}\selectfont$\,\pm$.003} & 0.627{\fontsize{5}{5}\selectfont$\,\pm$.013} \\
$400$\,s, top-$K$, bank & 11.36{\fontsize{5}{5}\selectfont$\,\pm$.19} & 0.642{\fontsize{5}{5}\selectfont$\,\pm$.008} & 0.558{\fontsize{5}{5}\selectfont$\,\pm$.012} & 0.293{\fontsize{5}{5}\selectfont$\,\pm$.019} & \second{0.830}{\fontsize{5}{5}\selectfont$\,\pm$.011} & \best{0.776}{\fontsize{5}{5}\selectfont$\,\pm$.005} & 0.875{\fontsize{5}{5}\selectfont$\,\pm$.003} & \best{0.690}{\fontsize{5}{5}\selectfont$\,\pm$.016} \\
\addlinespace[1pt]
\rowcolor{grouprow}
\multicolumn{9}{@{}l}{\textit{Web photographs}} \\
$10$\,s, no retrieval & -- & -- & 0.183{\fontsize{5}{5}\selectfont$\,\pm$.009} & 0.060{\fontsize{5}{5}\selectfont$\,\pm$.011} & 0.571{\fontsize{5}{5}\selectfont$\,\pm$.010} & 0.769{\fontsize{5}{5}\selectfont$\,\pm$.004} & 0.849{\fontsize{5}{5}\selectfont$\,\pm$.002} & \best{0.736}{\fontsize{5}{5}\selectfont$\,\pm$.003} \\
$10$\,s, $10$-chunk window & -- & -- & 0.134{\fontsize{5}{5}\selectfont$\,\pm$.025} & 0.032{\fontsize{5}{5}\selectfont$\,\pm$.023} & 0.540{\fontsize{5}{5}\selectfont$\,\pm$.018} & 0.742{\fontsize{5}{5}\selectfont$\,\pm$.006} & 0.839{\fontsize{5}{5}\selectfont$\,\pm$.006} & 0.702{\fontsize{5}{5}\selectfont$\,\pm$.008} \\
$10$\,s, top-$K$, no bank & -- & -- & 0.359{\fontsize{5}{5}\selectfont$\,\pm$.014} & 0.227{\fontsize{5}{5}\selectfont$\,\pm$.010} & 0.725{\fontsize{5}{5}\selectfont$\,\pm$.008} & 0.774{\fontsize{5}{5}\selectfont$\,\pm$.003} & 0.855{\fontsize{5}{5}\selectfont$\,\pm$.002} & 0.722{\fontsize{5}{5}\selectfont$\,\pm$.003} \\
$100$\,s, per-segment retrieval, bank & -- & -- & 0.302{\fontsize{5}{5}\selectfont$\,\pm$.026} & 0.193{\fontsize{5}{5}\selectfont$\,\pm$.030} & 0.687{\fontsize{5}{5}\selectfont$\,\pm$.025} & 0.767{\fontsize{5}{5}\selectfont$\,\pm$.009} & 0.847{\fontsize{5}{5}\selectfont$\,\pm$.002} & 0.715{\fontsize{5}{5}\selectfont$\,\pm$.017} \\
$100$\,s, top-$K$, ordered bank & -- & -- & \second{0.443}{\fontsize{5}{5}\selectfont$\,\pm$.014} & 0.307{\fontsize{5}{5}\selectfont$\,\pm$.025} & 0.779{\fontsize{5}{5}\selectfont$\,\pm$.012} & 0.767{\fontsize{5}{5}\selectfont$\,\pm$.002} & 0.854{\fontsize{5}{5}\selectfont$\,\pm$.003} & 0.675{\fontsize{5}{5}\selectfont$\,\pm$.012} \\
$100$\,s, top-$K$, bank & -- & -- & 0.434{\fontsize{5}{5}\selectfont$\,\pm$.013} & \second{0.318}{\fontsize{5}{5}\selectfont$\,\pm$.016} & \second{0.811}{\fontsize{5}{5}\selectfont$\,\pm$.017} & 0.778{\fontsize{5}{5}\selectfont$\,\pm$.007} & \second{0.857}{\fontsize{5}{5}\selectfont$\,\pm$.004} & 0.705{\fontsize{5}{5}\selectfont$\,\pm$.012} \\
$200$\,s, top-$K$, bank & -- & -- & \best{0.467}{\fontsize{5}{5}\selectfont$\,\pm$.010} & \best{0.359}{\fontsize{5}{5}\selectfont$\,\pm$.016} & \best{0.831}{\fontsize{5}{5}\selectfont$\,\pm$.006} & \second{0.782}{\fontsize{5}{5}\selectfont$\,\pm$.002} & \best{0.863}{\fontsize{5}{5}\selectfont$\,\pm$.003} & 0.654{\fontsize{5}{5}\selectfont$\,\pm$.013} \\
$400$\,s, top-$K$, bank & -- & -- & 0.396{\fontsize{5}{5}\selectfont$\,\pm$.024} & 0.285{\fontsize{5}{5}\selectfont$\,\pm$.015} & 0.806{\fontsize{5}{5}\selectfont$\,\pm$.017} & \best{0.783}{\fontsize{5}{5}\selectfont$\,\pm$.006} & \second{0.857}{\fontsize{5}{5}\selectfont$\,\pm$.002} & \second{0.733}{\fontsize{5}{5}\selectfont$\,\pm$.009} \\
\bottomrule
\end{tabular}}
\end{table}

%% file: figures/fig_ablation.tex
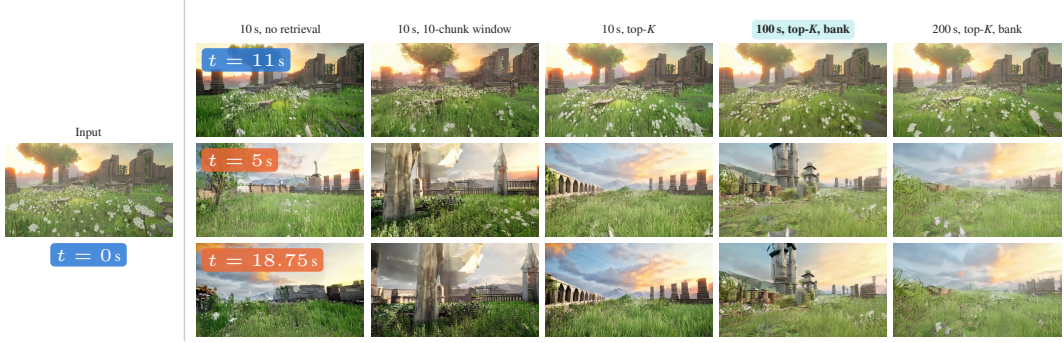
\begin{figure}[t]
\centering
\resizebox{\textwidth}{!}{%
\begin{tikzpicture}[
  img/.style={inner sep=0pt,outer sep=0pt},
  hd/.style={font=\fontsize{3.4}{4.0}\selectfont,anchor=south,inner sep=1.2pt,text=black!85,align=center},
  tag/.style={font=\fontsize{4.4}{5.2}\selectfont,text=white,fill opacity=0.85,text opacity=1,rounded corners=1.2pt,inner sep=1.6pt,anchor=north west},
]
\node[img,anchor=north west] (in) at (0,-0.970) {\includegraphics[width=1.62cm]{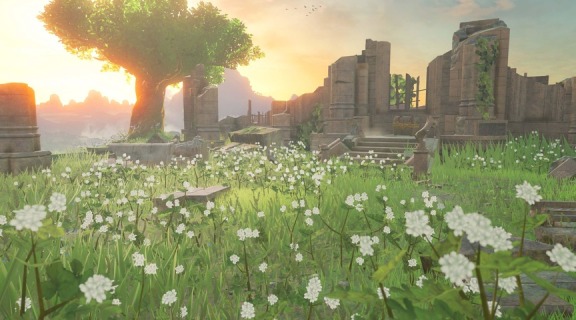}};
\node[hd] at (in.north) {Input};
\node[tag,fill=selMix,anchor=north] at ([yshift=-0.06cm]in.south) {$t = 0$\,s};
\draw[black!20,line width=0.4pt] (1.740,0.42) -- (1.740,-2.890);
\node[img,anchor=north west] (a0sw) at (1.860,0.000) {\includegraphics[width=1.62cm]{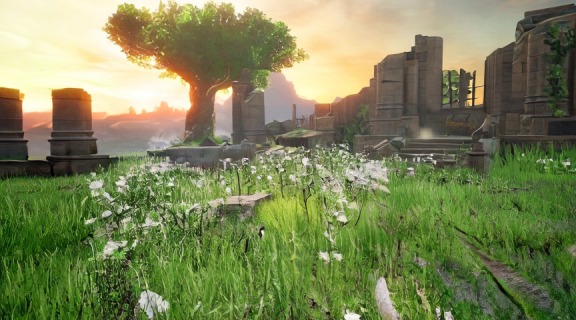}};
\node[tag,fill=selMix] at ([shift={(0.04,-0.04)}]a0sw.north west) {$t = 11$\,s};
\node[hd] at ([yshift=0.03cm]a0sw.north) {10\,s, no retrieval};
\node[img,anchor=north west] (a0sw10) at (3.550,0.000) {\includegraphics[width=1.62cm]{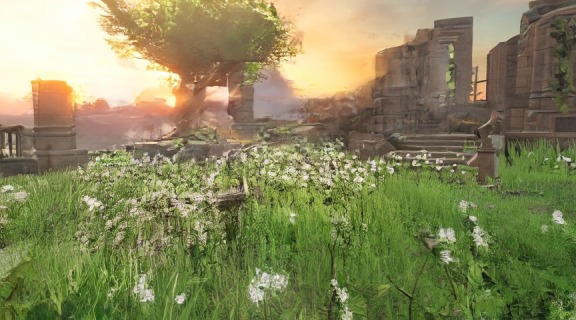}};
\node[hd] at ([yshift=0.03cm]a0sw10.north) {10\,s, 10-chunk window};
\node[img,anchor=north west] (a0topk) at (5.240,0.000) {\includegraphics[width=1.62cm]{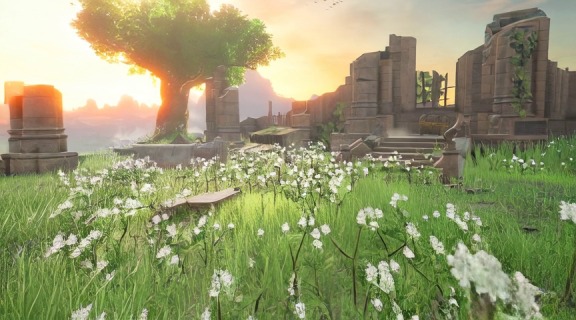}};
\node[hd] at ([yshift=0.03cm]a0topk.north) {10\,s, top-\textit{K}};
\node[img,anchor=north west] (a0ours) at (6.930,0.000) {\includegraphics[width=1.62cm]{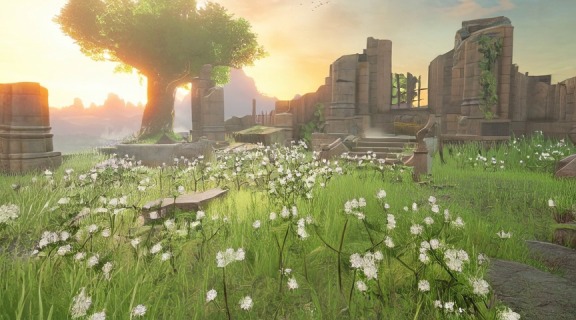}};
\node[hd,font=\fontsize{3.4}{4.0}\selectfont\bfseries,fill=ourrow,rounded corners=1.2pt] at ([yshift=0.03cm]a0ours.north) {100\,s, top-\textit{K}, bank};
\node[img,anchor=north west] (a0m199) at (8.620,0.000) {\includegraphics[width=1.62cm]{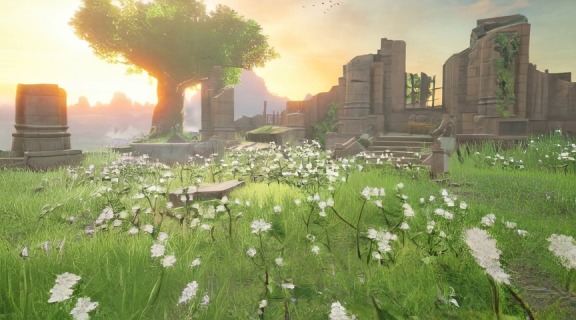}};
\node[hd] at ([yshift=0.03cm]a0m199.north) {200\,s, top-\textit{K}, bank};
\node[img,anchor=north west] (a1sw) at (1.860,-0.970) {\includegraphics[width=1.62cm]{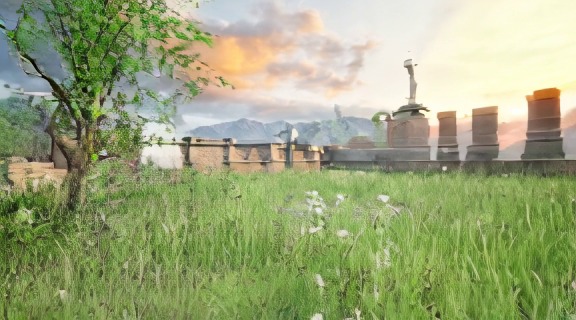}};
\node[tag,fill=selOv] at ([shift={(0.04,-0.04)}]a1sw.north west) {$t = 5$\,s};
\node[img,anchor=north west] (a1sw10) at (3.550,-0.970) {\includegraphics[width=1.62cm]{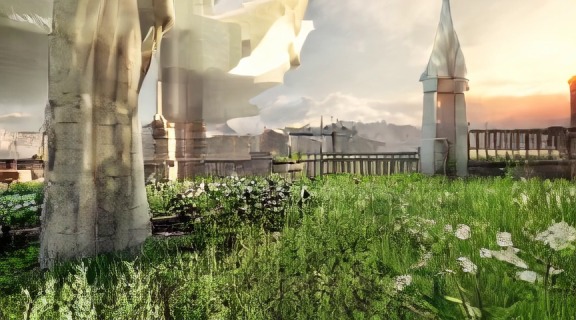}};
\node[img,anchor=north west] (a1topk) at (5.240,-0.970) {\includegraphics[width=1.62cm]{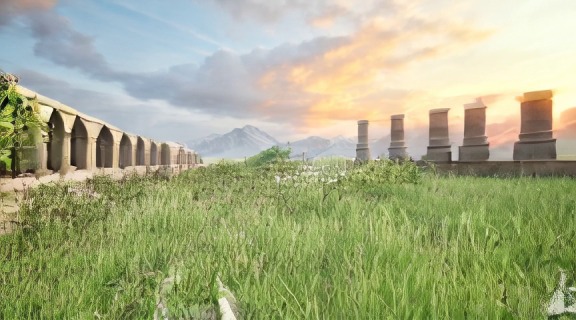}};
\node[img,anchor=north west] (a1ours) at (6.930,-0.970) {\includegraphics[width=1.62cm]{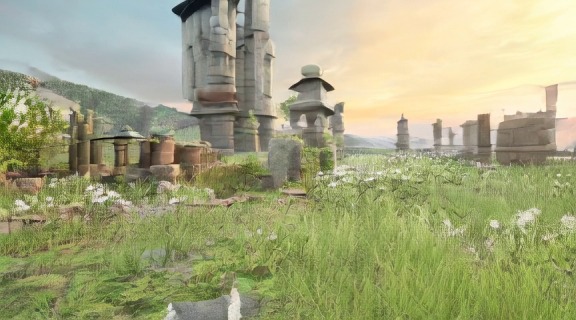}};
\node[img,anchor=north west] (a1m199) at (8.620,-0.970) {\includegraphics[width=1.62cm]{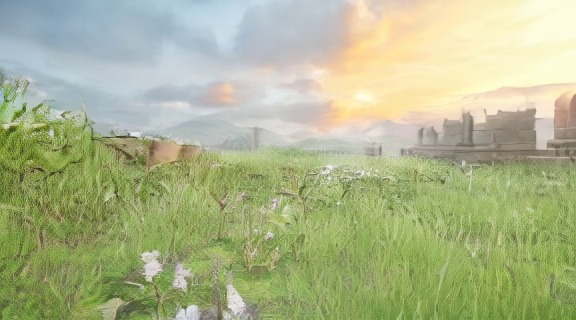}};
\node[img,anchor=north west] (a2sw) at (1.860,-1.940) {\includegraphics[width=1.62cm]{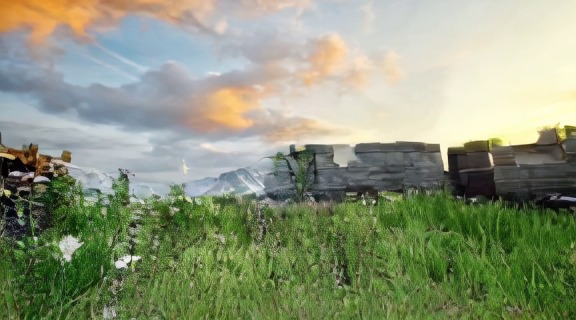}};
\node[tag,fill=selOv] at ([shift={(0.04,-0.04)}]a2sw.north west) {$t = 18.75$\,s};
\node[img,anchor=north west] (a2sw10) at (3.550,-1.940) {\includegraphics[width=1.62cm]{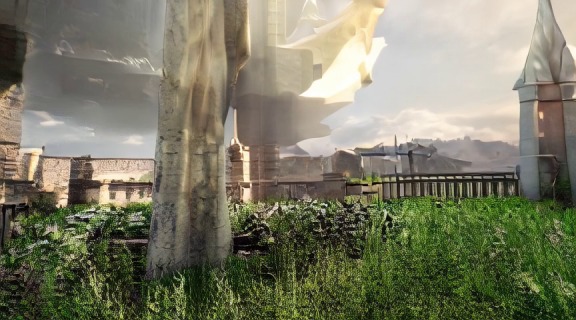}};
\node[img,anchor=north west] (a2topk) at (5.240,-1.940) {\includegraphics[width=1.62cm]{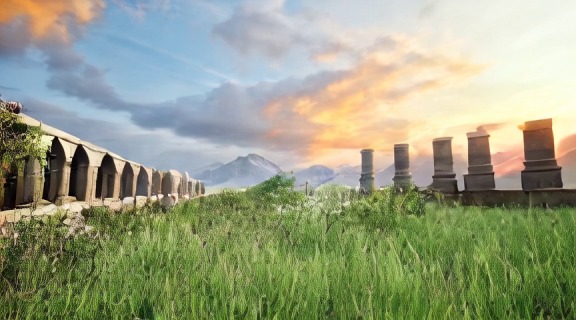}};
\node[img,anchor=north west] (a2ours) at (6.930,-1.940) {\includegraphics[width=1.62cm]{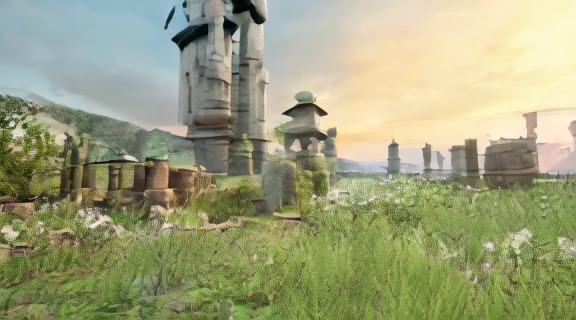}};
\node[img,anchor=north west] (a2m199) at (8.620,-1.940) {\includegraphics[width=1.62cm]{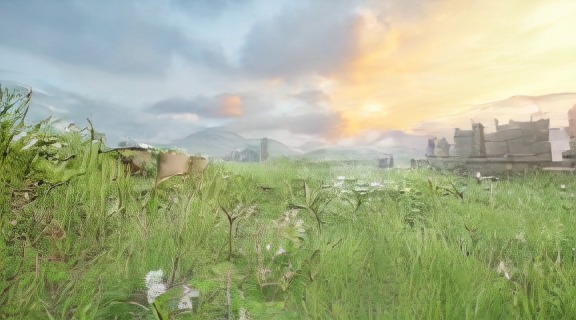}};
\end{tikzpicture}}
\caption{\textbf{Ablation study} on five models of
\hyperref[tab:headline]{Table~\ref*{tab:headline}} on a web photograph.
Top: back near the input view. Middle: a place a little further along. Bottom:
the return to it.}
\label{fig:ablation}
\end{figure}

%% file: figures/fig_interval.tex
\begin{figure}[t]
\centering
\includegraphics[width=\textwidth]{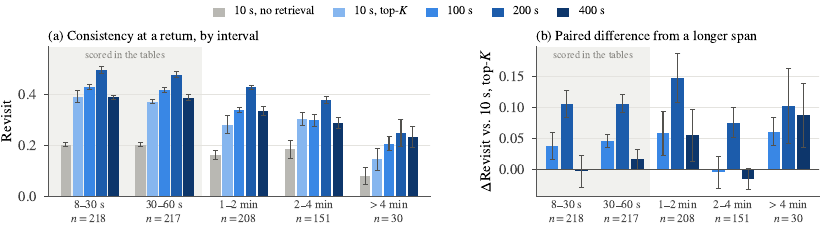}
\caption{\textbf{Revisit against return interval} over $400$\,s rollouts on the web
photographs; error bars are the standard deviation over five rollout seeds.
\textbf{(a)} Revisit per model, grouped by the interval between the two visits; $n$ counts
returns per seed. \textbf{(b)} Difference from $10$\,s top-$K$ on the same returns. Shading marks the first $60$\,s, which the tables score.}
\label{fig:interval}
\end{figure}

%% file: sections/05_discussion.tex
\section{Discussion}
\label{sec:disc}

\ours decouples the span a world model is trained on from the sequence it attends over: each scored chunk retrieves its own top-$K$ latents from the history by camera co-visibility, and the union of these requests forms a bank. Lengthening the span from $100$ to $400$\,s adds only $12\%$ to the step time. Retrieval raises revisit consistency over a sliding window on every split, and a span that reaches the first visits raises it further, here up to $200$\,s and not beyond, at a cost in image quality that we trace to conditioning on generated frames. Against open world models, \ours leads five of the eight memory columns of \hyperref[tab:sota]{Table~\ref*{tab:sota}} and CaR the other three, with \ours ahead on the returns the first frame alone cannot serve.

\paragraph{Limitations.} We mainly focus on static scenes: nothing moves but the camera, so a return is always to a place that should look the same, and Revisit asks only whether a model draws it the same way again. This leaves out scenes in which objects move or the place itself changes between visits, where drawing a place differently can be correct and consistency has to be judged against what should have changed. The cost in image quality comes from a mismatch between training, where the bank and the recent chunks are clean ground truth, and inference, where they hold the model's own output (\hyperref[app:selfcond]{App.~\ref*{app:selfcond}}). This is the exposure bias that Self Forcing \citep{selfforcing} removes by training on the model's own rollouts, and our distillation (\hyperref[sec:distill]{Sec.~\ref*{sec:distill}}) applies it to the bank as well as to the recent window; how much of the gap it closes remains to be measured.

%% file: appendix/00_overview.tex
\section*{Appendix}
\addcontentsline{toc}{section}{Appendix}

The appendix is organized in three parts. \hyperref[app:impl]{App.~\ref*{app:impl}} gives pseudocode for training and inference, shows that the shared bank is lossless and bounded, gives the overlap formula, grades the retrieval cues and chooses the retrieval width and bank capacity against ground truth with no trained model, and lists all settings. \hyperref[app:data]{App.~\ref*{app:data}} describes the training and test data and counts how far apart revisits lie. \hyperref[app:analysis]{App.~\ref*{app:analysis}} defines the revisit metrics, gives the seen-scene ablation, explains where the memory of a model trained on $10$\,s clips comes from and splits the gain by return interval, measures the cost of a training step, explains the baselines left out, checks how closely each system follows the camera path, tests what the bank must hold at inference, traces why image quality falls with span, and shows more returns from \ours and the open world models.

%% file: appendix/01_method.tex
\section{Method Details}
\label{app:impl}

\subsection{Training and Inference Procedures}
\label{app:algo}

\hyperref[alg:train]{Algorithm~\ref*{alg:train}} builds and scores one training sample, and \hyperref[alg:rollout]{Algorithm~\ref*{alg:rollout}} generates a video chunk by chunk. $\mathrm{TopK}\big(q, \mathcal{C}\big)$ returns the $\min(K, |\mathcal{C}|)$ latents of the candidate set $\mathcal{C}$ with the highest score $\Sco(q, \cdot)$ of \hyperref[eq:score]{Eq.~\ref*{eq:score}}, ties broken in favour of the candidate nearer $q$ in time, where $q$ is the camera of the last latent of the requesting chunk.

\paragraph{Inference.}
\label{sec:inference}
A rollout generates one chunk at a time, and the chunk being generated occupies the position of $T_1$: the model reads the first frame, a bank of $K$ entries, the $r$ most recent generated chunks and the chunk being denoised, $1 + K + rc + c$ slots in all. Retrieval follows the training rule with a single requesting chunk, so the bank is that chunk's own top-$K$ by \hyperref[eq:score]{Eq.~\ref*{eq:score}} over every latent generated before its window, recomputed at each step from the cameras seen so far. Nothing in the rollout depends on $\spanm$, which only sets how much history a training sample carries, so a rollout can run far beyond the longest training span: every generated chunk is attended over the same $1 + K + rc + c$ slots, whatever has been generated before it. The ranking is the one part that does grow. \hyperref[eq:score]{Eq.~\ref*{eq:score}} is evaluated over every latent generated so far, so selecting the bank for chunk $n$ costs $O(n)$ score evaluations against attention that costs nothing extra; it is a comparison of poses outside the network, and we leave it unindexed.

\begin{algorithm}[!ht]
\caption{One training step}
\label{alg:train}
\small
\begin{algorithmic}[1]
\Require episode latents $x_{0:N}$ with cameras $P_{0:N}$; constants $c, k, r, K$;
span range $[\spanm_{\min}, \spanm_{\max}]$
\State $\spanm \sim U[\spanm_{\min}, \spanm_{\max}]$; pick a start $s$ so the span of $1 + \spanm c$ latents fits
\State $\anchor \gets x_s$; $|\rec| \gets \min(r, \spanm - k)$, $|\mathcal{H}| \gets \max(0, \spanm - k - r)$ chunks \Comment{empty when $\spanm = k$}
\State split the next $\spanm$ chunks into history $\mathcal{H}$, recent $\rec$ and scored $T_1, \dots, T_k$
\For{$j = 1, \dots, k$}
  \State $S_j^{\mathcal{H}} \gets \mathrm{TopK}\big(q_j, \mathcal{H}\big)$ \Comment{$T_j$'s top-$K$ over the history; empty if $\mathcal{H} = \emptyset$}
\EndFor
\State $\band \gets \bigcup_j S_j^{\mathcal{H}}$ \Comment{$|\band| \le kK$}
\State sequence $z \gets [\,\anchor \mid \band \mid \rec \mid T_1 \dots T_k\,]$; temporal indices $0 \mid 1 \dots 1 \mid 2, 3, \dots$
\State express all cameras relative to the first latent of $T_1$
\State mask $M$: $\anchor, \band, \rec$ attend within the context and never into $\tgt$; each $T_j$ attends to
\Statex \hspace{\algorithmicindent} $\anchor$, itself, the $r$ chunks before it, and $\mathrm{TopK}\big(q_j, \mathcal{C}_j\big)$ over the candidates
\Statex \hspace{\algorithmicindent} $\mathcal{C}_j = \band \cup \{\text{chunks of } \rec \text{ outside its window}\} \cup \{T_1 \dots T_{j-r-1}\}$ completed before its window
\State draw a noise level $\sigma_j$ per scored chunk and replace $T_j$ in $z$ by its noised $\tilde T_j$ \Comment{diffusion forcing}
\State $\mathcal{L} \gets \sum_j \big\lVert f_\theta(z, M)_{T_j} - v_j \big\rVert^2$, with $v_j$ the flow-matching target of $T_j$ \Comment{loss on $\tgt$ only}
\State update $\theta$ with $\nabla_\theta \mathcal{L}$
\end{algorithmic}
\end{algorithm}

\begin{algorithm}[!ht]
\caption{Rollout}
\label{alg:rollout}
\small
\begin{algorithmic}[1]
\Require first-frame latent $x_0$; cameras $P_{0:nc}$ for $n$ chunks; constants $c, r, K$; denoising steps $S$
\State $\mathcal{G} \gets [\,x_0\,]$ \Comment{generated latents, in temporal order}
\For{$i = 1, \dots, n$}
  \State $\rec \gets$ the last $r$ generated chunks (fewer at the start)
  \State $\band \gets \mathrm{TopK}\big(q_i, \mathcal{G} \setminus (\{x_0\} \cup \rec)\big)$ \Comment{recomputed every chunk; $O(|\mathcal{G}|)$ scores}
  \State context $[\,x_0 \mid \band \mid \rec\,]$ at the positions of training; cache its keys and values
  \State $y \sim \mathcal{N}(0, I)$ for the new chunk, placed where training places $T_1$
  \For{$t = 1, \dots, S$}
    \State $y \gets$ one denoising step of $f_\theta$ on $y$, attending to the cached context
  \EndFor
  \State $\mathcal{G} \gets [\,\mathcal{G};\, y\,]$ \Comment{append the $c$ latents of $y$}
\EndFor
\State \Return decoded $\mathcal{G}$
\end{algorithmic}
\end{algorithm}

\subsection{The Shared Bank Is Lossless and Bounded}
\label{app:bounds}

Let $S_j^{\mathcal{H}} = \mathrm{TopK}\big(q_j, \mathcal{H}\big)$ be the top-$K$ of scored chunk $T_j$ over the history, so that $\band = \bigcup_{j=1}^{k} S_j^{\mathcal{H}}$ (\hyperref[alg:train]{Algorithm~\ref*{alg:train}}), and let $\mathcal{W}_j$ be the candidates already in the sequence: the recent chunks outside the window of $T_j$ and $T_1 \dots T_{j-r-1}$. The chunk reads $S_j = \mathrm{TopK}\big(q_j, \band \cup \mathcal{W}_j\big)$ (\hyperref[sec:masks]{Sec.~\ref*{sec:masks}}).

\begin{proposition}[The shared bank is lossless]
\label{prop:lossless}
For every $j$, $S_j = \mathrm{TopK}\big(q_j, \mathcal{H} \cup \mathcal{W}_j\big)$: reading from the shared bank selects exactly what $T_j$ would select from the whole history.
\end{proposition}

\begin{proof}
$\mathrm{TopK}$ ranks by score and breaks ties by time, a strict order, so a candidate is in the top-$K$ of a set only if fewer than $K$ members of that set precede it. If $|\mathcal{H}| < K$ then $S_j^{\mathcal{H}} = \mathcal{H} \subseteq \band$ and there is nothing to show. Otherwise a latent of $\mathcal{H} \setminus \band$ lies outside $S_j^{\mathcal{H}}$, so the $K$ members of $S_j^{\mathcal{H}} \subseteq \band$ all precede it and it is not in $\mathrm{TopK}\big(q_j, \mathcal{H} \cup \mathcal{W}_j\big)$. That set therefore lies in $\band \cup \mathcal{W}_j \subseteq \mathcal{H} \cup \mathcal{W}_j$, and the top-$K$ of a subset that contains the top-$K$ of the whole is the same set.
\end{proof}

The bank is also bounded. As a union of $k$ sets of at most $K$ latents, $|\band| \le kK$ for every history and span, so the sequence of \hyperref[eq:layout]{Eq.~\ref*{eq:layout}} holds $L \le 1 + kK + rc + kc$ slots and a step costs $O\big((1 + kK + rc + kc)^2\big)$ attention at any span; through the mask each $T_j$ attends to at most $1 + K + rc + c$ positions, whatever $|\band|$ and $\spanm$.

\label{app:layout}
With the constants of \hyperref[sec:setup]{Sec.~\ref*{sec:setup}}, packing the $k$ scored chunks of a sample into one sequence costs $1.82$ slots per scored latent at the measured bank size, against $(1 + K + rc + c)/c = 3.75$ when each chunk is run as its own sequence (\hyperref[tab:hparams]{Table~\ref*{tab:hparams}}; bank sizes from the training sampler on $20$ episodes, $30$ draws per span).

\subsection{Frustum Overlap}
\label{app:criterion}

Let $P_q, P_c \in \SE(3)$ be the camera-to-world poses of the query chunk and a candidate, and let $G$ be a grid of points in the query's camera frame, spread over the image and over a depth range $[z_{\mathrm{near}}, z_{\mathrm{far}}]$. Then
\begin{equation}
\Ovl(q,c) \;=\; \frac{1}{|G|}\sum_{g \in G}
  \mathds{1}\!\left[\,\pi\!\left(P_c^{-1} P_q\, g\right) \in \Omega
  \;\wedge\; z\!\left(P_c^{-1} P_q\, g\right) \in [z_{\mathrm{near}}, z_{\mathrm{far}}]\,\right],
\label{eq:overlap}
\end{equation}
where $\pi$ projects through the shared intrinsics and $\Omega$ is the image rectangle. $G$ is an $8\times8$ grid of cell centres over the image rectangle at $4$ depths, the centres of four log-uniform bins over $[z_{\mathrm{near}}, z_{\mathrm{far}}] = [0.5, 15]$ in the units of the corpus ($2$ to $60$\,m, one unit being $4$\,m), so the samples lie at $0.77$, $1.79$, $4.19$ and $9.81$ units; $256$ points in all. Bin centres keep every sample strictly inside the range, so a camera scores $\Ovl = 1$ against itself. The depth range is the one quantity in \hyperref[eq:overlap]{Eq.~\ref*{eq:overlap}} that carries a scene scale; in \hyperref[eq:score]{Eq.~\ref*{eq:score}} the translation term of $\Pdist$ carries one too. In $\Ovl$, a candidate is credited with a query point only if that point falls inside its frustum \emph{and} within this range of it, so $z_{\mathrm{far}}$ sets how far behind the camera's own position a candidate may be and still be said to have seen the place. \hyperref[app:scale]{App.~\ref*{app:scale}} measures what a mis-set scale costs.

\subsection{Retrieval Cues}
\label{app:selector}

Rules are graded as in \hyperref[sec:criterion]{Sec.~\ref*{sec:criterion}}, over $960$ scored chunks from $20$ episodes with each chunk's candidates thinned by an even stride to between $120$ and $240$ (fewer when the history is shorter), each rule's six picks against the six most similar candidates of that chunk.

\paragraph{Failure Modes.}
\label{app:blindspot}
Overlap cannot separate candidates along the viewing axis: moving the camera straight forward nests the new frustum inside the old one, and overlap still reads $1.000$ after $1.80$ units ($7.2$\,m) of travel.

\paragraph{Scale Sensitivity.}
\label{app:scale}
Under a global rescaling of a scene, the translation term of \hyperref[eq:posedist]{Eq.~\ref*{eq:posedist}} scales while the rotation term does not, whereas \hyperref[eq:overlap]{Eq.~\ref*{eq:overlap}} is invariant if the depth range scales with the scene. \hyperref[tab:scale]{Table~\ref*{tab:scale}} rescales translations while holding each cue's constants fixed. Over a hundredfold scale error, pose distance varies by $10.2$ points and overlap, with its depth range held fixed, by $13.5$. Every scene of our corpus shares metric units scaled by one factor, which removes most of this sensitivity. Shrinking translations tenfold helps pose distance, as does raising $w$, which weighs the same trade from the other side: every value from $4$ to $32$ lies within $1.7$ points, and the best, $w = 8$ to $16$, is $1.7$ points above the $w = 4$ we train with, chosen on an earlier version of the corpus.

\input{tables/scale}

\subsection{Retrieval Width and Bank Capacity}
\label{app:capacity}

\input{tables/capacity}

\hyperref[tab:capacity]{Table~\ref*{tab:capacity}} varies the number of frames a chunk retrieves and, separately, a hard cap on the bank, both at $\lambda = 0.2$ over $240$ query chunks from $12$ episodes. Capping the bank costs quality: moving from ten to sixty-four entries gains six points, and sixty-four comes within a point of an unbounded bank. Achievable gain falls slowly with width because overlap saturates at $1$ for any candidate containing the query view, leaving little to order beyond the first few picks. It falls from $60.0\%$ at $K = 2$ to $55.4\%$ at $K = 6$ and $45.9\%$ at $K = 24$. These percentages are on a different scale from \hyperref[fig:selector]{Figure~\ref*{fig:selector}}: each column is graded against its own ceiling, the $K$ most similar candidates, over $240$ query chunks from $12$ episodes, while \hyperref[fig:selector]{Figure~\ref*{fig:selector}} grades the six retrieved frames against the six best over $960$ chunks from $20$ episodes. The arms of each study are comparable with each other, not across the two, which is why $K = 6$ reads $55.4\%$ here and $57.2\%$ there. We use $K = 6$, which keeps most of that gain while giving each chunk several frames to read.
\label{app:kceiling}

\subsection{Settings}
\label{app:trainingdetail}

Each chunk draws its own noise level, and the flow-matching loss is taken on the scored chunks only. The first frame enters as a single-frame latent, as a photograph does at inference. \hyperref[tab:hparams]{Table~\ref*{tab:hparams}} lists all other settings.

\input{tables/hparams}

%% file: tables/scale.tex
\begin{table}[!htbp]
\caption{\textbf{Sensitivity to pose scale.} Achievable gain when translations are
rescaled with each cue's constants held fixed.}
\label{tab:scale}
\centering\scriptsize
\setlength{\tabcolsep}{6pt}
\renewcommand{\arraystretch}{1.0}
\begin{tabular}{@{}l cc@{}}
\toprule
Translations rescaled by & Pose distance $\Pdist$ ($w = 4$) & Frustum overlap $\Ovl$ \\
\midrule
$0.1\times$ & \best{$57.1\%$} & $39.5\%$ \\
$1\times$ (as posed) & \second{$55.9\%$} & \second{$47.4\%$} \\
$10\times$ & $46.9\%$ & \best{$53.0\%$} \\
\bottomrule
\end{tabular}
\end{table}

%% file: tables/capacity.tex
\begin{table}[!htbp]
\caption{\textbf{Retrieval width and bank capacity.} \emph{Top:} each width $K$ against
its own ceiling, the $K$ most similar candidates. \emph{Bottom:} six frames read from a
bank that admits frames by coverage, dropping near-duplicates, and is capped at the given size, together
with the latents already completed in the chunk's window, which supply the rest when the cap is below six;
$\infty$ lifts the cap but not the de-duplication, so it differs slightly from $K = 6$ above.}
\label{tab:capacity}
\centering\scriptsize
\setlength{\tabcolsep}{5pt}
\renewcommand{\arraystretch}{1.0}
\begin{tabular}{@{}l *{9}{c}@{}}
\toprule
Per-chunk retrieval $K$ & $1$ & $2$ & $3$ & $4$ & $6$ & $8$ & $12$ & $16$ & $24$ \\
\midrule
Achievable gain (\%)$\uparrow$ & \second{$59.6$} & \best{$60.0$} & $58.9$ & $57.4$ & $55.4$ & $54.2$ & $51.1$ & $49.1$ & $45.9$ \\
\bottomrule
\addlinespace[6pt]
\toprule
Bank capacity & $4$ & $6$ & $10$ & $16$ & $24$ & $40$ & $64$ & $\infty$ & \\
\midrule
Achievable gain (\%)$\uparrow$ & $46.5$ & $46.9$ & $49.1$ & $49.7$ & $50.4$ & $51.4$ & \second{$55.0$} & \best{$55.6$} & \\
\bottomrule
\end{tabular}
\end{table}

%% file: tables/hparams.tex
\begin{table}[!htbp]
\caption{\textbf{Settings.} Everything the runs of \hyperref[tab:headline]{Table~\ref*{tab:headline}} share; the ablation rows differ only as \hyperref[sec:span]{Sec.~\ref*{sec:span}} describes.}
\label{tab:hparams}
\centering\scriptsize
\setlength{\tabcolsep}{6pt}
\renewcommand{\arraystretch}{1.0}
\begin{tabular}{@{}l l >{\raggedright\arraybackslash}p{5.6cm}@{}}
\toprule
Group & Setting & Value \\
\midrule
Model and data & Initialization & Wan2.2-TI2V-5B \citep{wan} \\
 & Camera conditioning & PRoPE \citep{prope} \\
 & Video & $864\times480$ at $16$\,fps \\
\midrule
Training sample & Latents per chunk $c$ & $4$ \\
 & Scored chunks $k$, recent chunks $r$ & $10$, $1$ \\
 & Span $\spanm$ (chunks) & $U[10, \spanm_{\max}]$, $\spanm_{\max} = 100$ for \ours{} ($41$--$401$ latents) \\
 & First frame & single-frame latent \\
\midrule
Retrieval & Entries per chunk $K$ & $6$ \\
 & Score weights $\lambda$, $w$ (\hyperref[eq:score]{Eq.~\ref*{eq:score}}) & $0.2$, $4$ \\
 & Overlap grid, depth range & $8\times8\times4$ points, $[0.5, 15]$ units \\
\midrule
Sequence (derived) & Bank $|\band|$, bound / measured at $100$\,s & $kK = 60$ / ${\approx}\,28$ \\
 & Slots $L$, bound / measured at $100$\,s & $105$ / ${\approx}\,73$ (run median $65$, \hyperref[tab:cost]{Table~\ref*{tab:cost}}) \\
 & Slots each scored chunk reads & $1 + K + (r{+}1)c = 15$ \\
 & Slots per scored latent, packed / separate & $1.82$ / $3.75$ \\
\midrule
Optimization & Optimizer, schedule & AdamW, constant, no warmup \\
 & Learning rate (backbone, camera) & $2\times10^{-5}$, $10^{-4}$ \\
 & Gradient clipping & $1.0$ \\
 & Steps, batch size & $6{,}000$, $32$ ($32$ H200, bfloat16, FSDP) \\
 & Caption dropout & $10\%$ \\
\midrule
Evaluation & Denoising steps, guidance scale & $20$, $4$ \\
 & Negative prompt & Wan2.2 default (English) \\
 & Rollout seeds, scored window & $5$, first $60$\,s \\
 & Return & within $0.5$ units and $15^\circ$, $\ge 8$\,s apart, away $> 1.5$ units or $90^\circ$ between \\
 & Control pair & same time gap, cameras $> 1$ unit or $30^\circ$ apart \\
\bottomrule
\end{tabular}
\end{table}

%% file: appendix/02_data.tex
\section{Dataset}
\label{app:data}

\subsection{Training Data}
\label{app:pipeline}

\paragraph{Source.} The corpus is rendered in Unreal Engine~5.8 through UnrealCV along scripted coverage walks, with frozen trajectories, no physics, fixed exposure, and no motion blur, depth of field or on-screen overlays. A \emph{scene} is one engine map, and a \emph{walk} is one continuous capture of up to $6.7$\,h at $1280\times720$ and $24$\,fps through a pinhole camera with a $90^\circ$ horizontal field of view, at an eye height of $1.7$\,m. Each walk is cut into non-overlapping $15$\,min segments, the \emph{episodes} used throughout, and a shorter remainder is dropped.

\paragraph{Camera Poses.} Camera poses are the engine's ground-truth camera-to-world transforms, read per frame in metres. Each trajectory is recentered on its first frame and divided by a fixed factor of $4$, so one unit is $4$\,m in every scene. A latent takes the pose of the central frame of its four. When the camera translates, which it does on $47\%$ of latents, it moves a constant $0.0625$ units ($0.25$\,m, so $1$\,m/s) per latent; it turns by at most $7.5^\circ$ per latent, pitches by up to about $29^\circ$ and never rolls.

\paragraph{VAE Encoding.} Video is resampled from $24$ to $16$\,fps by nearest frame, with the same map for images and poses, and resized to $864\times480$ with Lanczos filtering and no crop, giving intrinsics $f_x = 432$, $f_y = 426.7$, $c_x = 432$ and $c_y = 240$. The Wan2.2-TI2V-5B VAE \citep{wan} encodes it with $16\times$ spatial and $4\times$ temporal compression into $48$ channels, so a $15$\,min episode holds $3{,}600$ latents. Every fourth latent also carries a caption of its frame, written by Qwen3-VL-8B and embedded with umT5, the text encoder of Wan2.2, and a single-frame encoding of that frame, so that the first frame of a training sample enters exactly as a photograph does at inference. No episode is filtered by camera motion, and the $41$-latent floor of the sampler removes none.

\paragraph{Split.} Training uses $380$ episodes from $51$ walks over $50$ scenes, $95.0$\,h in total, frozen as a list because the corpus keeps growing. The last episode of every training scene that has more than one is held out for the seen split below ($44$ episodes, $11.0$\,h; six scenes have a single episode and keep it), and the four scenes with the shortest captures are held out entirely for the unseen split.

\subsection{Revisit Statistics}
\label{app:revisitstats}
\label{app:windowcount}

\paragraph{Definition.} Two latents close in pose and far apart in time are not by themselves a revisit, since a camera that pauses satisfies both conditions without going anywhere. A latent $i$ is a return to an earlier latent $j$ when the camera centres are close, $\lVert t_i - t_j \rVert \le \tau$, and the camera left in between, $\max_{j < l < i} \lVert t_l - t_j \rVert > \varepsilon$, with both thresholds in the units of the corpus. For every returning latent we record its shortest interval, $(i - j^\star)/4$ seconds with $j^\star$ the latest qualifying $j$, since a training window must contain both ends to supervise the return. We count every latent of the $380$ training episodes, so no interval exceeds $15$\,min. One unit is $4$\,m, so $\tau = 0.5$ is $2$\,m and $\varepsilon = 1.5$ is $6$\,m.

\paragraph{Why This Differs From the Evaluation Metric.} \hyperref[tab:returndefs]{Table~\ref*{tab:returndefs}} sets the definitions side by side: the corpus count leaves out the heading because it asks whether a training window could hold both visits, not whether two frames show the same view. Adding the heading can only drop returning latents and move each partner earlier, so the intervals below are a lower bound on the supervision gap; at $\varepsilon = 3.0$ the shortest one moves from $13.25$ to $28.25$\,s.

\input{tables/returndefs}

\begin{figure}[!htbp]
\centering
\includegraphics[width=0.8\textwidth]{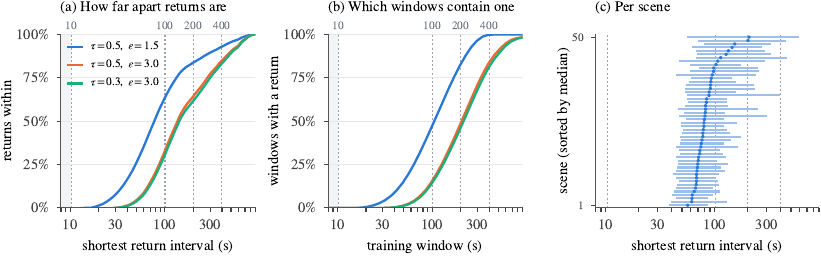}
\caption{\textbf{Revisits in the training corpus.} \textbf{(a)} Shortest return interval for
three definitions. \textbf{(b)} Share of training windows of each length that contain a
return. \textbf{(c)} Per scene, sorted by median, at $\tau = 0.5$, $\varepsilon = 1.5$.
Dashed lines mark spans of $10$, $100$, $200$ and $400$\,s.}
\label{fig:revisitstats}
\end{figure}

\paragraph{Findings.} At $\tau = 0.5$ and $\varepsilon = 1.5$ the corpus holds $528{,}231$ returning latents. None returns within $13.25$\,s, so no $10$\,s training window contains a return, and the median return takes $77.5$\,s (\hyperref[fig:revisitstats]{Figure~\ref*{fig:revisitstats}}a). Stricter definitions lengthen both: the shortest interval becomes $28.25$\,s and the median $136$\,s at $\varepsilon = 3.0$. A window must therefore reach well past ten seconds to contain a return: $3.8\%$ of $30$\,s windows and $22\%$ of $60$\,s windows contain one, against $47\%$ of $100$\,s windows and nearly every $400$\,s window, and $64\%$ of returns fit within $100$\,s, $84\%$ within $200$\,s and $93\%$ within $400$\,s (\hyperref[fig:revisitstats]{Figure~\ref*{fig:revisitstats}}b). The pattern holds across scenes, whose median intervals range from $56$ to $206$\,s (\hyperref[fig:revisitstats]{Figure~\ref*{fig:revisitstats}}c).

%% file: tables/returndefs.tex
\begin{table}[!htbp]
\caption{\textbf{Three definitions of a return.} Distances in the units of the corpus (one unit is $4$\,m).}
\label{tab:returndefs}
\centering\scriptsize
\setlength{\tabcolsep}{5pt}
\renewcommand{\arraystretch}{1.0}
\begin{tabular}{@{}>{\raggedright\arraybackslash}p{2.9cm} l l >{\raggedright\arraybackslash}p{3.6cm} >{\raggedright\arraybackslash}p{2.9cm}@{}}
\toprule
Definition & Position & Heading & Also required & Used for \\
\midrule
Corpus return (\hyperref[app:revisitstats]{App.~\ref*{app:revisitstats}}) & $\le \tau = 0.5$ & --- & camera away $> \varepsilon = 1.5$ in between & could one window hold both visits \\
Evaluation return (\hyperref[app:eval]{App.~\ref*{app:eval}}) & $\le 0.5$ & $\le 15^\circ$ & $\ge 8$\,s apart; away $> 1.5$ or $90^\circ$ & Revisit and Gain \\
Fold-back (\hyperref[app:eval]{App.~\ref*{app:eval}}) & \multicolumn{3}{>{\raggedright\arraybackslash}p{5.4cm}}{an evaluation return whose path back retraces the path out within $0.5$} & camera-following check \\
\bottomrule
\end{tabular}
\end{table}

%% file: appendix/03_experiments.tex
\section{Additional Experiments}
\label{app:analysis}

We first give the remaining details of the revisit metrics of \hyperref[sec:exp]{Sec.~\ref*{sec:exp}}. The analyses that follow report the seen-scene ablation, where the memory of a model trained on $10$\,s clips comes from and how the gain splits by return interval, the cost of a training step, the baselines left out of \hyperref[tab:sota]{Table~\ref*{tab:sota}}, how closely each system follows the camera path, what the bank must hold at inference and why image quality falls with span; the last section is a qualitative comparison.

\subsection{Evaluation Metrics}
\label{app:eval}

\hyperref[sec:setup]{Sec.~\ref*{sec:setup}} defines a return, Revisit and Gain; this section gives the
rest of the rule. A pair $(i, j)$ counts as a return only if the camera has left in between by more than
$1.5$ units or $90^\circ$, and consecutive pairs from one pass are thinned to a single return, so a slow
walk past a place contributes once rather than once per latent. The control pairs matched to a return at
its own time gap are those whose cameras differ by more than $1$ unit or $30^\circ$.

The correlation is taken over RGB values: each frame, resized to $864\times480$ where a system renders at
another size, is flattened with its three channels into one vector, and the two vectors are centered on
their own means. This is a deliberately literal measure. It rewards a model for putting the same
intensities back in the same places and is therefore sensitive to spatial misalignment, so a model that
draws the right scene from a slightly wrong pose is penalized alongside one that draws the wrong scene.
\hyperref[app:camfollow]{App.~\ref*{app:camfollow}} reports how closely each system follows the prescribed
path, which is what separates those two failures.

The DINO column of \hyperref[tab:headline]{Table~\ref*{tab:headline}} scores the same returns in
feature space. Each frame is resized to $224\times224$, normalized with the ImageNet statistics and
passed through DINO ViT-B/16 \citep{dino}, the weights VBench uses for subject consistency; the feature
is the CLS token of the final layer, L2-normalized, and the similarity of two frames is the dot product
of their features. It is averaged over the returns of a video and then over videos, as for Revisit. The
control pairs are the same as for Gain, so a DINO gain is the return similarity minus the control
similarity; the table reports the return similarity. Because the CLS token pools over the whole frame,
this measure forgives a small shift in pose that the pixel correlation penalizes.

After a few seconds of open-loop rollout a generated scene is no longer pixel-aligned with the rendered
ground truth, so the PSNR and LPIPS we report against it measure how close a rollout stays to it rather
than reconstruction.

We call a return $(i, j)$ a \emph{fold-back} when the path back retraces the path out: with $\mu = (i + j)/2$, the positions over $[\mu, i]$, read backwards, follow those over $[j, \mu]$ within $0.5$ units in discrete Fr\'echet distance, the tolerance that already defines a return. Of the $173$ returns on the web walks, $125$ ($72\%$), in all $20$ clips, are fold-backs, and $151$ ($87\%$) leave the place by turning away by more than $90^\circ$ rather than by walking off; at $1$ unit the fold-backs are $98\%$.

\subsection{Ablation on Seen Scenes}
\label{app:seen}

\hyperref[tab:headline_seen]{Table~\ref*{tab:headline_seen}} gives the seen split of \hyperref[tab:headline]{Table~\ref*{tab:headline}}: clips cut from held-out episodes of the training scenes with new trajectory.

\input{tables/headline_seen}

\hyperref[tab:sota_std]{Table~\ref*{tab:sota_std}} repeats \hyperref[tab:sota]{Table~\ref*{tab:sota}} with the standard deviation of every entry over the five rollout seeds.

\input{tables/sota_std}

\subsection{Where the \texorpdfstring{10\,s}{10 s} Model's Memory Comes From}
\label{app:tenswhy}

The $10$\,s top-$K$ run of \hyperref[tab:headline]{Table~\ref*{tab:headline}} reaches most of the memory of
the $100$\,s run: its DINO gain, the DINO similarity at a return minus that of the control pairs, which \hyperref[tab:headline]{Table~\ref*{tab:headline}} does not list, is $0.144$, $0.161$ and $0.166$ on seen scenes, unseen scenes and web photographs, against $0.227$, $0.240$ and $0.265$ for the $100$\,s run ($63\%$ to $67\%$), and its pixel Gain in \hyperref[tab:headline]{Tables~\ref*{tab:headline}} and~\ref{tab:headline_seen} is $71\%$ to $85\%$ of the $100$\,s run's. Ten seconds is its training span, not its reach at
inference. Every run retrieves from the whole generated history, with no limit on how old a candidate may be,
so at inference the two runs draw on the same frames and differ only in how old the frames they were trained
to read could be: on the seen clips both retrieve frames about $14.5$\,s old on average and up to $59$\,s old. A model trained on $10$\,s clips can use older frames because nothing in its input tells
it a frame's age: retrieval is by pose, and every bank entry shares one temporal index
(\hyperref[sec:masks]{Sec.~\ref*{sec:masks}}), so what it learns on frames less than $10$\,s old transfers
to frames retrieved from much further back.

\subsection{Gain by Return Interval and Type}
\label{app:gapsplit}

\hyperref[tab:tenswhy]{Table~\ref*{tab:tenswhy}} splits the Gain by the interval between the two visits and
by the kind of return. Three things follow. The $10$\,s top-$K$ run has a real gain at returns $20$ to
$60$\,s apart, far outside its training window, so it does use frames retrieved from long ago. What
long-span training adds lies at returns more than $10$\,s apart: on web photographs the $100$\,s run leads
by about $0.09$ at both $10$--$20$\,s and $20$--$60$\,s, while at $8$--$10$\,s the $10$\,s run is slightly
ahead on every split. And returns to the starting pose carry a gain even without retrieval, because the
first frame is always attended, so they do not isolate long-range memory; the mid-path returns are the
cleaner test.

\input{tables/tenswhy}

\subsection{Training Cost}
\label{app:cost}

\hyperref[tab:cost]{Table~\ref*{tab:cost}} reports the cost of the eight runs of \hyperref[tab:headline]{Table~\ref*{tab:headline}}, each trained for $6{,}000$ steps on $32$ H200 GPUs with one sample per GPU. Time per step includes data loading, retrieval, mask construction and checkpointing, and is pooled over the training segments between the supervisor's restarts; peak memory is the largest allocation on one GPU.

\input{tables/cost}

Adding the bank raises the cost once: the sequence grows from $41$ slots to a median of $65$ to $72$, and time per step and peak memory rise by factors of $2.3$ and $1.8$ over the sliding window. Beyond that, the span matters little. Quadrupling the maximum span from $100$ to $400$\,s leaves the sequence at its bound of $105$ slots and peak memory unchanged at $28.45$\,GB, and adds $12\%$ to the time per step, spent on reading and ranking a longer history and on a median sequence seven slots longer. For comparison, attending to every earlier chunk of a $10$\,s window already costs $15\%$ more per step than the sliding window. A $400$\,s window trained directly would hold $1{,}601$ latents, $15\times$ the largest sequence here. Sharing one bank across the scored block instead of retrieving per chunk is the one arrangement that is cheaper than ours, at $51$ slots and $4.17$\,s a step, and it is the arm that loses most of the memory (\hyperref[sec:span]{Sec.~\ref*{sec:span}}).

\subsection{Baselines Not in Table~\ref{tab:sota}}
\label{app:omitted}

The two systems closest to ours are absent from \hyperref[tab:sota]{Table~\ref*{tab:sota}} for different reasons. WorldMem \citep{worldmem} releases its code and weights, but it is a Minecraft model driven by a discrete action vocabulary: it cannot be given a photograph and a continuous camera path, and a version retrained on our corpus would no longer be the system its numbers describe. Context-as-Memory \citep{contextasmemory} has released its dataset but neither its code nor its weights, so it cannot be run. What can be compared is its design: one row of \hyperref[tab:headline]{Table~\ref*{tab:headline}} retrieves once for the whole scored block instead of once per chunk, so every chunk reads the same frames, as in Context-as-Memory, with everything else as in our runs.

\subsection{Camera Following}
\label{app:camfollow}

A model that ignored the camera path could score well on Revisit by drawing the same view all the time, so we check how closely each system turns when the path turns. We avoid pose reconstruction, which would fail on some clips, and measure the image instead: for every quarter second we estimate the horizontal shift between consecutive frames by phase correlation, and compare it with the yaw the path prescribes. \emph{Correlation} is the Pearson correlation between the prescribed yaw change and the measured shift, and \emph{Direction} the fraction of steps whose image shifts the way the path turns; neither depends on a system's resolution, frame rate or field of view. The implied field of view reads image shift per degree of yaw as a pinhole focal length. The top row applies the same measurement to the rendered frames of the seen split along their engine poses, which move the way the web walks do.

\input{tables/camfollow}

Every system turns the right way almost always, so direction separates little. Our runs track the prescribed yaw with a mean correlation of $0.961$, $0.95$ to $0.97$ across the six runs, and realize the turns at about the scale the paths assume, with a mean implied field of view of $90.6^\circ$ against $89^\circ$ for rendered frames. The open world models follow the path less closely, with correlations of $0.70$ to $0.88$,
and read the prescribed yaw at their own scale: LingBot-World~2.0 turns as though its camera
were narrower than the paths assume, $73^\circ$, and DreamX-World~1.0 and HY-WorldPlay~1.5 as
though it were far wider, $133^\circ$ and $124^\circ$. HY-WorldPlay~1.5 follows least closely of
all, so part of its low Revisit may be path following rather than memory; the others turn
the right way on $95\%$ of steps or more. Systems can therefore differ in the scale at which
they realize the prescribed motion, which the last column records. That difference matters little for most returns, which retrace the path out (\hyperref[app:eval]{App.~\ref*{app:eval}}): a model that applies a roughly consistent gain to its turns and steps comes back close to the pose of its first visit, and a lower Revisit there reflects what it draws rather than where it stands.

\subsection{Bank Contents}
\label{app:bankuse}

An unused retrieval mechanism is the standing failure mode in this area, so we test use directly. Holding the checkpoint, trajectory and seed fixed, we vary \emph{only the contents of the bank's slots} at inference and keep everything else in the sequence as trained. Each arm answers a different question. If the bank were ignored, filling it from \emph{another episode} would change nothing; if it were read but unnecessary, \emph{emptying} it would cost nothing; and if what the criterion selects mattered only as extra frames, $K$ \emph{random} frames from the same history would do as well as the retrieved ones.

\input{tables/bankuse}

\hyperref[tab:bankuse]{Table~\ref*{tab:bankuse}} answers all three, on the returns that need the bank. Mid-path Revisit falls from $0.440$ with the retrieved bank to $0.207$ with random frames from the same history, $0.105$ with an empty bank and $0.073$ with frames from another episode, and Gain from $0.319$ to $0.109$, $0.019$ and $0.016$. A bank from the wrong episode leaves the model no better off than an empty one and far below the retrieved bank, so the model reads what the bank holds rather than treating it as filler. Retrieved frames also beat random ones from the same history by more than a factor of two on both columns, so what the criterion selects matters beyond the number of frames it supplies.

\input{figures/fig_sota_more}

\subsection{Why Image Quality Falls with Span}
\label{app:selfcond}

\hyperref[tab:selfcond]{Table~\ref*{tab:selfcond}} replaces what a rollout conditions on,
all else as in \hyperref[sec:setup]{Sec.~\ref*{sec:setup}}. At inference every slot but the first
frame is generated; in training all are clean. Given a clean history, every model from $10$ to $200$\,s comes within about a point of the ground-truth videos'
imaging quality, so the loss comes from conditioning on generated frames. Every model pays it, two to three times more on unseen scenes, and only the $200$\,s run
loses clearly more. With a bank, most of it comes through the bank: a ground-truth bank alone
recovers three quarters to four fifths of the gap on unseen scenes. Guidance is not the cause:
without it, at seed $0$, the $200$\,s run still scores below $10$\,s top-$K$.

\input{tables/selfcond}

\makeatletter\setlength{\@fptop}{0pt}\makeatother
\input{figures/fig_teaser_grid}

%% file: tables/headline_seen.tex
\begin{table}[!htbp]
\caption{\textbf{Ablation of \hyperref[tab:headline]{Table~\ref*{tab:headline}} on seen scenes}, held-out episodes of the training scenes along new trajectories. Same runs, protocol and columns; best in bold, second underlined.}
\label{tab:headline_seen}
\centering\scriptsize
\setlength{\tabcolsep}{4pt}
\renewcommand{\arraystretch}{1.0}
\resizebox{\textwidth}{!}{%
\begin{tabular}{@{}l cc ccc ccc@{}}
\toprule
& \multicolumn{2}{c}{Fidelity to GT} & \multicolumn{3}{c}{Memory} & \multicolumn{3}{c}{VBench} \\
\cmidrule(lr){2-3}\cmidrule(lr){4-6}\cmidrule(l){7-9}
Training & PSNR$\uparrow$ & LPIPS$\downarrow$ & Revisit$\uparrow$ & Gain$\uparrow$ & DINO$\uparrow$ & Subj.$\uparrow$ & Backg.$\uparrow$ & Imag.$\uparrow$ \\
\midrule
\rowcolor{grouprow}
\multicolumn{9}{@{}l}{\textit{Seen scenes, new trajectories}} \\
$10$\,s, no retrieval & 11.48{\fontsize{5}{5}\selectfont$\,\pm$.18} & 0.610{\fontsize{5}{5}\selectfont$\,\pm$.008} & 0.333{\fontsize{5}{5}\selectfont$\,\pm$.017} & 0.074{\fontsize{5}{5}\selectfont$\,\pm$.021} & 0.633{\fontsize{5}{5}\selectfont$\,\pm$.013} & 0.801{\fontsize{5}{5}\selectfont$\,\pm$.002} & 0.882{\fontsize{5}{5}\selectfont$\,\pm$.003} & 0.648{\fontsize{5}{5}\selectfont$\,\pm$.005} \\
$10$\,s, $10$-chunk window & 11.08{\fontsize{5}{5}\selectfont$\,\pm$.14} & 0.641{\fontsize{5}{5}\selectfont$\,\pm$.008} & 0.321{\fontsize{5}{5}\selectfont$\,\pm$.019} & 0.081{\fontsize{5}{5}\selectfont$\,\pm$.017} & 0.599{\fontsize{5}{5}\selectfont$\,\pm$.008} & 0.774{\fontsize{5}{5}\selectfont$\,\pm$.004} & 0.875{\fontsize{5}{5}\selectfont$\,\pm$.002} & 0.635{\fontsize{5}{5}\selectfont$\,\pm$.005} \\
$10$\,s, top-$K$, no bank & 11.41{\fontsize{5}{5}\selectfont$\,\pm$.28} & \second{0.609}{\fontsize{5}{5}\selectfont$\,\pm$.008} & 0.427{\fontsize{5}{5}\selectfont$\,\pm$.026} & 0.186{\fontsize{5}{5}\selectfont$\,\pm$.026} & 0.744{\fontsize{5}{5}\selectfont$\,\pm$.011} & 0.799{\fontsize{5}{5}\selectfont$\,\pm$.006} & 0.883{\fontsize{5}{5}\selectfont$\,\pm$.004} & 0.654{\fontsize{5}{5}\selectfont$\,\pm$.008} \\
$100$\,s, per-segment retrieval, bank & 11.92{\fontsize{5}{5}\selectfont$\,\pm$.10} & \second{0.609}{\fontsize{5}{5}\selectfont$\,\pm$.006} & 0.454{\fontsize{5}{5}\selectfont$\,\pm$.025} & 0.226{\fontsize{5}{5}\selectfont$\,\pm$.031} & 0.761{\fontsize{5}{5}\selectfont$\,\pm$.020} & 0.797{\fontsize{5}{5}\selectfont$\,\pm$.005} & 0.879{\fontsize{5}{5}\selectfont$\,\pm$.003} & \second{0.658}{\fontsize{5}{5}\selectfont$\,\pm$.008} \\
$100$\,s, top-$K$, ordered bank & 12.13{\fontsize{5}{5}\selectfont$\,\pm$.18} & \best{0.603}{\fontsize{5}{5}\selectfont$\,\pm$.004} & \second{0.536}{\fontsize{5}{5}\selectfont$\,\pm$.017} & \second{0.280}{\fontsize{5}{5}\selectfont$\,\pm$.019} & \second{0.834}{\fontsize{5}{5}\selectfont$\,\pm$.007} & 0.800{\fontsize{5}{5}\selectfont$\,\pm$.002} & \second{0.884}{\fontsize{5}{5}\selectfont$\,\pm$.002} & 0.639{\fontsize{5}{5}\selectfont$\,\pm$.011} \\
$100$\,s, top-$K$, bank & \second{12.21}{\fontsize{5}{5}\selectfont$\,\pm$.15} & 0.619{\fontsize{5}{5}\selectfont$\,\pm$.003} & 0.492{\fontsize{5}{5}\selectfont$\,\pm$.021} & 0.258{\fontsize{5}{5}\selectfont$\,\pm$.027} & 0.820{\fontsize{5}{5}\selectfont$\,\pm$.010} & \second{0.804}{\fontsize{5}{5}\selectfont$\,\pm$.001} & 0.880{\fontsize{5}{5}\selectfont$\,\pm$.002} & 0.646{\fontsize{5}{5}\selectfont$\,\pm$.015} \\
$200$\,s, top-$K$, bank & \best{12.24}{\fontsize{5}{5}\selectfont$\,\pm$.13} & \best{0.603}{\fontsize{5}{5}\selectfont$\,\pm$.005} & \best{0.548}{\fontsize{5}{5}\selectfont$\,\pm$.006} & \best{0.308}{\fontsize{5}{5}\selectfont$\,\pm$.022} & \best{0.858}{\fontsize{5}{5}\selectfont$\,\pm$.009} & \best{0.806}{\fontsize{5}{5}\selectfont$\,\pm$.006} & \best{0.886}{\fontsize{5}{5}\selectfont$\,\pm$.003} & 0.637{\fontsize{5}{5}\selectfont$\,\pm$.012} \\
$400$\,s, top-$K$, bank & 11.72{\fontsize{5}{5}\selectfont$\,\pm$.09} & 0.621{\fontsize{5}{5}\selectfont$\,\pm$.005} & 0.487{\fontsize{5}{5}\selectfont$\,\pm$.015} & 0.250{\fontsize{5}{5}\selectfont$\,\pm$.014} & 0.820{\fontsize{5}{5}\selectfont$\,\pm$.010} & 0.801{\fontsize{5}{5}\selectfont$\,\pm$.004} & 0.882{\fontsize{5}{5}\selectfont$\,\pm$.004} & \best{0.684}{\fontsize{5}{5}\selectfont$\,\pm$.010} \\
\bottomrule
\end{tabular}}
\end{table}

%% file: tables/sota_std.tex
\begin{table}[!htbp]
\caption{\textbf{\hyperref[tab:sota]{Table~\ref*{tab:sota}} with its spreads.} Mean over five rollout seeds, with the standard deviation of the per-seed means in small type.}
\label{tab:sota_std}
\centering\scriptsize
\setlength{\tabcolsep}{4pt}
\renewcommand{\arraystretch}{1.12}
\resizebox{\textwidth}{!}{%
\begin{tabular}{@{}l cccc cccc ccc@{}}
\toprule
& \multicolumn{4}{c}{Starting pose} & \multicolumn{4}{c}{Mid-path} & \multicolumn{3}{c}{VBench} \\
\cmidrule(lr){2-5}\cmidrule(lr){6-9}\cmidrule(l){10-12}
Method & Revisit$\uparrow$ & Gain$\uparrow$ & DINO$\uparrow$ & D.\,Gain$\uparrow$ & Revisit$\uparrow$ & Gain$\uparrow$ & DINO$\uparrow$ & D.\,Gain$\uparrow$ & Subj.$\uparrow$ & Backg.$\uparrow$ & Imag.$\uparrow$ \\
\midrule
LingBot-World~2.0 & 0.143{\fontsize{5}{5}\selectfont$\,\pm$.026} & $-$0.006{\fontsize{5}{5}\selectfont$\,\pm$.032} & 0.526{\fontsize{5}{5}\selectfont$\,\pm$.013} & $-$0.001{\fontsize{5}{5}\selectfont$\,\pm$.029} & 0.147{\fontsize{5}{5}\selectfont$\,\pm$.017} & 0.011{\fontsize{5}{5}\selectfont$\,\pm$.015} & 0.546{\fontsize{5}{5}\selectfont$\,\pm$.020} & 0.017{\fontsize{5}{5}\selectfont$\,\pm$.019} & 0.765{\fontsize{5}{5}\selectfont$\,\pm$.004} & 0.852{\fontsize{5}{5}\selectfont$\,\pm$.004} & 0.741{\fontsize{5}{5}\selectfont$\,\pm$.004} \\
DreamX-World~1.0 & 0.146{\fontsize{5}{5}\selectfont$\,\pm$.009} & 0.033{\fontsize{5}{5}\selectfont$\,\pm$.021} & 0.547{\fontsize{5}{5}\selectfont$\,\pm$.016} & $-$0.045{\fontsize{5}{5}\selectfont$\,\pm$.013} & 0.109{\fontsize{5}{5}\selectfont$\,\pm$.006} & 0.009{\fontsize{5}{5}\selectfont$\,\pm$.007} & 0.591{\fontsize{5}{5}\selectfont$\,\pm$.013} & 0.007{\fontsize{5}{5}\selectfont$\,\pm$.013} & 0.781{\fontsize{5}{5}\selectfont$\,\pm$.001} & 0.868{\fontsize{5}{5}\selectfont$\,\pm$.001} & 0.684{\fontsize{5}{5}\selectfont$\,\pm$.006} \\
HY-WorldPlay~1.5$^\dagger$ & 0.243{\fontsize{5}{5}\selectfont$\,\pm$.014} & 0.040{\fontsize{5}{5}\selectfont$\,\pm$.034} & 0.679{\fontsize{5}{5}\selectfont$\,\pm$.022} & 0.025{\fontsize{5}{5}\selectfont$\,\pm$.031} & 0.194{\fontsize{5}{5}\selectfont$\,\pm$.011} & 0.037{\fontsize{5}{5}\selectfont$\,\pm$.010} & 0.674{\fontsize{5}{5}\selectfont$\,\pm$.004} & 0.027{\fontsize{5}{5}\selectfont$\,\pm$.004} & \best{0.839}{\fontsize{5}{5}\selectfont$\,\pm$.002} & \best{0.908}{\fontsize{5}{5}\selectfont$\,\pm$.001} & 0.702{\fontsize{5}{5}\selectfont$\,\pm$.008} \\
Matrix-Game~3.0 & 0.254{\fontsize{5}{5}\selectfont$\,\pm$.018} & 0.075{\fontsize{5}{5}\selectfont$\,\pm$.017} & 0.717{\fontsize{5}{5}\selectfont$\,\pm$.013} & 0.042{\fontsize{5}{5}\selectfont$\,\pm$.019} & 0.283{\fontsize{5}{5}\selectfont$\,\pm$.012} & 0.089{\fontsize{5}{5}\selectfont$\,\pm$.020} & \second{0.760}{\fontsize{5}{5}\selectfont$\,\pm$.016} & 0.100{\fontsize{5}{5}\selectfont$\,\pm$.016} & \second{0.813}{\fontsize{5}{5}\selectfont$\,\pm$.004} & \second{0.895}{\fontsize{5}{5}\selectfont$\,\pm$.002} & \best{0.761}{\fontsize{5}{5}\selectfont$\,\pm$.003} \\
Infinite-World & 0.109{\fontsize{5}{5}\selectfont$\,\pm$.022} & 0.019{\fontsize{5}{5}\selectfont$\,\pm$.031} & 0.610{\fontsize{5}{5}\selectfont$\,\pm$.026} & 0.013{\fontsize{5}{5}\selectfont$\,\pm$.022} & 0.168{\fontsize{5}{5}\selectfont$\,\pm$.017} & 0.070{\fontsize{5}{5}\selectfont$\,\pm$.012} & 0.646{\fontsize{5}{5}\selectfont$\,\pm$.013} & 0.052{\fontsize{5}{5}\selectfont$\,\pm$.007} & 0.789{\fontsize{5}{5}\selectfont$\,\pm$.005} & 0.856{\fontsize{5}{5}\selectfont$\,\pm$.003} & \second{0.746}{\fontsize{5}{5}\selectfont$\,\pm$.006} \\
CaR & \second{0.350}{\fontsize{5}{5}\selectfont$\,\pm$.017} & \second{0.220}{\fontsize{5}{5}\selectfont$\,\pm$.048} & \best{0.843}{\fontsize{5}{5}\selectfont$\,\pm$.032} & \best{0.305}{\fontsize{5}{5}\selectfont$\,\pm$.028} & \best{0.450}{\fontsize{5}{5}\selectfont$\,\pm$.010} & \second{0.312}{\fontsize{5}{5}\selectfont$\,\pm$.035} & 0.759{\fontsize{5}{5}\selectfont$\,\pm$.008} & \second{0.211}{\fontsize{5}{5}\selectfont$\,\pm$.011} & 0.787{\fontsize{5}{5}\selectfont$\,\pm$.007} & 0.846{\fontsize{5}{5}\selectfont$\,\pm$.003} & 0.715{\fontsize{5}{5}\selectfont$\,\pm$.003} \\
\rowcolor{ourrow}
\multicolumn{1}{@{}>{\columncolor{ourrow}[0pt][\tabcolsep]}l}{\textbf{\ours (ours)}} & \best{0.365}{\fontsize{5}{5}\selectfont$\,\pm$.013} & \best{0.230}{\fontsize{5}{5}\selectfont$\,\pm$.024} & \second{0.787}{\fontsize{5}{5}\selectfont$\,\pm$.017} & \second{0.219}{\fontsize{5}{5}\selectfont$\,\pm$.031} & \second{0.447}{\fontsize{5}{5}\selectfont$\,\pm$.011} & \best{0.329}{\fontsize{5}{5}\selectfont$\,\pm$.015} & \best{0.816}{\fontsize{5}{5}\selectfont$\,\pm$.020} & \best{0.258}{\fontsize{5}{5}\selectfont$\,\pm$.015} & 0.778{\fontsize{5}{5}\selectfont$\,\pm$.007} & 0.857{\fontsize{5}{5}\selectfont$\,\pm$.004} & \multicolumn{1}{>{\columncolor{ourrow}[\tabcolsep][0pt]}c@{}}{0.705{\fontsize{5}{5}\selectfont$\,\pm$.012}} \\
\bottomrule
\end{tabular}}
\end{table}

%% file: tables/tenswhy.tex
\begin{table}[!htbp]
\caption{\textbf{Gain by return interval and type.} Pixel Gain with the five rollout seeds pooled;
seen and unseen scenes as in \hyperref[tab:headline]{Table~\ref*{tab:headline}}, web photographs as in
\hyperref[tab:sota]{Table~\ref*{tab:sota}}. \emph{Start}: the earlier frame lies in the first second;
\emph{mid}: every other return. Pooling pairs rather than clips makes \emph{all} differ slightly from
\hyperref[tab:headline]{Table~\ref*{tab:headline}}. For the $100$\,s run the three interval bins hold
$25$, $140$ and $470$ returns on seen scenes,
$30$, $105$ and $230$ on unseen ones and
$20$, $100$ and $605$ on web photographs.}
\label{tab:tenswhy}
\centering\scriptsize
\setlength{\tabcolsep}{5pt}
\renewcommand{\arraystretch}{1.0}
\begin{tabular}{@{}l ccc ccc@{}}
\toprule
& \multicolumn{3}{c}{Interval between the visits} & \multicolumn{3}{c}{Type} \\
\cmidrule(lr){2-4}\cmidrule(l){5-7}
Training & $8$--$10$\,s & $10$--$20$\,s & $20$--$60$\,s & Start & Mid & All \\
\midrule
\rowcolor{grouprow}
\multicolumn{7}{@{}l}{\textit{Seen scenes}} \\
$10$\,s, no retrieval & 0.230 & 0.064 & 0.073 & 0.160 & 0.067 & 0.078 \\
$10$\,s, top-$K$ & 0.397 & 0.174 & 0.185 & 0.130 & 0.199 & 0.191 \\
$100$\,s, top-$K$, bank & 0.337 & 0.216 & 0.253 & 0.213 & 0.252 & 0.248 \\
\addlinespace[1pt]
\rowcolor{grouprow}
\multicolumn{7}{@{}l}{\textit{Unseen scenes}} \\
$10$\,s, no retrieval & 0.100 & 0.069 & 0.016 & 0.220 & 0.031 & 0.038 \\
$10$\,s, top-$K$ & 0.466 & 0.241 & 0.242 & 0.250 & 0.260 & 0.260 \\
$100$\,s, top-$K$, bank & 0.414 & 0.279 & 0.270 & 0.222 & 0.288 & 0.285 \\
\addlinespace[1pt]
\rowcolor{grouprow}
\multicolumn{7}{@{}l}{\textit{Web photographs}} \\
$10$\,s, no retrieval & 0.062 & 0.015 & 0.065 & 0.135 & 0.046 & 0.058 \\
$10$\,s, top-$K$ & 0.515 & 0.284 & 0.221 & 0.182 & 0.247 & 0.238 \\
$100$\,s, top-$K$, bank & 0.458 & 0.374 & 0.307 & 0.249 & 0.332 & 0.321 \\
\bottomrule
\end{tabular}
\end{table}

%% file: tables/cost.tex
\begin{table}[!htbp]
\caption{\textbf{Training cost of the runs in
\hyperref[tab:headline]{Table~\ref*{tab:headline}}}, from their logs. Slots: median and
maximum over logged steps. Wall clock is the step time over the $6{,}000$ steps every run trains for.}
\label{tab:cost}
\centering\scriptsize
\setlength{\tabcolsep}{5pt}
\renewcommand{\arraystretch}{1.0}
\begin{tabular}{@{}l cccc@{}}
\toprule
Training & Sequence slots & Step time (s) & Wall clock (h) & Peak mem.\ (GB) \\
\midrule
$10$\,s, no retrieval & $41$ / $41$ & $3.49$ & $5.8$ & $16.05$ \\
$10$\,s, $10$-chunk window & $41$ / $41$ & $4.02$ & $6.7$ & $16.05$ \\
$10$\,s, top-$K$, no bank & $41$ / $41$ & $3.68$ & $6.1$ & $16.05$ \\
$100$\,s, per-segment retrieval, bank & $51$ / $51$ & $4.17$ & $6.9$ & $17.99$ \\
$100$\,s, top-$K$, ordered bank & $65$ / $105$ & $8.18$ & $13.6$ & $28.45$ \\
\rowcolor{ourrow}
\multicolumn{1}{@{}>{\columncolor{ourrow}[0pt][\tabcolsep]}l}{$100$\,s, top-$K$, bank} & $65$ / $105$ & $8.02$ & $13.4$ & \multicolumn{1}{>{\columncolor{ourrow}[\tabcolsep][0pt]}c@{}}{$28.45$} \\
\rowcolor{ourrow}
\multicolumn{1}{@{}>{\columncolor{ourrow}[0pt][\tabcolsep]}l}{$200$\,s, top-$K$, bank} & $70$ / $105$ & $8.67$ & $14.4$ & \multicolumn{1}{>{\columncolor{ourrow}[\tabcolsep][0pt]}c@{}}{$28.45$} \\
\rowcolor{ourrow}
\multicolumn{1}{@{}>{\columncolor{ourrow}[0pt][\tabcolsep]}l}{$400$\,s, top-$K$, bank} & $72$ / $105$ & $9.00$ & $15.0$ & \multicolumn{1}{>{\columncolor{ourrow}[\tabcolsep][0pt]}c@{}}{$28.45$} \\
\bottomrule
\end{tabular}
\end{table}

%% file: tables/camfollow.tex
\begin{table}[!htbp]
\caption{\textbf{Camera following}: yaw only, web
photographs, first $60$\,s, rollout seed $0$. $^\dagger$Field of view implied by image
shift per degree of prescribed yaw; the paths assume $90^\circ$. \ours: mean over six runs of \hyperref[tab:headline]{Table~\ref*{tab:headline}}, the three trained on $10$\,s clips and the $100$, $200$ and $400$\,s runs.}
\label{tab:camfollow}
\centering\scriptsize
\setlength{\tabcolsep}{6pt}
\renewcommand{\arraystretch}{1.0}
\begin{tabular}{@{}l ccc@{}}
\toprule
System & Correlation$\uparrow$ & Direction$\uparrow$ & Implied hfov$^\dagger$ \\
\midrule
\rowcolor{grouprow}
\multicolumn{4}{@{}l}{\textit{Calibration}} \\
Rendered frames, engine poses & $0.906$ & $0.974$ & $89.1^\circ$ \\
\rowcolor{grouprow}
\multicolumn{4}{@{}l}{\textit{Open world models}} \\
LingBot-World~2.0 & $0.855$ & $0.966$ & $73.3^\circ$ \\
DreamX-World~1.0 & $0.825$ & $0.955$ & $132.5^\circ$ \\
HY-WorldPlay~1.5 & $0.697$ & $0.909$ & $123.6^\circ$ \\
Matrix-Game~3.0 & $0.849$ & $0.968$ & $98.4^\circ$ \\
Infinite-World & $0.877$ & $0.958$ & $107.5^\circ$ \\
\midrule
\textbf{\ours} (mean of $6$ runs) & $0.961$ & $0.986$ & $90.6^\circ$ \\
\bottomrule
\end{tabular}
\end{table}

%% file: tables/bankuse.tex
\begin{table}[!htbp]
\caption{\textbf{Bank contents at inference.} The $100$\,s run with only the bank's frames
replaced: its top-$K$ (ours), nothing, $K$ random frames from the same history, or $K$
frames from another episode. Web photographs, first $60$\,s, returns split as in
\hyperref[tab:sota]{Table~\ref*{tab:sota}}, DINO columns as there; rollout seed $0$.}
\label{tab:bankuse}
\centering\scriptsize
\setlength{\tabcolsep}{4pt}
\renewcommand{\arraystretch}{1.0}
\begin{tabular}{@{}l cccc cccc@{}}
\toprule
& \multicolumn{4}{c}{Starting pose} & \multicolumn{4}{c}{Mid-path} \\
\cmidrule(lr){2-5}\cmidrule(l){6-9}
Bank contents & Revisit$\uparrow$ & Gain$\uparrow$ & DINO$\uparrow$ & D.\,Gain$\uparrow$ & Revisit$\uparrow$ & Gain$\uparrow$ & DINO$\uparrow$ & D.\,Gain$\uparrow$ \\
\midrule
\rowcolor{ourrow}
\multicolumn{1}{@{}>{\columncolor{ourrow}[0pt][\tabcolsep]}l}{Top-$K$ (ours)} & \textbf{0.377} & \textbf{0.213} & \textbf{0.794} & \textbf{0.218} & \textbf{0.440} & \textbf{0.319} & \textbf{0.826} & \multicolumn{1}{>{\columncolor{ourrow}[\tabcolsep][0pt]}c@{}}{\textbf{0.256}} \\
Empty & 0.176 & 0.053 & 0.553 & 0.085 & 0.105 & 0.019 & 0.504 & 0.035 \\
$K$ random frames & 0.153 & 0.055 & 0.583 & 0.044 & 0.207 & 0.109 & 0.634 & 0.087 \\
$K$ frames, another episode & 0.071 & 0.012 & 0.445 & $-$0.020 & 0.073 & 0.016 & 0.459 & $-$0.014 \\
\bottomrule
\end{tabular}
\par\vspace{5pt}
\end{table}

%% file: figures/fig_sota_more.tex
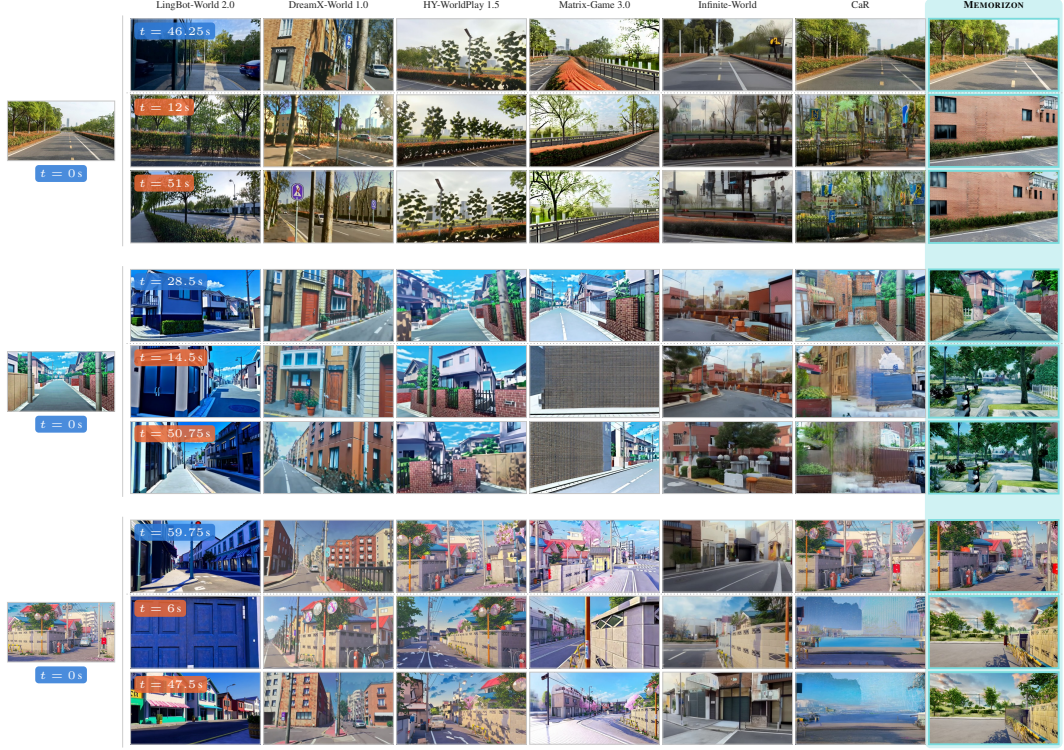
\begin{figure}[!tbp]
\centering
\resizebox{\textwidth}{!}{%
\begin{tikzpicture}[
  img/.style={inner sep=0pt,outer sep=0pt},
  hd/.style={font=\fontsize{4.4}{5.2}\selectfont,anchor=south,inner sep=1.5pt,text=black!85,text height=1.7ex,text depth=0.5ex},
  tag/.style={font=\fontsize{5}{6}\selectfont,text=white,fill opacity=0.85,text opacity=1,rounded corners=1.5pt,inner sep=2pt,anchor=north west},
  sc/.style={font=\tiny,anchor=north,inner sep=2pt,text=black!55},
  blk/.style={font=\fontsize{6}{7}\selectfont\bfseries,anchor=south west,inner sep=2.5pt,text=black!75,fill=black!6,rounded corners=2pt},
]
\fill[ourrow,rounded corners=2.5pt] (13.860,0.30) rectangle (15.940,-11.020);
\node[img,anchor=north west] (sv00in) at (0,-1.243) {\includegraphics[width=1.62cm]{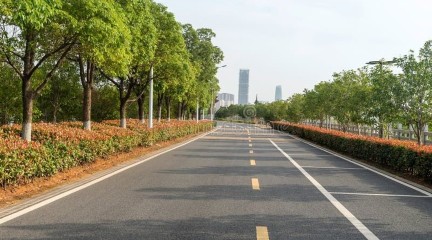}};
\draw[black!22,line width=0.4pt] (sv00in.north west) rectangle (sv00in.south east);
\node[tag,fill=selMix,anchor=north] at ([yshift=-0.06cm]sv00in.south) {$t = 0$\,s};
\draw[black!20,line width=0.4pt] (1.740,0.050) -- (1.740,-3.437);
\node[img,anchor=north west] (sv00r1lingbot) at (1.860,0.000) {\includegraphics[width=1.96cm]{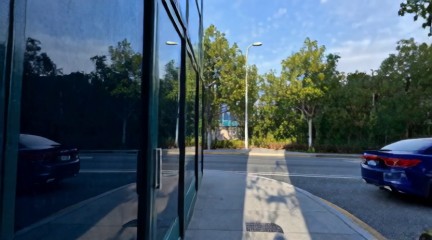}};
\draw[black!22,line width=0.4pt] (sv00r1lingbot.north west) rectangle (sv00r1lingbot.south east);
\node[tag,fill=selMix] at ([shift={(0.05,-0.05)}]sv00r1lingbot.north west) {$t = 46.25$\,s};
\node[hd] at ([yshift=0.03cm]sv00r1lingbot.north) {LingBot-World~2.0};
\node[img,anchor=north west] (sv00r1dreamx) at (3.870,0.000) {\includegraphics[width=1.96cm]{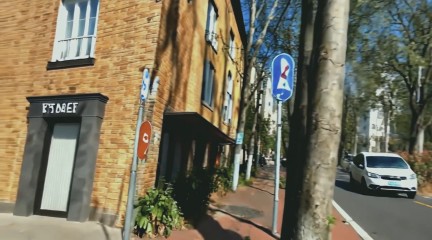}};
\draw[black!22,line width=0.4pt] (sv00r1dreamx.north west) rectangle (sv00r1dreamx.south east);
\node[hd] at ([yshift=0.03cm]sv00r1dreamx.north) {DreamX-World~1.0};
\node[img,anchor=north west] (sv00r1hyworld) at (5.880,0.000) {\includegraphics[width=1.96cm]{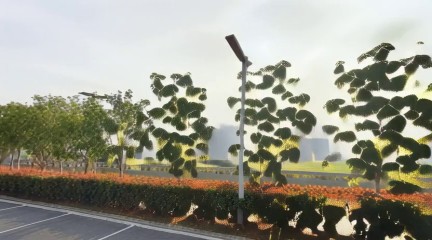}};
\draw[black!22,line width=0.4pt] (sv00r1hyworld.north west) rectangle (sv00r1hyworld.south east);
\node[hd] at ([yshift=0.03cm]sv00r1hyworld.north) {HY-WorldPlay~1.5};
\node[img,anchor=north west] (sv00r1matrix) at (7.890,0.000) {\includegraphics[width=1.96cm]{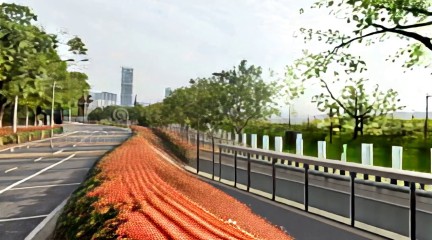}};
\draw[black!22,line width=0.4pt] (sv00r1matrix.north west) rectangle (sv00r1matrix.south east);
\node[hd] at ([yshift=0.03cm]sv00r1matrix.north) {Matrix-Game~3.0};
\node[img,anchor=north west] (sv00r1infworld) at (9.900,0.000) {\includegraphics[width=1.96cm]{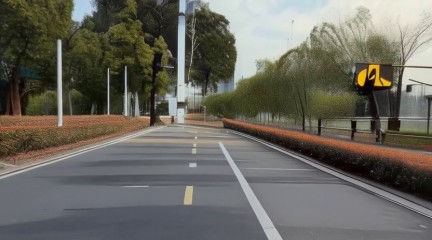}};
\draw[black!22,line width=0.4pt] (sv00r1infworld.north west) rectangle (sv00r1infworld.south east);
\node[hd] at ([yshift=0.03cm]sv00r1infworld.north) {Infinite-World};
\node[img,anchor=north west] (sv00r1car) at (11.910,0.000) {\includegraphics[width=1.96cm]{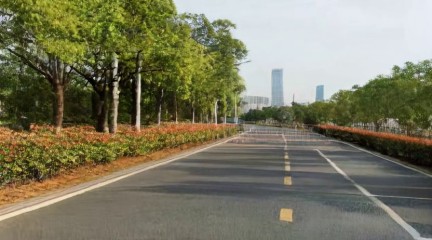}};
\draw[black!22,line width=0.4pt] (sv00r1car.north west) rectangle (sv00r1car.south east);
\node[hd] at ([yshift=0.03cm]sv00r1car.north) {CaR};
\node[img,anchor=north west] (sv00r1ours) at (13.920,0.000) {\includegraphics[width=1.96cm]{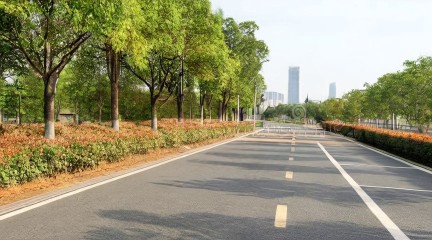}};
\draw[ourline,line width=0.8pt] (sv00r1ours.north west) rectangle (sv00r1ours.south east);
\node[hd,font=\fontsize{4.4}{5.2}\selectfont\bfseries,text=black] at ([yshift=0.03cm]sv00r1ours.north) {\ours};
\draw[black!30,line width=0.4pt,densely dotted] (1.800,-1.119) -- (15.940,-1.119);
\node[img,anchor=north west] (sv00r2lingbot) at (1.860,-1.149) {\includegraphics[width=1.96cm]{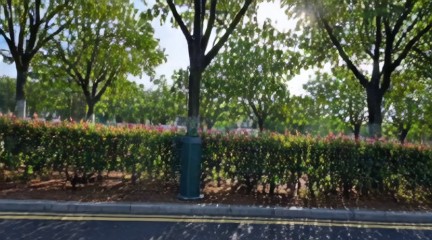}};
\draw[black!22,line width=0.4pt] (sv00r2lingbot.north west) rectangle (sv00r2lingbot.south east);
\node[tag,fill=selOv] at ([shift={(0.05,-0.05)}]sv00r2lingbot.north west) {$t = 12$\,s};
\node[img,anchor=north west] (sv00r2dreamx) at (3.870,-1.149) {\includegraphics[width=1.96cm]{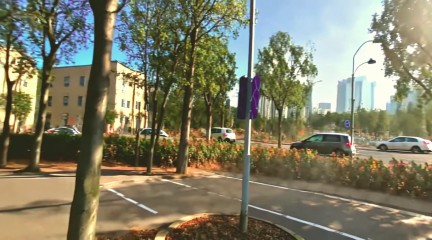}};
\draw[black!22,line width=0.4pt] (sv00r2dreamx.north west) rectangle (sv00r2dreamx.south east);
\node[img,anchor=north west] (sv00r2hyworld) at (5.880,-1.149) {\includegraphics[width=1.96cm]{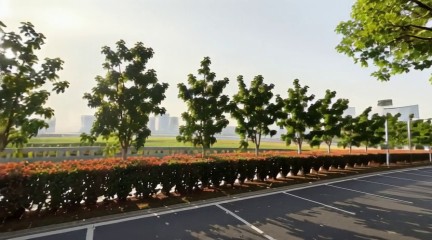}};
\draw[black!22,line width=0.4pt] (sv00r2hyworld.north west) rectangle (sv00r2hyworld.south east);
\node[img,anchor=north west] (sv00r2matrix) at (7.890,-1.149) {\includegraphics[width=1.96cm]{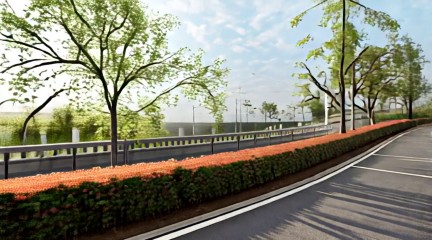}};
\draw[black!22,line width=0.4pt] (sv00r2matrix.north west) rectangle (sv00r2matrix.south east);
\node[img,anchor=north west] (sv00r2infworld) at (9.900,-1.149) {\includegraphics[width=1.96cm]{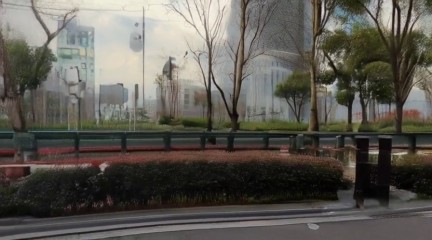}};
\draw[black!22,line width=0.4pt] (sv00r2infworld.north west) rectangle (sv00r2infworld.south east);
\node[img,anchor=north west] (sv00r2car) at (11.910,-1.149) {\includegraphics[width=1.96cm]{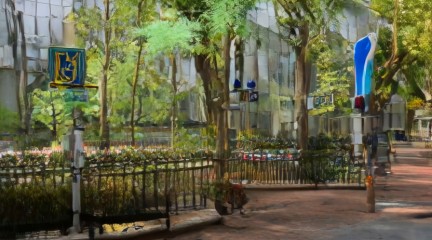}};
\draw[black!22,line width=0.4pt] (sv00r2car.north west) rectangle (sv00r2car.south east);
\node[img,anchor=north west] (sv00r2ours) at (13.920,-1.149) {\includegraphics[width=1.96cm]{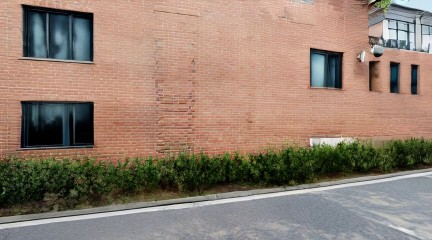}};
\draw[ourline,line width=0.8pt] (sv00r2ours.north west) rectangle (sv00r2ours.south east);
\node[img,anchor=north west] (sv00r3lingbot) at (1.860,-2.298) {\includegraphics[width=1.96cm]{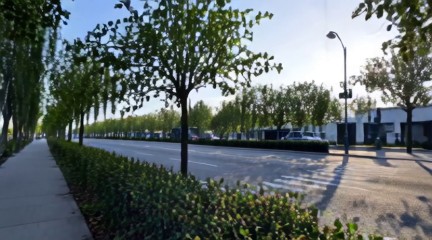}};
\draw[black!22,line width=0.4pt] (sv00r3lingbot.north west) rectangle (sv00r3lingbot.south east);
\node[tag,fill=selOv] at ([shift={(0.05,-0.05)}]sv00r3lingbot.north west) {$t = 51$\,s};
\node[img,anchor=north west] (sv00r3dreamx) at (3.870,-2.298) {\includegraphics[width=1.96cm]{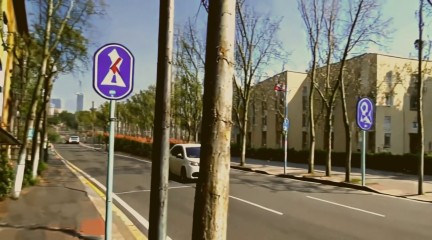}};
\draw[black!22,line width=0.4pt] (sv00r3dreamx.north west) rectangle (sv00r3dreamx.south east);
\node[img,anchor=north west] (sv00r3hyworld) at (5.880,-2.298) {\includegraphics[width=1.96cm]{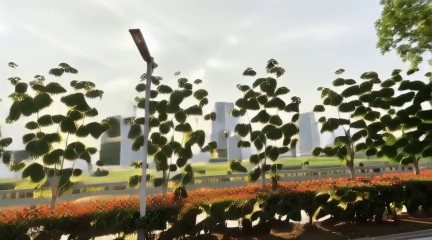}};
\draw[black!22,line width=0.4pt] (sv00r3hyworld.north west) rectangle (sv00r3hyworld.south east);
\node[img,anchor=north west] (sv00r3matrix) at (7.890,-2.298) {\includegraphics[width=1.96cm]{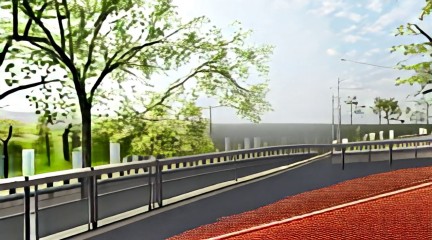}};
\draw[black!22,line width=0.4pt] (sv00r3matrix.north west) rectangle (sv00r3matrix.south east);
\node[img,anchor=north west] (sv00r3infworld) at (9.900,-2.298) {\includegraphics[width=1.96cm]{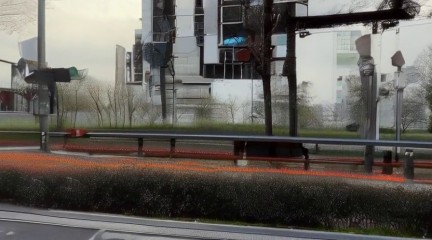}};
\draw[black!22,line width=0.4pt] (sv00r3infworld.north west) rectangle (sv00r3infworld.south east);
\node[img,anchor=north west] (sv00r3car) at (11.910,-2.298) {\includegraphics[width=1.96cm]{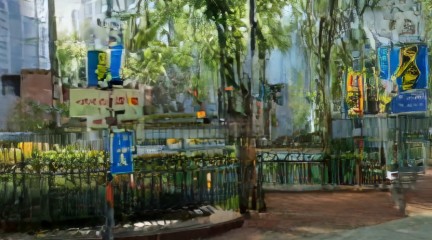}};
\draw[black!22,line width=0.4pt] (sv00r3car.north west) rectangle (sv00r3car.south east);
\node[img,anchor=north west] (sv00r3ours) at (13.920,-2.298) {\includegraphics[width=1.96cm]{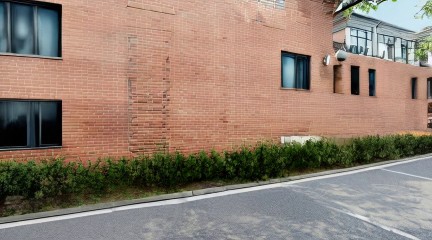}};
\draw[ourline,line width=0.8pt] (sv00r3ours.north west) rectangle (sv00r3ours.south east);
\node[img,anchor=north west] (sv01in) at (0,-5.030) {\includegraphics[width=1.62cm]{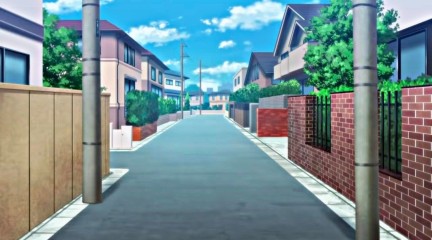}};
\draw[black!22,line width=0.4pt] (sv01in.north west) rectangle (sv01in.south east);
\node[tag,fill=selMix,anchor=north] at ([yshift=-0.06cm]sv01in.south) {$t = 0$\,s};
\draw[black!20,line width=0.4pt] (1.740,-3.737) -- (1.740,-7.223);
\node[img,anchor=north west] (sv01r1lingbot) at (1.860,-3.787) {\includegraphics[width=1.96cm]{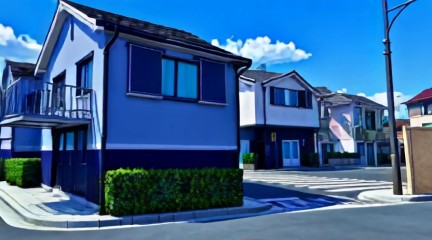}};
\draw[black!22,line width=0.4pt] (sv01r1lingbot.north west) rectangle (sv01r1lingbot.south east);
\node[tag,fill=selMix] at ([shift={(0.05,-0.05)}]sv01r1lingbot.north west) {$t = 28.5$\,s};
\node[img,anchor=north west] (sv01r1dreamx) at (3.870,-3.787) {\includegraphics[width=1.96cm]{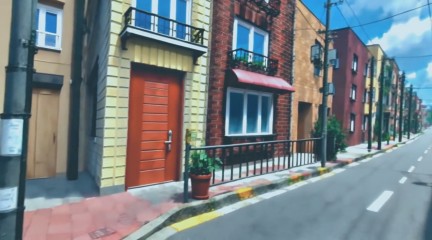}};
\draw[black!22,line width=0.4pt] (sv01r1dreamx.north west) rectangle (sv01r1dreamx.south east);
\node[img,anchor=north west] (sv01r1hyworld) at (5.880,-3.787) {\includegraphics[width=1.96cm]{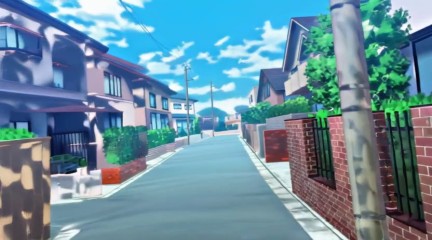}};
\draw[black!22,line width=0.4pt] (sv01r1hyworld.north west) rectangle (sv01r1hyworld.south east);
\node[img,anchor=north west] (sv01r1matrix) at (7.890,-3.787) {\includegraphics[width=1.96cm]{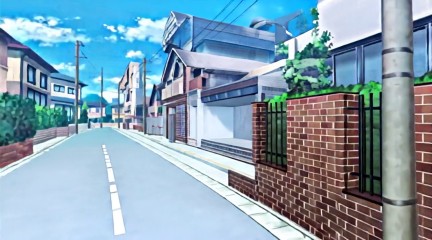}};
\draw[black!22,line width=0.4pt] (sv01r1matrix.north west) rectangle (sv01r1matrix.south east);
\node[img,anchor=north west] (sv01r1infworld) at (9.900,-3.787) {\includegraphics[width=1.96cm]{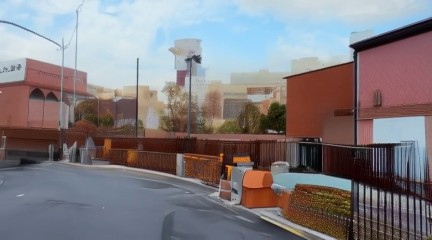}};
\draw[black!22,line width=0.4pt] (sv01r1infworld.north west) rectangle (sv01r1infworld.south east);
\node[img,anchor=north west] (sv01r1car) at (11.910,-3.787) {\includegraphics[width=1.96cm]{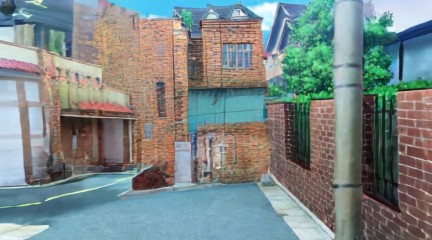}};
\draw[black!22,line width=0.4pt] (sv01r1car.north west) rectangle (sv01r1car.south east);
\node[img,anchor=north west] (sv01r1ours) at (13.920,-3.787) {\includegraphics[width=1.96cm]{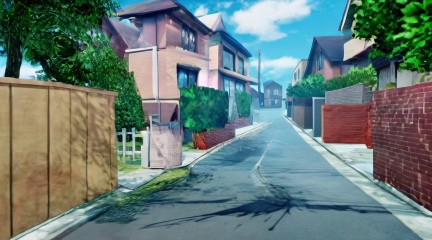}};
\draw[ourline,line width=0.8pt] (sv01r1ours.north west) rectangle (sv01r1ours.south east);
\draw[black!30,line width=0.4pt,densely dotted] (1.800,-4.906) -- (15.940,-4.906);
\node[img,anchor=north west] (sv01r2lingbot) at (1.860,-4.936) {\includegraphics[width=1.96cm]{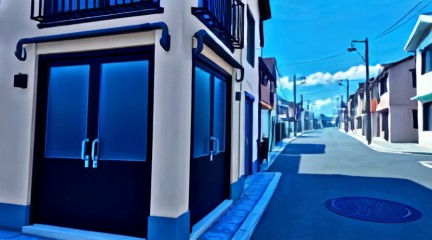}};
\draw[black!22,line width=0.4pt] (sv01r2lingbot.north west) rectangle (sv01r2lingbot.south east);
\node[tag,fill=selOv] at ([shift={(0.05,-0.05)}]sv01r2lingbot.north west) {$t = 14.5$\,s};
\node[img,anchor=north west] (sv01r2dreamx) at (3.870,-4.936) {\includegraphics[width=1.96cm]{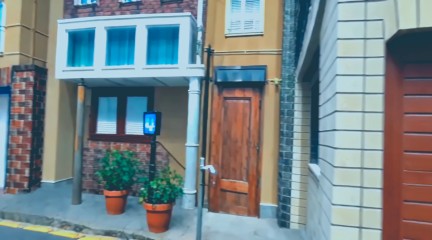}};
\draw[black!22,line width=0.4pt] (sv01r2dreamx.north west) rectangle (sv01r2dreamx.south east);
\node[img,anchor=north west] (sv01r2hyworld) at (5.880,-4.936) {\includegraphics[width=1.96cm]{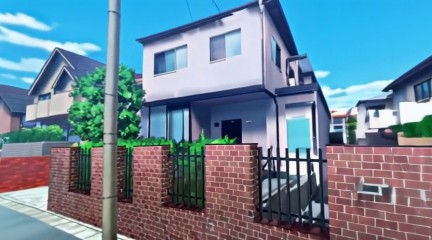}};
\draw[black!22,line width=0.4pt] (sv01r2hyworld.north west) rectangle (sv01r2hyworld.south east);
\node[img,anchor=north west] (sv01r2matrix) at (7.890,-4.936) {\includegraphics[width=1.96cm]{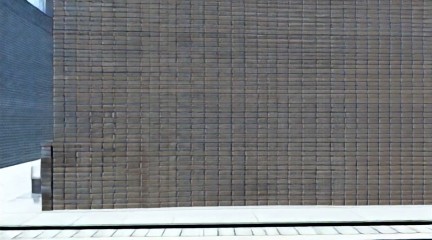}};
\draw[black!22,line width=0.4pt] (sv01r2matrix.north west) rectangle (sv01r2matrix.south east);
\node[img,anchor=north west] (sv01r2infworld) at (9.900,-4.936) {\includegraphics[width=1.96cm]{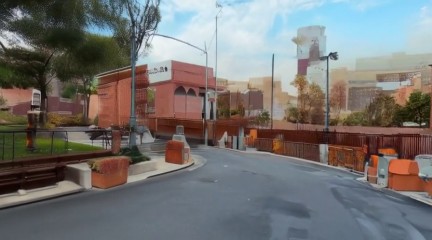}};
\draw[black!22,line width=0.4pt] (sv01r2infworld.north west) rectangle (sv01r2infworld.south east);
\node[img,anchor=north west] (sv01r2car) at (11.910,-4.936) {\includegraphics[width=1.96cm]{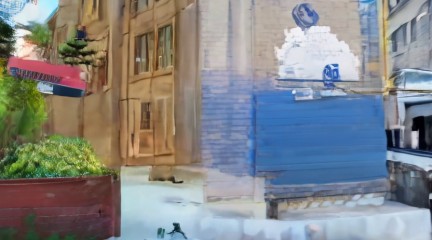}};
\draw[black!22,line width=0.4pt] (sv01r2car.north west) rectangle (sv01r2car.south east);
\node[img,anchor=north west] (sv01r2ours) at (13.920,-4.936) {\includegraphics[width=1.96cm]{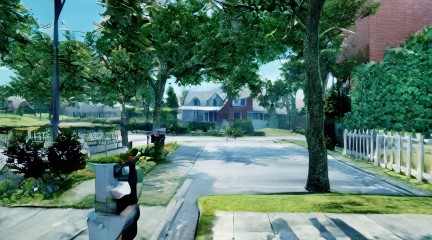}};
\draw[ourline,line width=0.8pt] (sv01r2ours.north west) rectangle (sv01r2ours.south east);
\node[img,anchor=north west] (sv01r3lingbot) at (1.860,-6.084) {\includegraphics[width=1.96cm]{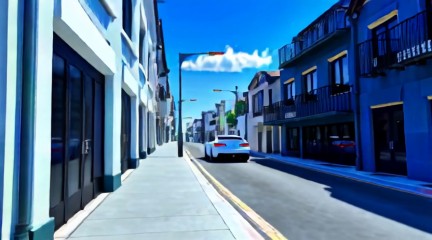}};
\draw[black!22,line width=0.4pt] (sv01r3lingbot.north west) rectangle (sv01r3lingbot.south east);
\node[tag,fill=selOv] at ([shift={(0.05,-0.05)}]sv01r3lingbot.north west) {$t = 50.75$\,s};
\node[img,anchor=north west] (sv01r3dreamx) at (3.870,-6.084) {\includegraphics[width=1.96cm]{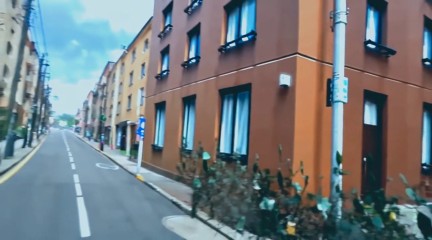}};
\draw[black!22,line width=0.4pt] (sv01r3dreamx.north west) rectangle (sv01r3dreamx.south east);
\node[img,anchor=north west] (sv01r3hyworld) at (5.880,-6.084) {\includegraphics[width=1.96cm]{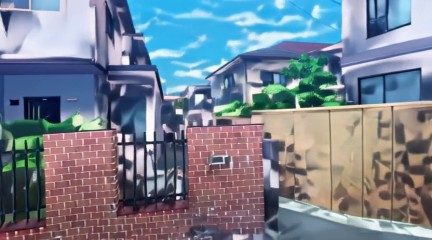}};
\draw[black!22,line width=0.4pt] (sv01r3hyworld.north west) rectangle (sv01r3hyworld.south east);
\node[img,anchor=north west] (sv01r3matrix) at (7.890,-6.084) {\includegraphics[width=1.96cm]{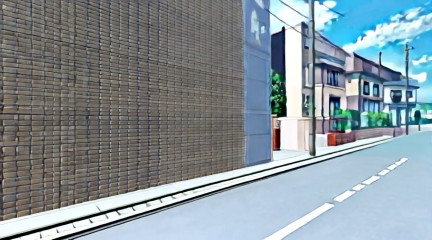}};
\draw[black!22,line width=0.4pt] (sv01r3matrix.north west) rectangle (sv01r3matrix.south east);
\node[img,anchor=north west] (sv01r3infworld) at (9.900,-6.084) {\includegraphics[width=1.96cm]{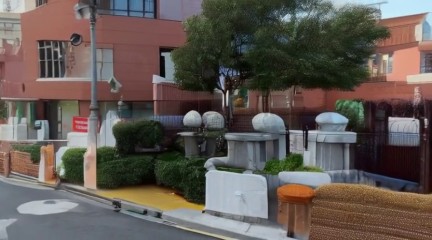}};
\draw[black!22,line width=0.4pt] (sv01r3infworld.north west) rectangle (sv01r3infworld.south east);
\node[img,anchor=north west] (sv01r3car) at (11.910,-6.084) {\includegraphics[width=1.96cm]{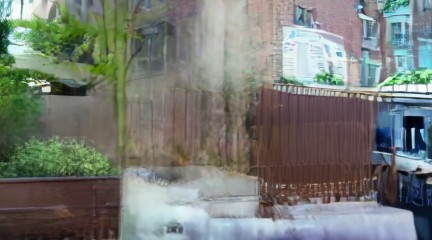}};
\draw[black!22,line width=0.4pt] (sv01r3car.north west) rectangle (sv01r3car.south east);
\node[img,anchor=north west] (sv01r3ours) at (13.920,-6.084) {\includegraphics[width=1.96cm]{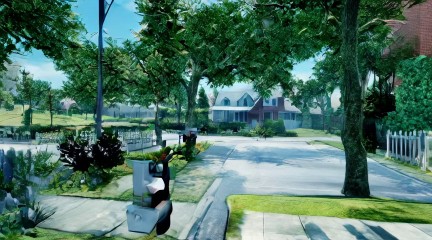}};
\draw[ourline,line width=0.8pt] (sv01r3ours.north west) rectangle (sv01r3ours.south east);
\node[img,anchor=north west] (sv17in) at (0,-8.817) {\includegraphics[width=1.62cm]{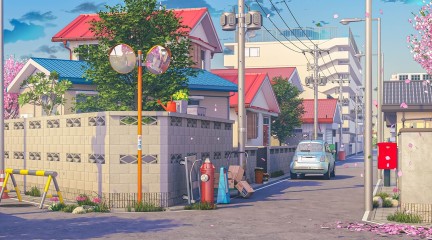}};
\draw[black!22,line width=0.4pt] (sv17in.north west) rectangle (sv17in.south east);
\node[tag,fill=selMix,anchor=north] at ([yshift=-0.06cm]sv17in.south) {$t = 0$\,s};
\draw[black!20,line width=0.4pt] (1.740,-7.523) -- (1.740,-11.010);
\node[img,anchor=north west] (sv17r1lingbot) at (1.860,-7.573) {\includegraphics[width=1.96cm]{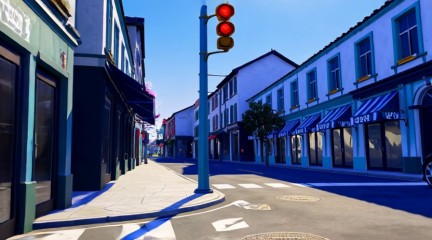}};
\draw[black!22,line width=0.4pt] (sv17r1lingbot.north west) rectangle (sv17r1lingbot.south east);
\node[tag,fill=selMix] at ([shift={(0.05,-0.05)}]sv17r1lingbot.north west) {$t = 59.75$\,s};
\node[img,anchor=north west] (sv17r1dreamx) at (3.870,-7.573) {\includegraphics[width=1.96cm]{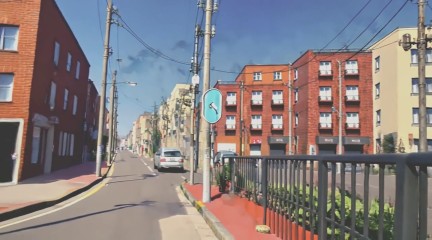}};
\draw[black!22,line width=0.4pt] (sv17r1dreamx.north west) rectangle (sv17r1dreamx.south east);
\node[img,anchor=north west] (sv17r1hyworld) at (5.880,-7.573) {\includegraphics[width=1.96cm]{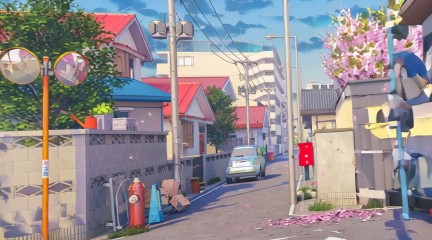}};
\draw[black!22,line width=0.4pt] (sv17r1hyworld.north west) rectangle (sv17r1hyworld.south east);
\node[img,anchor=north west] (sv17r1matrix) at (7.890,-7.573) {\includegraphics[width=1.96cm]{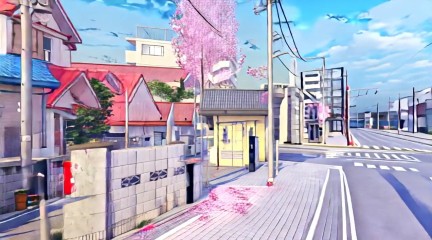}};
\draw[black!22,line width=0.4pt] (sv17r1matrix.north west) rectangle (sv17r1matrix.south east);
\node[img,anchor=north west] (sv17r1infworld) at (9.900,-7.573) {\includegraphics[width=1.96cm]{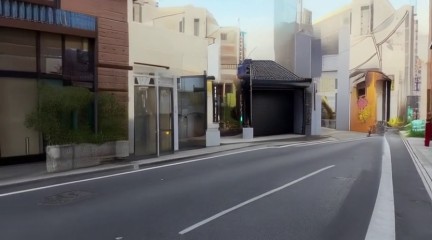}};
\draw[black!22,line width=0.4pt] (sv17r1infworld.north west) rectangle (sv17r1infworld.south east);
\node[img,anchor=north west] (sv17r1car) at (11.910,-7.573) {\includegraphics[width=1.96cm]{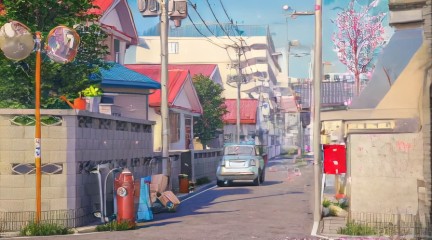}};
\draw[black!22,line width=0.4pt] (sv17r1car.north west) rectangle (sv17r1car.south east);
\node[img,anchor=north west] (sv17r1ours) at (13.920,-7.573) {\includegraphics[width=1.96cm]{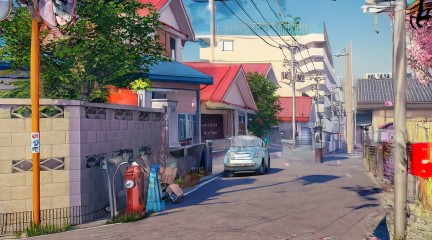}};
\draw[ourline,line width=0.8pt] (sv17r1ours.north west) rectangle (sv17r1ours.south east);
\draw[black!30,line width=0.4pt,densely dotted] (1.800,-8.692) -- (15.940,-8.692);
\node[img,anchor=north west] (sv17r2lingbot) at (1.860,-8.722) {\includegraphics[width=1.96cm]{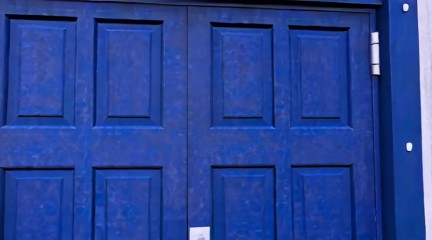}};
\draw[black!22,line width=0.4pt] (sv17r2lingbot.north west) rectangle (sv17r2lingbot.south east);
\node[tag,fill=selOv] at ([shift={(0.05,-0.05)}]sv17r2lingbot.north west) {$t = 6$\,s};
\node[img,anchor=north west] (sv17r2dreamx) at (3.870,-8.722) {\includegraphics[width=1.96cm]{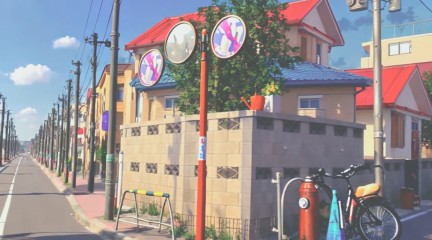}};
\draw[black!22,line width=0.4pt] (sv17r2dreamx.north west) rectangle (sv17r2dreamx.south east);
\node[img,anchor=north west] (sv17r2hyworld) at (5.880,-8.722) {\includegraphics[width=1.96cm]{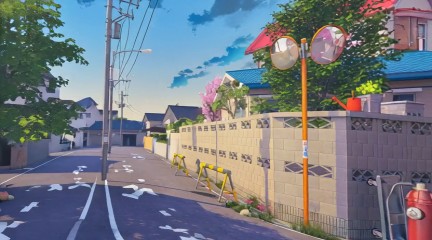}};
\draw[black!22,line width=0.4pt] (sv17r2hyworld.north west) rectangle (sv17r2hyworld.south east);
\node[img,anchor=north west] (sv17r2matrix) at (7.890,-8.722) {\includegraphics[width=1.96cm]{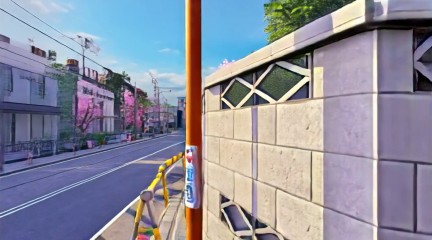}};
\draw[black!22,line width=0.4pt] (sv17r2matrix.north west) rectangle (sv17r2matrix.south east);
\node[img,anchor=north west] (sv17r2infworld) at (9.900,-8.722) {\includegraphics[width=1.96cm]{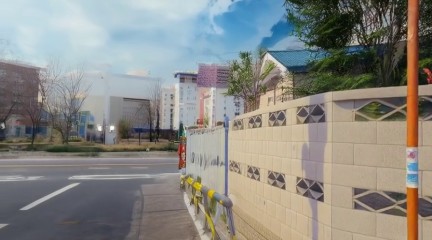}};
\draw[black!22,line width=0.4pt] (sv17r2infworld.north west) rectangle (sv17r2infworld.south east);
\node[img,anchor=north west] (sv17r2car) at (11.910,-8.722) {\includegraphics[width=1.96cm]{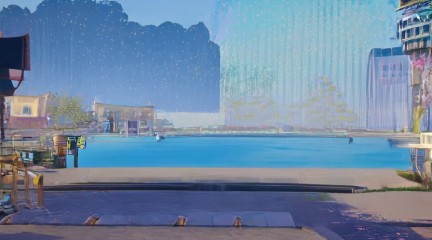}};
\draw[black!22,line width=0.4pt] (sv17r2car.north west) rectangle (sv17r2car.south east);
\node[img,anchor=north west] (sv17r2ours) at (13.920,-8.722) {\includegraphics[width=1.96cm]{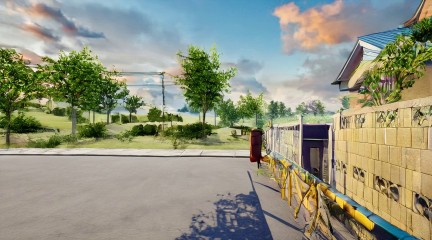}};
\draw[ourline,line width=0.8pt] (sv17r2ours.north west) rectangle (sv17r2ours.south east);
\node[img,anchor=north west] (sv17r3lingbot) at (1.860,-9.871) {\includegraphics[width=1.96cm]{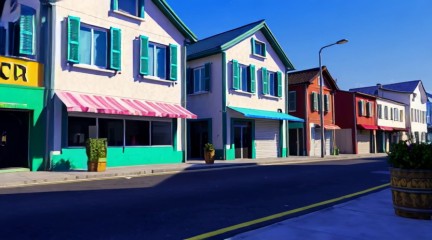}};
\draw[black!22,line width=0.4pt] (sv17r3lingbot.north west) rectangle (sv17r3lingbot.south east);
\node[tag,fill=selOv] at ([shift={(0.05,-0.05)}]sv17r3lingbot.north west) {$t = 47.5$\,s};
\node[img,anchor=north west] (sv17r3dreamx) at (3.870,-9.871) {\includegraphics[width=1.96cm]{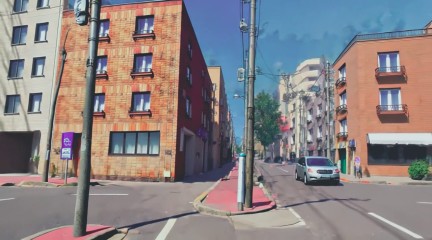}};
\draw[black!22,line width=0.4pt] (sv17r3dreamx.north west) rectangle (sv17r3dreamx.south east);
\node[img,anchor=north west] (sv17r3hyworld) at (5.880,-9.871) {\includegraphics[width=1.96cm]{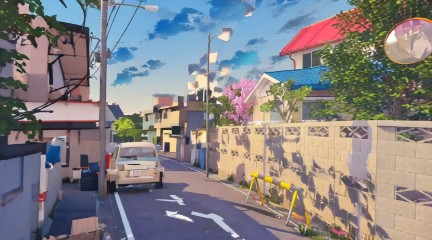}};
\draw[black!22,line width=0.4pt] (sv17r3hyworld.north west) rectangle (sv17r3hyworld.south east);
\node[img,anchor=north west] (sv17r3matrix) at (7.890,-9.871) {\includegraphics[width=1.96cm]{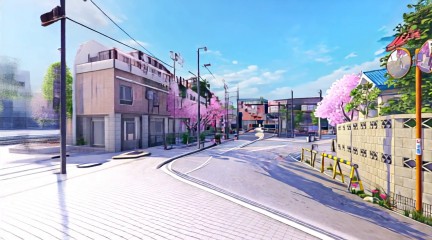}};
\draw[black!22,line width=0.4pt] (sv17r3matrix.north west) rectangle (sv17r3matrix.south east);
\node[img,anchor=north west] (sv17r3infworld) at (9.900,-9.871) {\includegraphics[width=1.96cm]{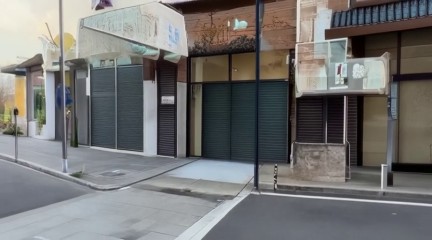}};
\draw[black!22,line width=0.4pt] (sv17r3infworld.north west) rectangle (sv17r3infworld.south east);
\node[img,anchor=north west] (sv17r3car) at (11.910,-9.871) {\includegraphics[width=1.96cm]{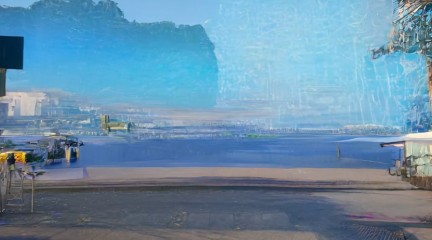}};
\draw[black!22,line width=0.4pt] (sv17r3car.north west) rectangle (sv17r3car.south east);
\node[img,anchor=north west] (sv17r3ours) at (13.920,-9.871) {\includegraphics[width=1.96cm]{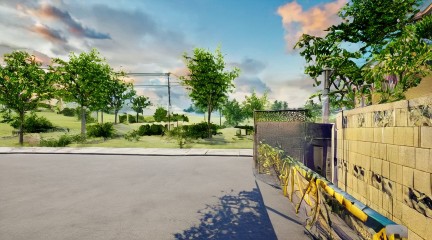}};
\draw[ourline,line width=0.8pt] (sv17r3ours.north west) rectangle (sv17r3ours.south east);
\end{tikzpicture}}
\caption{\textbf{More returns from open world models}, extending \hyperref[fig:revisit]{Figure~\ref*{fig:revisit}} to three more web photographs. Each block follows one photograph's
walk. Blue tags mark returns to the starting pose; orange tags the two ends of a return
to a pose met mid-path. Rollout seed $0$.}
\label{fig:sotamore}
\end{figure}

%% file: tables/selfcond.tex
\begin{table}[!htbp]
\caption{\textbf{Quality against what fills the context} on seen and unseen scenes, first
$60$\,s, five seeds. \emph{Generated}: ordinary rollout. \emph{GT bank}: the bank holds
ground-truth latents. \emph{GT context}: bank and recent window both ground truth. The
ground-truth videos themselves score $0.678$ and $0.803$ on seen scenes and $0.719$ and
$0.774$ on unseen ones. With a clean history each chunk generates
only four latents, so the seed spread rounds to zero.}
\label{tab:selfcond}
\centering\scriptsize
\setlength{\tabcolsep}{4pt}
\renewcommand{\arraystretch}{1.0}
\begin{tabular}{@{}l ccc ccc@{}}
\toprule
& \multicolumn{3}{c}{Imaging$\uparrow$} & \multicolumn{3}{c}{Subject$\uparrow$} \\
\cmidrule(lr){2-4}\cmidrule(l){5-7}
Training & Generated & GT bank & GT context & Generated & GT bank & GT context \\
\midrule
\rowcolor{grouprow}
\multicolumn{7}{@{}l}{\textit{Seen scenes}} \\
$10$\,s, no retrieval & 0.648{\fontsize{5}{5}\selectfont$\,\pm$.005} & -- & 0.671{\fontsize{5}{5}\selectfont$\,\pm$.001} & 0.801{\fontsize{5}{5}\selectfont$\,\pm$.002} & -- & 0.820{\fontsize{5}{5}\selectfont$\,\pm$.000} \\
$10$\,s, top-$K$, no bank in training & 0.654{\fontsize{5}{5}\selectfont$\,\pm$.008} & 0.663{\fontsize{5}{5}\selectfont$\,\pm$.004} & 0.673{\fontsize{5}{5}\selectfont$\,\pm$.000} & 0.799{\fontsize{5}{5}\selectfont$\,\pm$.006} & 0.811{\fontsize{5}{5}\selectfont$\,\pm$.002} & 0.824{\fontsize{5}{5}\selectfont$\,\pm$.000} \\
$100$\,s, top-$K$, bank & 0.646{\fontsize{5}{5}\selectfont$\,\pm$.015} & 0.666{\fontsize{5}{5}\selectfont$\,\pm$.004} & 0.668{\fontsize{5}{5}\selectfont$\,\pm$.000} & 0.804{\fontsize{5}{5}\selectfont$\,\pm$.001} & 0.818{\fontsize{5}{5}\selectfont$\,\pm$.002} & 0.820{\fontsize{5}{5}\selectfont$\,\pm$.000} \\
$200$\,s, top-$K$, bank & 0.637{\fontsize{5}{5}\selectfont$\,\pm$.012} & 0.656{\fontsize{5}{5}\selectfont$\,\pm$.003} & 0.666{\fontsize{5}{5}\selectfont$\,\pm$.000} & 0.806{\fontsize{5}{5}\selectfont$\,\pm$.006} & 0.820{\fontsize{5}{5}\selectfont$\,\pm$.001} & 0.821{\fontsize{5}{5}\selectfont$\,\pm$.000} \\
\addlinespace[1pt]
\rowcolor{grouprow}
\multicolumn{7}{@{}l}{\textit{Unseen scenes}} \\
$10$\,s, no retrieval & 0.666{\fontsize{5}{5}\selectfont$\,\pm$.014} & -- & 0.720{\fontsize{5}{5}\selectfont$\,\pm$.001} & 0.751{\fontsize{5}{5}\selectfont$\,\pm$.006} & -- & 0.803{\fontsize{5}{5}\selectfont$\,\pm$.000} \\
$10$\,s, top-$K$, no bank in training & 0.675{\fontsize{5}{5}\selectfont$\,\pm$.016} & 0.704{\fontsize{5}{5}\selectfont$\,\pm$.003} & 0.721{\fontsize{5}{5}\selectfont$\,\pm$.001} & 0.757{\fontsize{5}{5}\selectfont$\,\pm$.003} & 0.780{\fontsize{5}{5}\selectfont$\,\pm$.005} & 0.809{\fontsize{5}{5}\selectfont$\,\pm$.000} \\
$100$\,s, top-$K$, bank & 0.668{\fontsize{5}{5}\selectfont$\,\pm$.013} & 0.706{\fontsize{5}{5}\selectfont$\,\pm$.000} & 0.714{\fontsize{5}{5}\selectfont$\,\pm$.001} & 0.772{\fontsize{5}{5}\selectfont$\,\pm$.006} & 0.801{\fontsize{5}{5}\selectfont$\,\pm$.004} & 0.805{\fontsize{5}{5}\selectfont$\,\pm$.000} \\
$200$\,s, top-$K$, bank & 0.627{\fontsize{5}{5}\selectfont$\,\pm$.013} & 0.689{\fontsize{5}{5}\selectfont$\,\pm$.007} & 0.710{\fontsize{5}{5}\selectfont$\,\pm$.001} & 0.776{\fontsize{5}{5}\selectfont$\,\pm$.008} & 0.801{\fontsize{5}{5}\selectfont$\,\pm$.002} & 0.806{\fontsize{5}{5}\selectfont$\,\pm$.000} \\
\bottomrule
\end{tabular}
\end{table}

%% file: figures/fig_teaser_grid.tex
\definecolor{selTwo}{HTML}{1BAF7A}
\begin{figure}[!tbp]
\centering
\resizebox{\textwidth}{!}{%
\begin{tikzpicture}[
  img/.style={inner sep=0pt,outer sep=0pt},
  hd/.style={font=\fontsize{4.6}{5.4}\selectfont,anchor=south,inner sep=1.2pt,text=black!85,align=center},
  tag/.style={font=\fontsize{3.0}{3.6}\selectfont,text=white,fill opacity=0.85,text opacity=1,rounded corners=0.9pt,inner sep=0.9pt,anchor=north west},
]
\node[img,anchor=north west] (g0c0) at (0.000,0.000) {\includegraphics[width=1.62cm]{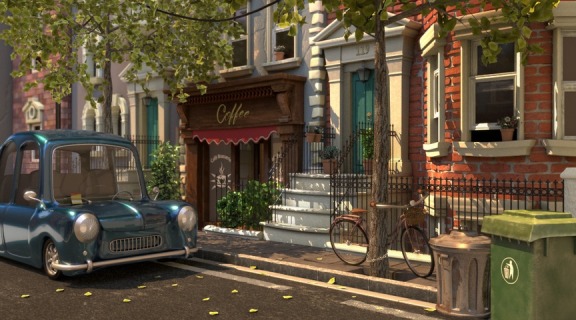}};
\node[hd] at ([yshift=0.03cm]g0c0.north) {Input};
\node[img,anchor=north west] (g0c1) at (1.900,0.000) {\includegraphics[width=1.62cm]{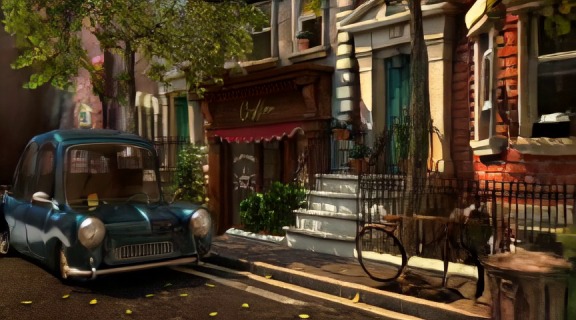}};
\node[tag,fill=selMix] at ([shift={(0.03,-0.03)}]g0c1.north west) {$t = 22.75$\,s};
\node[hd] at ([yshift=0.03cm]g0c1.north) {Near first pose};
\node[img,anchor=north west] (g0c2) at (3.580,0.000) {\includegraphics[width=1.62cm]{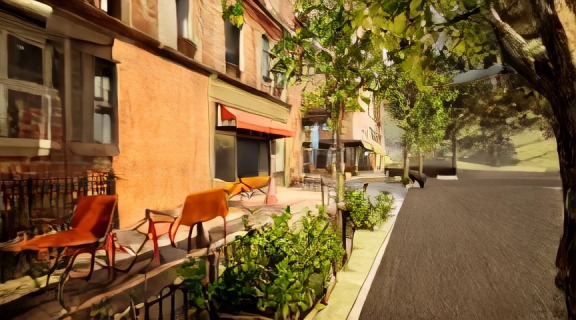}};
\node[tag,fill=selOv] at ([shift={(0.03,-0.03)}]g0c2.north west) {$t = 10$\,s};
\node[hd] at ([yshift=0.03cm]g0c2.north) {Return 1, first};
\node[img,anchor=north west] (g0c3) at (5.260,0.000) {\includegraphics[width=1.62cm]{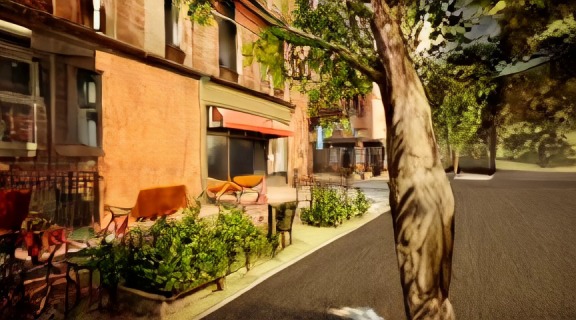}};
\node[tag,fill=selOv] at ([shift={(0.03,-0.03)}]g0c3.north west) {$t = 18$\,s};
\node[hd] at ([yshift=0.03cm]g0c3.north) {Return 1, back};
\node[img,anchor=north west] (g0c4) at (6.940,0.000) {\includegraphics[width=1.62cm]{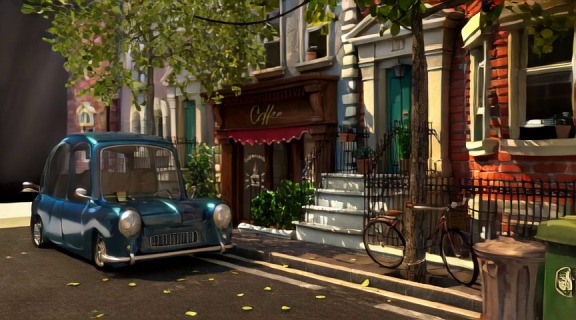}};
\node[tag,fill=selTwo] at ([shift={(0.03,-0.03)}]g0c4.north west) {$t = 0$\,s};
\node[hd] at ([yshift=0.03cm]g0c4.north) {Return 2, first};
\node[img,anchor=north west] (g0c5) at (8.620,0.000) {\includegraphics[width=1.62cm]{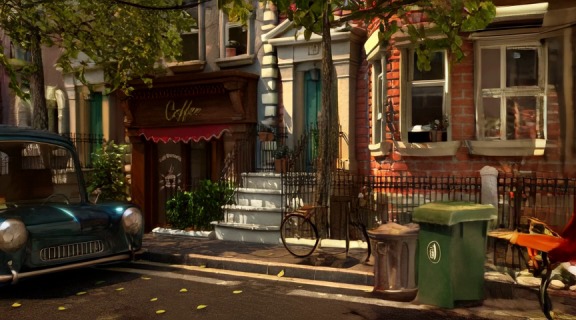}};
\node[tag,fill=selTwo] at ([shift={(0.03,-0.03)}]g0c5.north west) {$t = 22$\,s};
\node[hd] at ([yshift=0.03cm]g0c5.north) {Return 2, back};
\node[img,anchor=north west] (g1c0) at (0.000,-0.980) {\includegraphics[width=1.62cm]{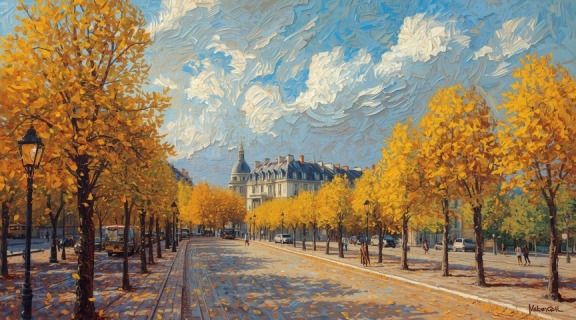}};
\node[img,anchor=north west] (g1c1) at (1.900,-0.980) {\includegraphics[width=1.62cm]{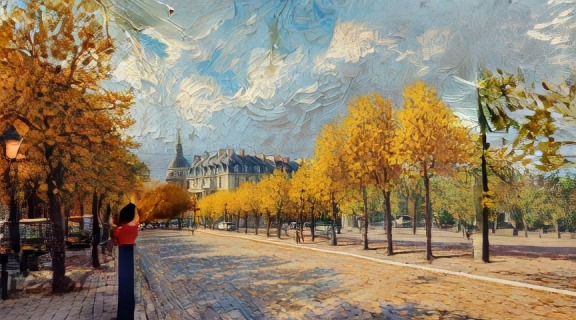}};
\node[tag,fill=selMix] at ([shift={(0.03,-0.03)}]g1c1.north west) {$t = 14.75$\,s};
\node[img,anchor=north west] (g1c2) at (3.580,-0.980) {\includegraphics[width=1.62cm]{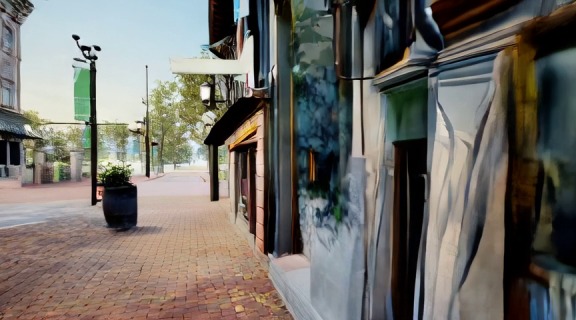}};
\node[tag,fill=selOv] at ([shift={(0.03,-0.03)}]g1c2.north west) {$t = 6.75$\,s};
\node[img,anchor=north west] (g1c3) at (5.260,-0.980) {\includegraphics[width=1.62cm]{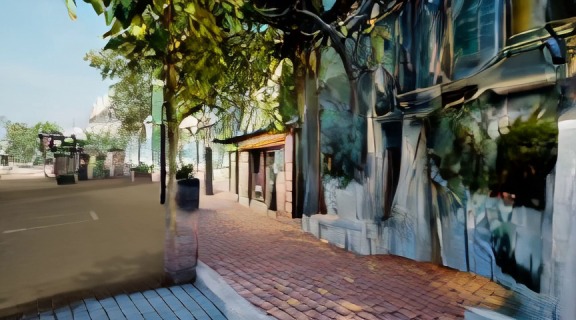}};
\node[tag,fill=selOv] at ([shift={(0.03,-0.03)}]g1c3.north west) {$t = 21.25$\,s};
\node[img,anchor=north west] (g1c4) at (6.940,-0.980) {\includegraphics[width=1.62cm]{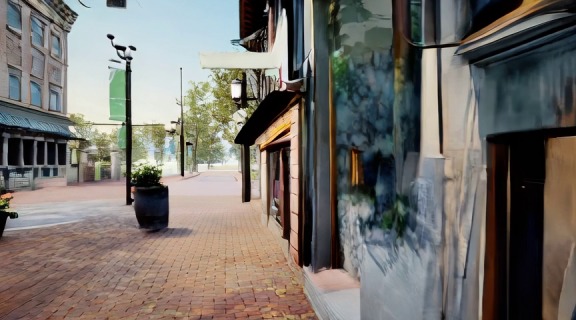}};
\node[tag,fill=selTwo] at ([shift={(0.03,-0.03)}]g1c4.north west) {$t = 7.75$\,s};
\node[img,anchor=north west] (g1c5) at (8.620,-0.980) {\includegraphics[width=1.62cm]{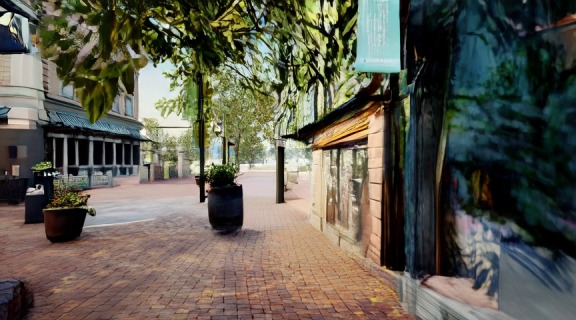}};
\node[tag,fill=selTwo] at ([shift={(0.03,-0.03)}]g1c5.north west) {$t = 25.25$\,s};
\node[img,anchor=north west] (g2c0) at (0.000,-1.960) {\includegraphics[width=1.62cm]{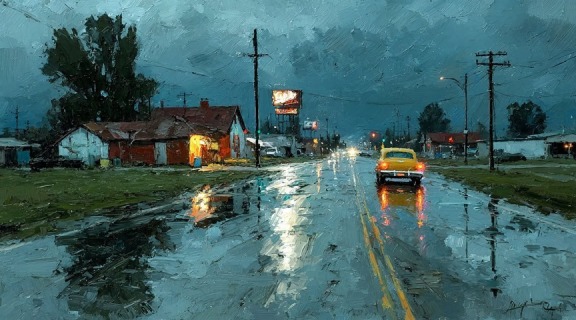}};
\node[img,anchor=north west] (g2c1) at (1.900,-1.960) {\includegraphics[width=1.62cm]{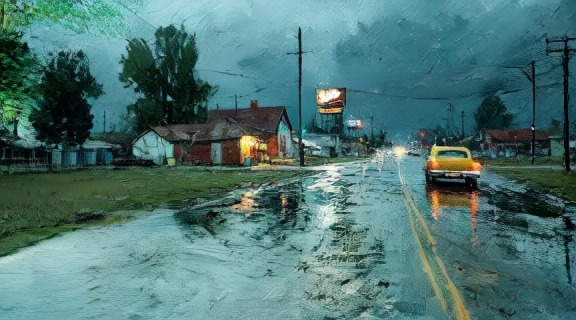}};
\node[tag,fill=selMix] at ([shift={(0.03,-0.03)}]g2c1.north west) {$t = 17.25$\,s};
\node[img,anchor=north west] (g2c2) at (3.580,-1.960) {\includegraphics[width=1.62cm]{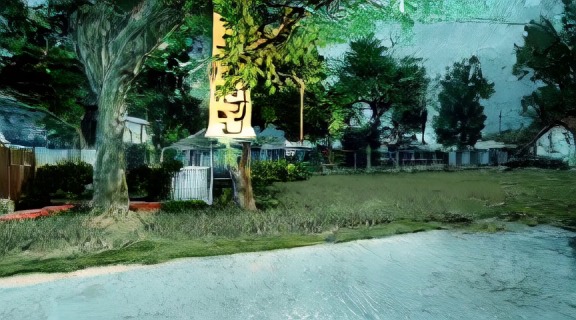}};
\node[tag,fill=selOv] at ([shift={(0.03,-0.03)}]g2c2.north west) {$t = 3.75$\,s};
\node[img,anchor=north west] (g2c3) at (5.260,-1.960) {\includegraphics[width=1.62cm]{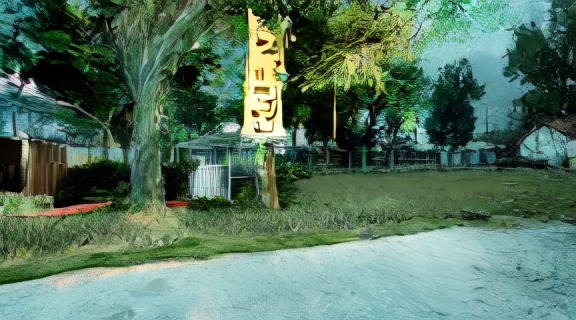}};
\node[tag,fill=selOv] at ([shift={(0.03,-0.03)}]g2c3.north west) {$t = 15$\,s};
\node[img,anchor=north west] (g2c4) at (6.940,-1.960) {\includegraphics[width=1.62cm]{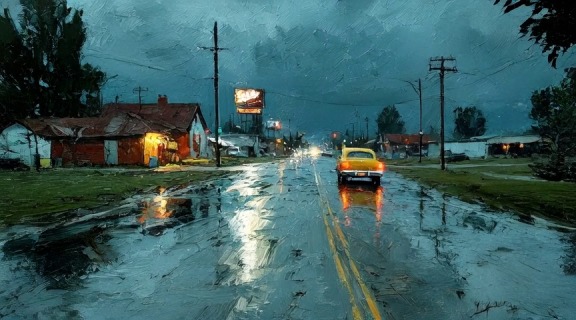}};
\node[tag,fill=selTwo] at ([shift={(0.03,-0.03)}]g2c4.north west) {$t = 0.5$\,s};
\node[img,anchor=north west] (g2c5) at (8.620,-1.960) {\includegraphics[width=1.62cm]{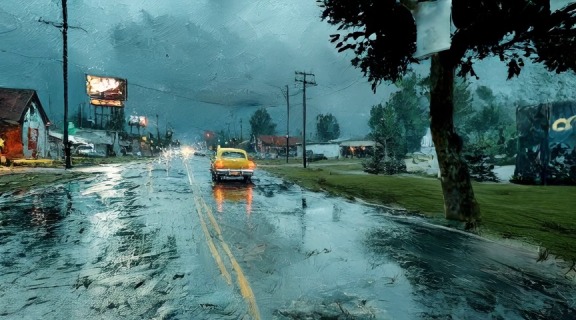}};
\node[tag,fill=selTwo] at ([shift={(0.03,-0.03)}]g2c5.north west) {$t = 18.75$\,s};
\node[img,anchor=north west] (g3c0) at (0.000,-2.940) {\includegraphics[width=1.62cm]{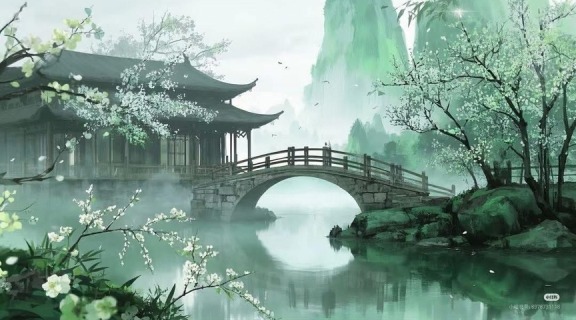}};
\node[img,anchor=north west] (g3c1) at (1.900,-2.940) {\includegraphics[width=1.62cm]{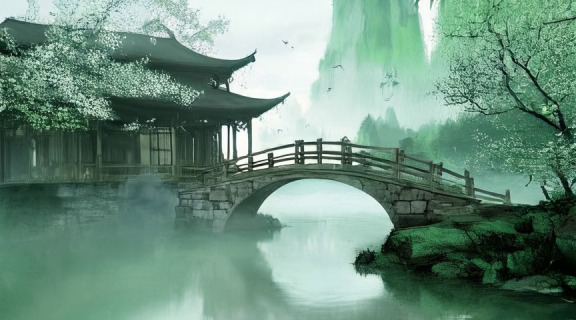}};
\node[tag,fill=selMix] at ([shift={(0.03,-0.03)}]g3c1.north west) {$t = 17$\,s};
\node[img,anchor=north west] (g3c2) at (3.580,-2.940) {\includegraphics[width=1.62cm]{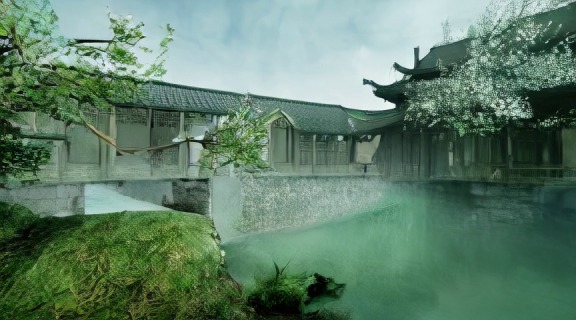}};
\node[tag,fill=selOv] at ([shift={(0.03,-0.03)}]g3c2.north west) {$t = 6$\,s};
\node[img,anchor=north west] (g3c3) at (5.260,-2.940) {\includegraphics[width=1.62cm]{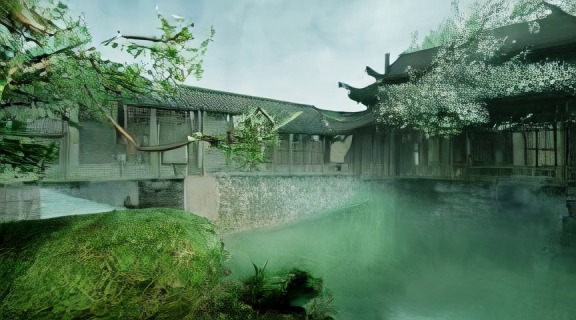}};
\node[tag,fill=selOv] at ([shift={(0.03,-0.03)}]g3c3.north west) {$t = 14.25$\,s};
\node[img,anchor=north west] (g3c4) at (6.940,-2.940) {\includegraphics[width=1.62cm]{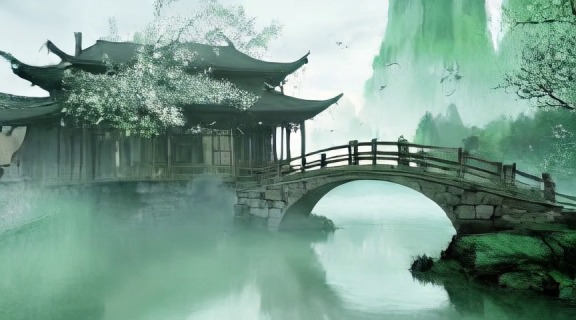}};
\node[tag,fill=selTwo] at ([shift={(0.03,-0.03)}]g3c4.north west) {$t = 4$\,s};
\node[img,anchor=north west] (g3c5) at (8.620,-2.940) {\includegraphics[width=1.62cm]{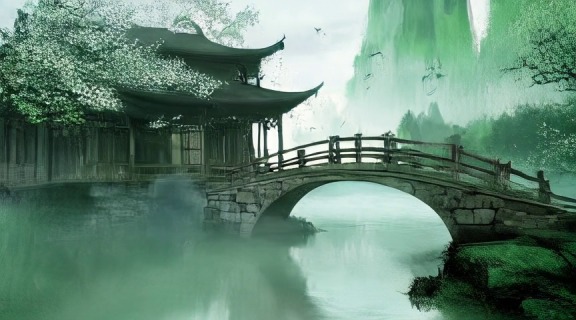}};
\node[tag,fill=selTwo] at ([shift={(0.03,-0.03)}]g3c5.north west) {$t = 18.25$\,s};
\node[img,anchor=north west] (g4c0) at (0.000,-3.920) {\includegraphics[width=1.62cm]{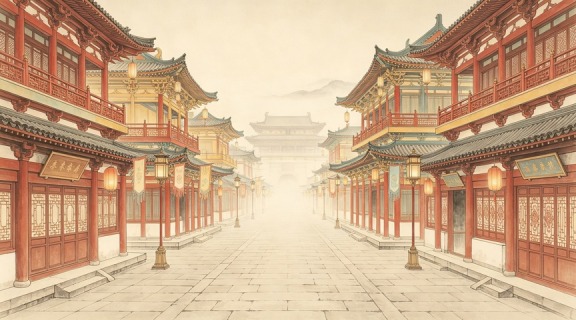}};
\node[img,anchor=north west] (g4c1) at (1.900,-3.920) {\includegraphics[width=1.62cm]{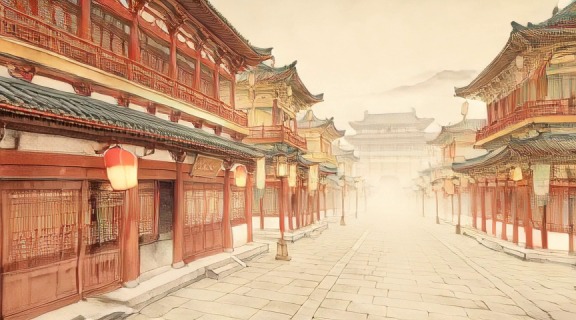}};
\node[tag,fill=selMix] at ([shift={(0.03,-0.03)}]g4c1.north west) {$t = 8$\,s};
\node[img,anchor=north west] (g4c2) at (3.580,-3.920) {\includegraphics[width=1.62cm]{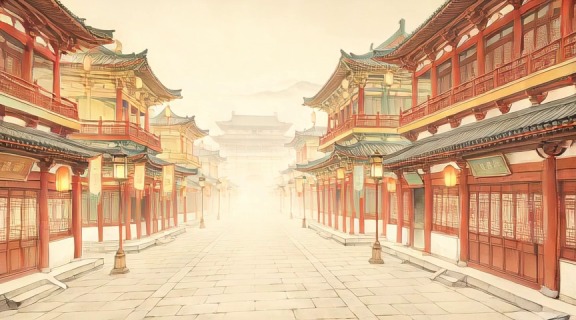}};
\node[tag,fill=selOv] at ([shift={(0.03,-0.03)}]g4c2.north west) {$t = 0$\,s};
\node[img,anchor=north west] (g4c3) at (5.260,-3.920) {\includegraphics[width=1.62cm]{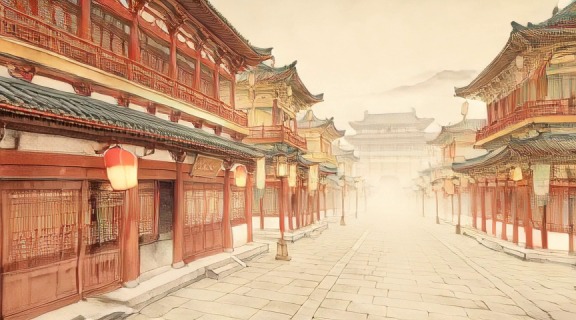}};
\node[tag,fill=selOv] at ([shift={(0.03,-0.03)}]g4c3.north west) {$t = 8$\,s};
\node[img,anchor=north west] (g4c4) at (6.940,-3.920) {\includegraphics[width=1.62cm]{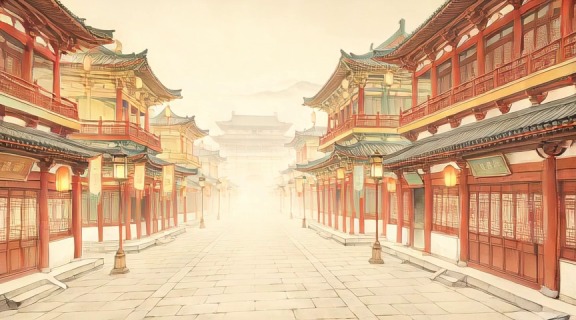}};
\node[tag,fill=selTwo] at ([shift={(0.03,-0.03)}]g4c4.north west) {$t = 0$\,s};
\node[img,anchor=north west] (g4c5) at (8.620,-3.920) {\includegraphics[width=1.62cm]{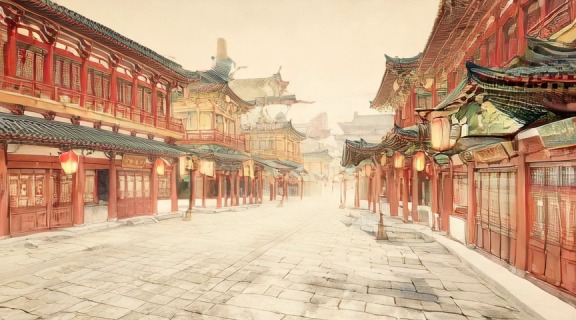}};
\node[tag,fill=selTwo] at ([shift={(0.03,-0.03)}]g4c5.north west) {$t = 14.5$\,s};
\node[img,anchor=north west] (g5c0) at (0.000,-4.900) {\includegraphics[width=1.62cm]{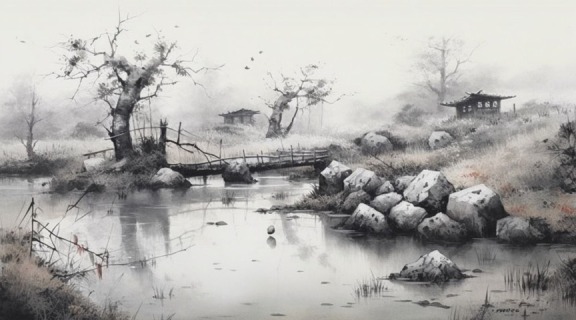}};
\node[img,anchor=north west] (g5c1) at (1.900,-4.900) {\includegraphics[width=1.62cm]{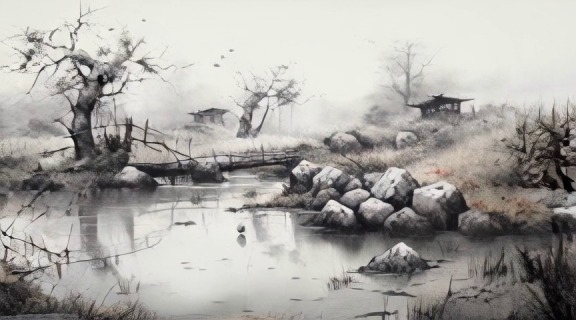}};
\node[tag,fill=selMix] at ([shift={(0.03,-0.03)}]g5c1.north west) {$t = 8$\,s};
\node[img,anchor=north west] (g5c2) at (3.580,-4.900) {\includegraphics[width=1.62cm]{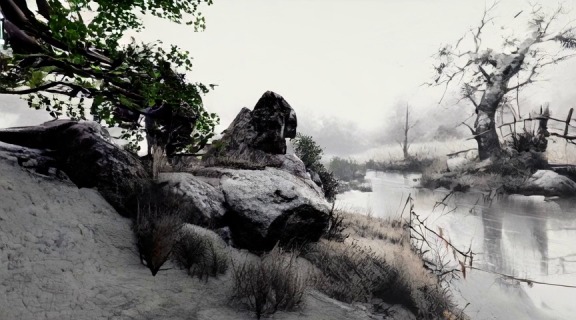}};
\node[tag,fill=selOv] at ([shift={(0.03,-0.03)}]g5c2.north west) {$t = 2.5$\,s};
\node[img,anchor=north west] (g5c3) at (5.260,-4.900) {\includegraphics[width=1.62cm]{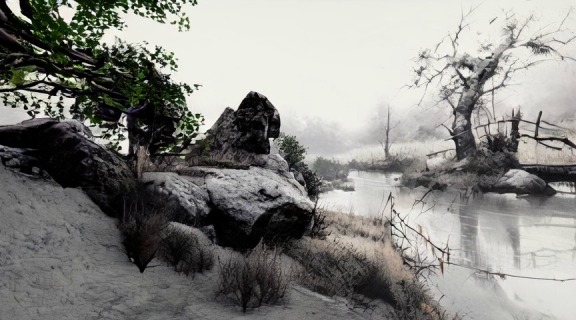}};
\node[tag,fill=selOv] at ([shift={(0.03,-0.03)}]g5c3.north west) {$t = 10.5$\,s};
\node[img,anchor=north west] (g5c4) at (6.940,-4.900) {\includegraphics[width=1.62cm]{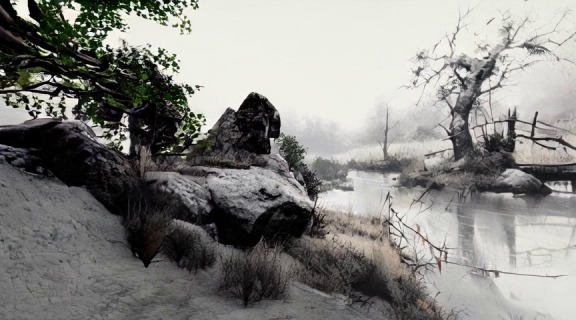}};
\node[tag,fill=selTwo] at ([shift={(0.03,-0.03)}]g5c4.north west) {$t = 3.25$\,s};
\node[img,anchor=north west] (g5c5) at (8.620,-4.900) {\includegraphics[width=1.62cm]{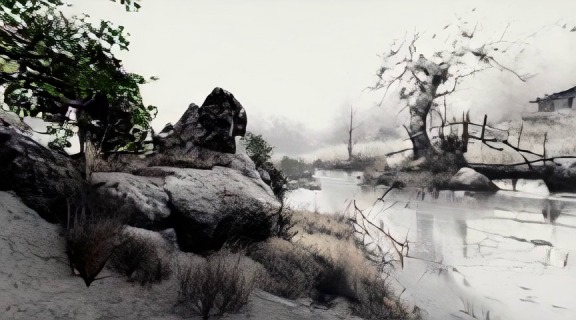}};
\node[tag,fill=selTwo] at ([shift={(0.03,-0.03)}]g5c5.north west) {$t = 19.5$\,s};
\end{tikzpicture}}
\caption{\textbf{More returns from \ours} on six web photographs. Blue: a pose close to the first frame's. Orange and green: two returns, each shown at its first visit and on coming back; they meet the Revisit rule of \hyperref[sec:setup]{Sec.~\ref*{sec:setup}}, or, where a path has only one such return, the same rule at twice its tolerance.}
\label{fig:teasergrid0}
\end{figure}
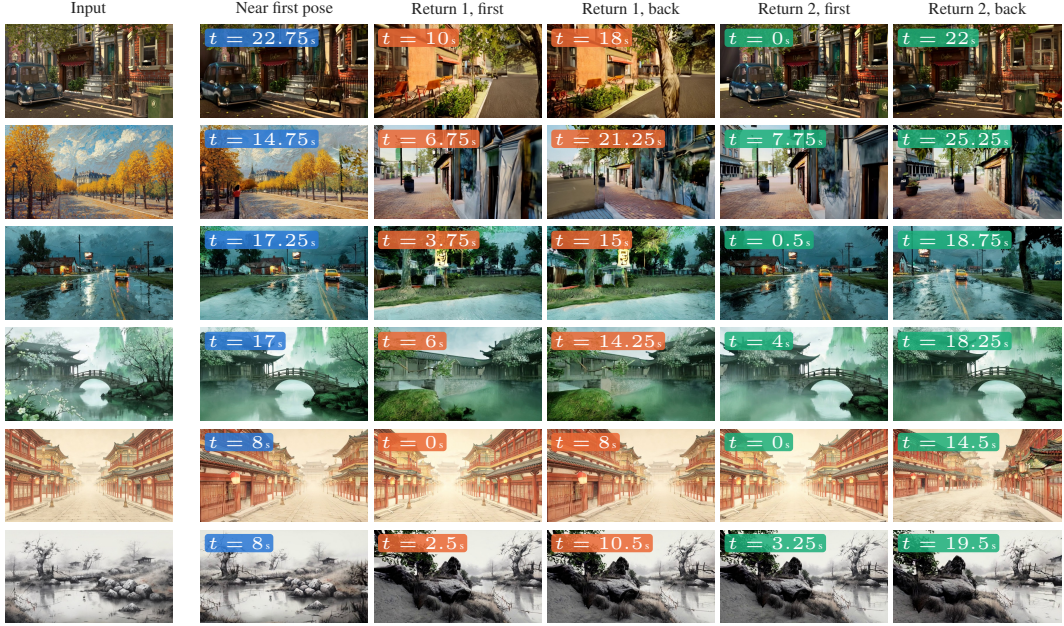

%% file: appendix/04_visual.tex
\subsection{Additional Visual Comparisons}
\label{app:visual}

\hyperref[fig:sotamore]{Figure~\ref*{fig:sotamore}} extends \hyperref[fig:revisit]{Figure~\ref*{fig:revisit}} to three more web photographs; \hyperref[fig:teasergrid0]{Figure~\ref*{fig:teasergrid0}} shows more returns from \ours alone.